\documentclass[twoside,11pt]{article}

\usepackage{blindtext}
\usepackage[abbrvbib]{jmlr2e}
\usepackage[english]{babel}
\usepackage{advdate}
\usepackage{algorithm}
\usepackage{algpseudocode}
\usepackage{amsmath}
\usepackage{amssymb}
\usepackage{booktabs}
\usepackage{braket}
\usepackage{caption}
\usepackage{enumitem}
\usepackage{bbm}
\usepackage{mathtools}
\usepackage{soul}
\usepackage{subcaption}
\usepackage{tikz}
\usetikzlibrary{arrows.meta, positioning}

\usepackage{xcolor}
\usepackage{graphicx}

\newcommand{\dt}{\mathrm{d}t}

\newcommand{\dWt}{\mathrm{d}W_t}

\newcommand{\N}{\mathbb{N}}

\newcommand{\R} {\mathbb R}

\newcommand{\E} {\mathbb{E}}

\newcommand{\ran} {\mathrm{ran}}
\newcommand{\cov} {\mathrm{Cov}}

\newcommand{\Id}{\mathrm{I}_d}

\DeclareMathOperator*{\argmax}{arg\,max}
\DeclareMathOperator*{\argmin}{arg\,min}

\newcommand{\rv}[1]{\textcolor{black}{#1}}

\usepackage[capitalize]{cleveref}
\usepackage{lastpage}

\ShortHeadings{Identifying Drift, Diffusion, and Causal Structure from Temporal Snapshots}{Guan, Janssen, Rahmani, Warren, Zhang, Robeva, and Schiebinger}
\firstpageno{1}

\begin{document}

\jmlrheading{27}{2026}{1-\pageref{LastPage}}{3/25; Revised
8/26}{9/26}{25-0541}{Vincent Guan, Joseph Janssen, Hossein Rahmani, Andrew Warren, Stephen Zhang, Elina Robeva and Geoffrey Schiebinger}

\title{Identifying Drift, Diffusion, and Causal Structure from Temporal Snapshots}

\author{%
\name Vincent Guan$^{1*}$ \email vguan23@math.ubc.ca\\
\name Joseph Janssen$^{2*}$ \email joejanssen@eoas.ubc.ca\\
\name Hossein Rahmani$^{1}$ \email hrahmani@math.ubc.ca\\
\name Andrew Warren$^{1}$ \email awarren@math.ubc.ca\\
\name Stephen Y. Zhang$^{3}$ \email syz@syz.id.au\\
\name Elina Robeva$^{1}$ \email erobeva@math.ubc.ca\\
\name Geoffrey Schiebinger$^{1}$ \email geoff@math.ubc.ca\\
\\
\addr $^{1}$ Department of Mathematics, University of British Columbia\\
\addr $^{2}$ Department of Earth, Ocean and Atmospheric Sciences, University of British Columbia\\
\addr $^{3}$ School of Mathematics and Statistics, University of Melbourne\\
\addr $^{*}$ Corresponding author
}



\editor{Bryon Aragam}

\maketitle

\begin{abstract}
Stochastic differential equations (SDEs) are a fundamental tool for modelling dynamic processes, including gene regulatory networks (GRNs), contaminant transport, financial markets, and image generation. However, learning the underlying SDE from data is a challenging task, especially if individual trajectories are not observable. Motivated by burgeoning research in single-cell datasets, we present the first comprehensive approach for jointly identifying the drift and diffusion of an SDE from its temporal marginals. Assuming linear drift and additive diffusion, we \rv{show that non-identifiability can only arise if the initial distribution possesses generalized rotational symmetries. We further prove that even if this condition holds, the drift and diffusion can almost always be recovered from the marginals. Additionally, we show that the causal graph of any SDE with additive diffusion can be recovered from the identified SDE parameters.} To complement this theory, we adapt entropy-regularized optimal transport to handle anisotropic diffusion, and introduce APPEX (Alternating Projection Parameter Estimation from $X_0$), an iterative algorithm designed to estimate the drift, diffusion, and causal graph of an additive noise SDE, solely from temporal marginals, and demonstrate its effectiveness on simulated data from linear additive noise SDEs.
\end{abstract}


\section{Introduction}
This work presents the first comprehensive approach for jointly identifying the drift and diffusion of a stochastic differential equation (SDE) from observed temporal marginals. While parameter estimation has been studied extensively using trajectory data, either from one long trajectory \citep{nielsen2000parameter, bishwal2007parameter}, or multiple short trajectories \citep{manten2024signature, lu2021learning, pawlowicz2019lagrangian}, it is often impossible to observe individual trajectories when tracking large populations. For example, in single-cell RNA sequencing (scRNA-seq) datasets, the destructive nature of sequencing technologies prevents specific cells from being tracked over time \citep{trapnell2014dynamics, setty2016wishbone, farrell2018single}. Similarly, hydrogeochemical sensors cannot track distinct particles when studying contaminant flow, providing only data on plume migration \citep{man2007stochastic, hazas2023evolution,salamon2007modeling,elfeki2006prediction,adams1992field,boggs1992field,mackay1986natural,nenna2011application}. Observational data may therefore be limited to samples from temporal marginals of the process at various measurement times. We refer to this observational setting as \textit{marginals-only}.

Many stochastic processes are not identifiable under the marginals-only setting, since their distributions may evolve identically, despite having distinct parameter sets \citep{weinreb2018fundamental}. Previous works therefore simplify the problem by assuming that the drift is irrotational and that the diffusion is isotropic, with known diffusivity \citep{weinreb2018fundamental, lavenant2021towards,vargas2021solving,chizat2022trajectory}.


\subsection{Our Contributions}

\rv{We prove that the drift, diffusion, and underlying causal graph of an SDE with linear drift and additive diffusion are generically identifiable (meaning they can be uniquely recovered from the data for almost every SDE parameter configuration) from its temporal marginals.} We also introduce a practical algorithm for estimating each of these quantities solely from the marginals of a general additive diffusion process, and demonstrate the effectiveness of our method on simulated data.
We illustrate our theoretical contributions in Figure \ref{fig:APPEX_main_figure} and our algorithmic contributions in Figure \ref{fig:estimation_steps}. A detailed overview of our contributions is also given below. 

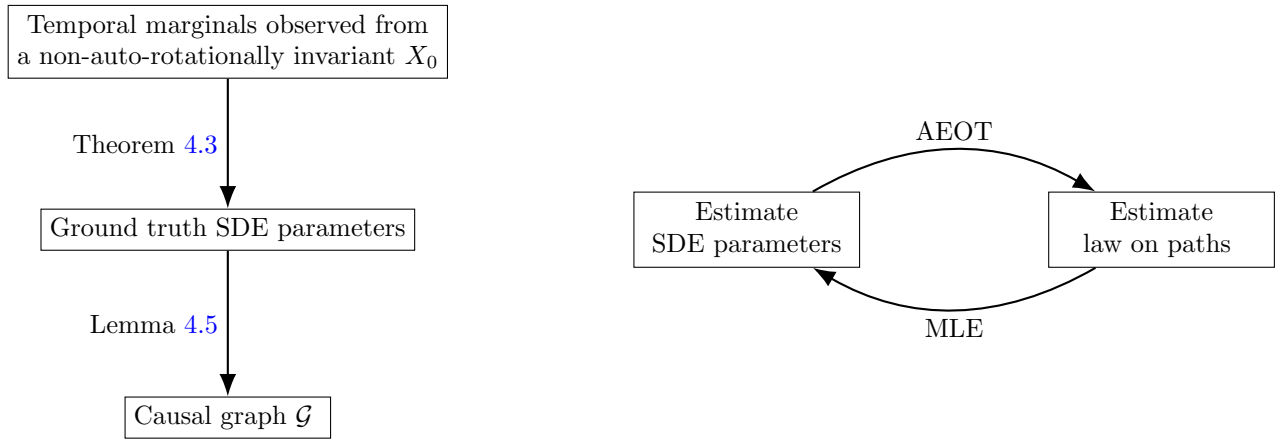
\begin{figure}
\centering
\begin{minipage}{0.45\textwidth}
    \centering
    \begin{subfigure}[b]{\textwidth} 
        \centering
        \vspace{-5pt}
        \begin{tikzpicture}[
          rect/.style = {rectangle, draw, minimum width=1cm, minimum height=0.5cm, align=center},
          arrow/.style = {draw, -{Latex[length=3mm]}, thick},
          node distance=1cm
        ]

        \node[rect] (X0) at (0, 0) {Temporal marginals observed from\\ a non-auto-rotationally invariant $X_0$};
        \node[rect] (params) at (0, -2.5) {Ground truth SDE parameters}; 
        \node[rect] (graph) at (0, -5) {Causal graph $\mathcal{G}$ \vphantom{X}}; 

        \draw[arrow] (X0) -- (params) node[midway, left, align=center] {Theorem \ref{thm: identifiability}};
        \draw[arrow] (params) -- (graph) node[midway, left] {Lemma \ref{prop: identify_causal_graph_SDE_add_noise}};

        \end{tikzpicture}
        
        \vspace{5pt}
        \caption{Theoretical contributions: identifying drift, diffusion, and causal structure from temporal marginals}
        \label{fig:APPEX_main_figure}
    \end{subfigure}
\end{minipage}%
\hfill
\begin{minipage}{0.45\textwidth}
    \centering
    \begin{subfigure}[b]{\textwidth} 
        \centering
        \vspace{25pt}
        \begin{tikzpicture}[
          rect/.style = {rectangle, draw, minimum width=3cm, minimum height=1cm, align=center},
          arrow/.style = {draw, -{Latex[length=3mm]}, thick},
          node distance=1cm
        ]

        \node[rect] (est_params) at (0, -2.5) {Estimate of\\
        SDE parameters}; 
        \node[rect, right=of est_params, xshift=1.5cm] (est_law) {Estimate of\\ law on paths \vphantom{X}}; 

        \draw[arrow, bend left=30] (est_params) to node[midway, above] {AEOT} (est_law);
        \draw[arrow, bend left=30] (est_law) to node[midway, below] {MLE} (est_params);

        \end{tikzpicture}
        
        \vspace{40pt}
        \caption{Algorithmic contributions: APPEX algorithm}
        \label{fig:estimation_steps}
    \end{subfigure}
\end{minipage}

\caption{Outline of our theoretical (a) and algorithmic (b) contributions. In (b), our algorithm Alternating Projection Parameter Estimation from $X_0$ alternates between Anisotropic Entropy-regularized Optimal Transport (AEOT) and Maximum Likelihood Estimation (MLE).}
\label{fig:contributions_figure}
\end{figure}

\begin{enumerate}
    \item We provide a theoretical foundation for system identification from temporal marginals.\begin{enumerate}
        \item  \rv{In Section \ref{sec:identify}, we prove identifiability conditions for jointly recovering the drift and diffusion of a linear additive noise SDE from continuously observed marginals. In \cref{thm: identifiability}, we prove that unidentifiable parameter configurations exist if and only if the initial distribution possesses generalized rotational symmetries, a condition we call auto-rotational invariance (ARI). Then, \cref{thm:stronger} provides a stronger unconditional guarantee: even when the initial distribution is ARI, the set of unidentifiable configurations has measure zero. Thus, linear additive noise SDEs are generically identifiable.}
        
        \item In Section \ref{sec:causal_graph}, 
        we connect parameter estimation to causal structure learning via dynamic structural causal models \citep{boeken2024dynamic}. We show that the causal graph of any additive noise model can be recovered from the identified SDE parameters.
    \end{enumerate}
    \item In Section \ref{sec:methods}, we introduce our algorithm: Alternating Projection Parameter Estimation from $X_0$ (APPEX). \rv{It is the first Schr\"odinger Bridge based method designed to estimate drift and diffusion from temporal marginals of additive noise SDEs, without prior knowledge.}
    \begin{enumerate}
        \item APPEX alternates between a trajectory inference step, which is a Schr\"odinger bridge problem, and a parameter estimation step, which is solved via maximum likelihood estimation. With each iteration, APPEX decreases the KL divergence between its inferred trajectories and the law of its estimated SDE.
        \item To solve the Schr\"odinger Bridge problem for trajectory inference, we adapt entropic optimal transport for parameter dependent transport costs, including arbitrary anisotropic diffusion.
        \item We test APPEX's efficacy across a wide range of experiments in Section \ref{sec:experiments}. We first use APPEX to empirically demonstrate that classical non-identifiable linear additive noise SDEs are resolved when $X_0$ is initialized appropriately according to Theorem \ref{thm: identifiability}.
        We then demonstrate APPEX's effectiveness over a large-scale experiment on higher dimensional SDEs with randomly generated parameter sets. Our results demonstrate that APPEX identifies arbitrary linear additive noise SDEs with significantly higher accuracy than the widely-used Waddington-OT (WOT) method \citep{schiebinger2019optimal}. Our final experiments demonstrate that APPEX can be used to identify the system's causal graph, even in the presence of correlated diffusion from latent confounders.
    \end{enumerate}
\end{enumerate}

\section{Background and Related Work}
In this section, we first give an overview of SDEs and a discussion of some real-world applications. We then review relevant literature and techniques for statistical inference from the ``marginals-only" observational setting. Finally, we discuss the problem of non-identifiability of SDEs under this setting.

\subsection{Stochastic Differential Equations (SDEs) and Modeling Real-World Systems}
The mathematical framework of SDEs, or drift-diffusion processes, originates from statistical mechanics. In such applications, the drift models the velocity of particles flowing in a carrier fluid, while diffusion is driven by Brownian motion, and governs their dispersion due to random collisions \citep{einstein1956investigations, macinnes1992stochastic}. Formally, a $d$-dimensional time-homogeneous SDE is driven by an $m$-dimensional Brownian motion $W_t$, and is parameterized by a drift function $b(x): \R^d \to \R^d$ and a diffusion function $\sigma(x): \R^d \to \R^{d \times m}$, such that
\begin{equation}
    \mathrm{d}X_t = b(X_t) \:\mathrm{d}t+ \sigma(X_t)\: \mathrm{d}W_t, \  X_0 \sim p_0,
    \label{eq: general_SDE}
\end{equation}
where $b(x)$ and $\sigma(x)$ are Lipschitz to ensure the solution's existence and uniqueness \citep{fang2020adaptive}.

By capturing deterministic and stochastic dynamics, SDEs provide a powerful framework for modeling complex systems across a variety of fields. For example, SDEs have been used to model cell differentiation in single-cell biology, where a stem cell changes into another cell type, such as a blood cell or bone cell. 
The drift can be related to the set of genes that directly regulate the expression of a gene of interest \citep{zhang2024joint,aalto2020gene,atanackovic2024dyngfn,zhao2024optimal,tejada2023causal}, and therefore provides insights into genetic diseases and gene therapies. Meanwhile, the diffusion informs the extent to which cell fates are determined by initial conditions \citep{forrow2024consistent}. In hydrological systems, drift is linked to important and unknown properties of the subsurface, such as hydraulic conductivity through average flow velocity \citep{hazas2023evolution,beven1993dispersion}, while the diffusion usually describes material heterogeneity or turbulence \citep{beven1993dispersion}, and therefore informs the applicability of popular models, such as Darcy's law \citep{man2007stochastic,oh2010stochastic,lichtner2002new}. Together, these parameters determine pollutant fates in hydrological systems, a key concern in safeguarding drinking water sources \citep{o2002part,frind2006well,locatelli2019simple,chen1999assessment,paulson1997transport,cahill2019advancing}. In addition to applications to physical sciences, SDEs have also been prominently used for pricing stocks and options in financial markets \citep{black1973pricing} and for image generation in computer vision \citep{meng2021sdedit}. 

\subsection{Linear Additive Noise SDEs}
In this paper, we focus on time-homogeneous additive noise SDEs. These SDEs are characterized by an arbitrary drift function $b(X_t)$ and a possibly anisotropic space-independent diffusion $G \in \mathbb{R}^{d \times m}$, such that 
\begin{equation}
    \mathrm{d}X_t = b(X_t) \:\mathrm{d}t+ G\: \mathrm{d}W_t, \  X_0 \sim p_0.
    \label{eq: additive_noise_SDE}
\end{equation}
We also refer to $H = GG^\top \in \R^{d \times d}$ as the (observational) diffusion, since it is only possible to \rv{identify} $H$ \citep{wang2024generator, pavliotis2014stochastic}. Time-homogeneous linear additive noise SDEs are of particular interest \citep{du2023dynamical, wang2024generator, zhang2024joint}, as they generalize the popular Ornstein-Uhlenbeck process. A time-homogeneous linear additive noise SDE obeys the form:
\begin{equation}
    \mathrm{d}X_t = AX_t \: \mathrm{d}t+ G\:\mathrm{d}W_t, \ \ X_0 \sim p_0, \ \ A \in \R^{d \times d}.
    \label{eq: linear_additive_noise_SDE}
\end{equation}
We overview important properties of this SDE in Section \ref{sec:linear_additive_noise} of the appendix, including its closed form solution and transition density.

\subsection{Inference from Temporal Marginals}
The marginals-only observational setting dates back to at least Aristotle, who dissected bird eggs at various developmental stages, in order to infer the processes and mediators (e.g., heat) behind embryogenesis \citep{waddington1935animal,ogle1882aristotle,aristotle1942generation}. If the underlying system is modeled by an SDE, then the evolution of its temporal marginals $p_t$ is defined by its Fokker-Planck equation. For example, a straightforward computation shows that the marginals of a linear additive noise SDE evolve according to the following Fokker-Planck equation:
\begin{equation}
    \frac{\partial}{\partial t}p(x,t) = -\nabla \cdot \left((Ax) p(x,t)\right) + \frac{1}{2} \nabla \cdot (H \nabla p(x,t)), \ \ p(x,0)=p_0(x).
    \label{eq: FP_linear_additive_noise}
\end{equation}

Contemporary research in statistical inference from temporal marginals largely comes from single-cell biology. Many works focus on trajectory inference for cell fates \citep{schiebinger2019optimal, lavenant2021towards,yachimura2023scegot, chizat2022trajectory}, while others infer the causal graph (GRN) \citep{aalto2020gene,brouillard2020differentiable, tejada2023causal, rohbeck2024bicycle}, or perform parameter estimation \citep{chakraborty2009parameter, shen2024learning}. Recent works jointly estimate a subset of these quantities by leveraging their intrinsic relationships. For instance, \citet{vargas2021solving} and \citet{shen2024multi} are primarily interested in trajectory inference, and iteratively estimate drift and trajectories. Meanwhile, \citet{zhang2024joint} iteratively estimates drift and trajectories, while additionally applying permanent interventions on the drift dynamics for network inference. 



\subsection{Schr\"odinger Bridge Problem}
Given a set of temporal marginals, trajectory inference is the task of estimating a law on paths that obeys each marginal constraint while closely approximating desired reference dynamics. This is commonly implemented by solving the Schr\"odinger Bridge (SB) problem \citep{chizat2022trajectory, shen2024learning}, defined below for a pair of marginals. Given the transition law $\mathcal{K}$ of a reference process, an initial distribution $\mu$, and a final distribution $\nu$, the SB solution is the law of the stochastic process, which satisfies both marginal constraints, while minimizing relative entropy to $\mathcal{K}$:
\begin{equation}
\pi^* = \text{Proj}_{\Pi(\mu,\nu)}^{KL}(\mathcal{K}) = \argmin_{\pi \in \Pi(\mu,\nu)} D_{KL}(\pi \| \mathcal{K}).
\label{eq: schrodinger_bridge}
\end{equation}

The SB problem has also gained interest in the machine learning community, due to its connection to generative diffusion models. Score estimation of time-reversed SDEs is an instance of the SB problem \citep{de2021diffusion}, and SB solvers are now widely used in generative modeling due to their flexibility \citep{de2021diffusion, chen2021likelihood, shi2024diffusion, huang2024one}. 

\subsection{Identifiability of SDE Parameters}
The problem of non-identifiability occurs if there exists an alternative drift-diffusion pair, e.g. $(\tilde{A}, \tilde{H}) \neq (A, H)$ for linear-additive noise SDEs, which shares the same marginals. Equivalently, the same curve of marginals solves both Fokker-Planck equations across all observed times, given initial distribution $p_0$. Non-identifiability can occur due to many factors. For example, distinct SDEs can share the same stationary distribution, and would hence be non-identifiable from marginals if $p_0$ is initialized at equilibrium \citep{lavenant2021towards, shen2024learning}. Also, if the initial distribution is standard normal $p_0 \sim \mathcal{N}(0, I_d)$, then it would be impossible to distinguish between a pure Brownian motion, and an SDE with an additional rotational vector field \citep{shen2024learning}. We  overview three classical examples of non-identifiability in Section \ref{sec:Noniden_examples} of the Appendix.

Because of non-identifiability, previous works require assumptions about the underlying SDE to perform inference. The standard assumptions are that the drift is an irrotational vector field $\nabla \psi$ \citep{weinreb2018fundamental, lavenant2021towards, vargas2021solving, chizat2022trajectory, terpin2024learning} and that the diffusion is both known and given by isotropic Brownian motion $\sigma \mathrm{d}W_t$ \citep{weinreb2018fundamental, lavenant2021towards, vargas2021solving, chizat2022trajectory, shen2024multi, zhang2024joint, terpin2024learning}. It is also common to leverage additional perturbational data from interventions, and to impose regularization, to better learn the system  \citep{rohbeck2024bicycle, zhang2024joint,maddu2024inferring}. However, these constraints impose unrealistic conditions on the data. An irrotational drift field cannot model negative feedback loops or repressilator dynamics, which are common in biological data \citep{weinreb2018fundamental}. Similarly, isotropic diffusion cannot model variable noise scales or correlated noise structures, which commonly arise in real data \citep{mogensen2022graphical, rohbeck2024bicycle, santos2024learning}. While most works use a known diffusion estimate to infer drift, \citet{forrow2024consistent} infers diffusion, but requires a drift estimate. In both cases, the accuracy of the prior estimate is crucial, because misspecified diffusion typically leads to poor drift estimation and vice-versa \citep{beven1993dispersion,guan2025gradient}. This is problematic, since knowing an accurate estimate of drift or diffusion {\em a priori} is unrealistic in practice. Finally, we note that recent work estimates drift and diffusion, but does not consider the uniqueness of learned parameters, due to non-identifiability \citep{terpin2024learning, brogat2024learning}.

To advance the field of statistical inference from temporal marginals, we work to remove previous assumptions by determining conditions for identifiability and designing a method that leverages them. \citet{wang2024generator} showed that identifiability of a linear additive noise SDE from ground truth trajectories with fixed $X_0 \in \R^d$ is equivalent to a non-degenerate rank condition based jointly on $X_0$, the linear drift $A$ and the observational diffusion $GG^\top$. In contrast, our work studies identifiability of linear additive noise SDEs in the more general ``marginals-only" setting. We also demonstrate that identifiability from continuous observation can be guaranteed across all linear additive noise SDEs solely by checking the initial distribution. To the best of our knowledge, this has not been shown even in the case of trajectory-based observations.

\section{\rv{Identification of Linear SDEs from Temporal Marginals}}
\label{sec:identify}

\rv{In this section, we prove three identifiability results of increasing generality, which provide conditions for uniquely recovering the drift and diffusion of linear additive noise SDEs. Our theory demonstrates that SDE non-identifiability is only possible when the observed marginals share a generalized rotational symmetry. Furthermore, identifiability holds under generic conditions, since the set of SDE parameters that can preserve these symmetries over the course of the continuous evolution has Lebesgue measure $0$.}

\subsection{SDE Identifiability Under Isotropic Diffusion}
\label{sec: SDE_param_identifiability_isotropic}
We first consider the case of SDEs with linear drift and isotropic diffusion:
\begin{align}
    dX_t = AX_t \dt + \sigma \dWt,
    \label{eq: isotropic_linear_sde}
\end{align}
and we consider the question of when we can jointly identify the drift $A$ and diffusivity $\sigma^2$ from continuously observed marginals $(p_t: t \ge 0)$. It turns out that non-identifiability is intrinsically related to the marginals sharing a generalized rotational symmetry. We first define what we mean by a generalized rotation.
\begin{definition}[Generalized rotation]
\label{def:generalized_rotation}
Let $d \ge 2$. We define a $\Sigma$-generalized rotation in $\R^d$ as a matrix exponential $e^{B\theta}$, such that $\theta$ is the rotation angle and $B \in \R^{d \times d}$ is skew-symmetric with respect to $\Sigma \succeq 0$, i.e. $B\Sigma + \Sigma B^\top = 0$. 
\end{definition}
Note that classical rotations correspond to the case where we restrict $\Sigma = I_d$, since this enforces that the matrix $B$ is skew-symmetric. By allowing $\Sigma \neq I_d$, we enable rotations to be defined over anisotropic geometries. If $\Sigma \succ 0$ is full rank, then $\Sigma$-generalized rotations are equivalent to classical rotations in a transformed space, after lengths are modified by scaling (see Lemma \ref{lem:conjugate_rotation} in appendix). Otherwise, if $\Sigma \succeq 0$ is rank-deficient, then the condition $B\Sigma + \Sigma B^\top = 0$ forces $B$ to leave the range of $\Sigma$ invariant, and to be skew with respect to $\Sigma$ there, while its action on the kernel of $\Sigma$ is unconstrained.
We now extend this notion to random variables, by defining a distribution $X$ to be \emph{auto-rotationally invariant} (ARI) if its law is invariant to the action $e^{B\theta}$ for some nonzero matrix $B \in \R^{d \times d}$.
\begin{definition}[Auto-rotational invariance]
\label{def:full_ARI}
Let $d \ge 2$. We define a $d$-dimensional random variable $X$ to be auto-rotationally invariant (ARI) if there exists some square matrix $B \neq 0$, such that for all $\theta \in \R$, $e^{B\theta}X \sim X$. For the case $d=1$, we by convention consider mean-zero  Gaussians $\mathcal{N}(0,\sigma^2)$ as the only ARI distributions.
\end{definition}
We use the term \textit{auto-rotational} because it turns out that invariance of the law under the transformation $X \to e^{B\theta}X$ implies that $e^{B\theta}$ must be a generalized rotation (Definition \ref{def:generalized_rotation}), where the geometry $\Sigma$ is the distribution's second moment $\Sigma = \E[XX^\top]$  (see Lemma \ref{lem: lyapunov_condition}).
Auto-rotational invariance therefore requires that the law of the random variable is rotationally symmetric with respect to the geometry imposed by its own covariance structure. If $\Sigma = \E[XX^\top] \succ 0$ is full rank, then this is a highly restrictive property, since the density must be constant along orbits of $e^{B\theta}$, which each lie within an ellipsoid determined by $\Sigma$ (see Lemma \ref{lemma:level_sets}). For example, all mean zero Gaussians are auto-rotationally invariant (see Lemma \ref{lemma: gaussian_ARI}). Otherwise, if $\Sigma = \E[XX^\top] \succeq 0$ is rank-degenerate, then the distribution also satisfies auto-rotational invariance (see Lemma \ref{lemma: degenerate_X_isotropy}).


We now present our first identifiability result, which states that the SDE parameters $(A, \sigma^2)$ can be uniquely identified from continuously observed marginals $X_t \sim p_t$, unless each marginal is auto-rotationally invariant with respect to the same generalized rotation $e^{B\theta}$.
\begin{theorem}[Identifiability under isotropic diffusion]
Suppose that the observed \\ marginals $(p_t: t \ge 0)$ can be produced by multiple distinct parameter sets $(A,\sigma^2) \neq (\tilde{A},\tilde{\sigma}^2)$. Then, each marginal $X_t \sim p_t$ is auto-rotationally invariant with respect to the same generalized rotational action $e^{B\theta}$, i.e. for all $X_t \sim p_t$, we have that $e^{B\theta}X_t \sim X_t$ for all $\theta \in \R$.
\end{theorem}

\begin{proof}
\rv{First, suppose that distinct parameter sets $(A,\sigma^2) \neq (\tilde{A},\tilde{\sigma}^2)$ produce the same marginals $(p_t: t \ge 0)_{A,\sigma^2} =  (p_t: t \ge 0)_{\tilde{A},\tilde{\sigma}^2}$. Then, their Fokker-Planck evolution equations are equal:
\begin{align*}
    \partial_t p_t(x) &= -\nabla \cdot (Ax p_t(x)) + \frac{\sigma^2}{2}\Delta p_t(x) \\
    &= -\nabla \cdot (\tilde{A}x p_t(x)) + \frac{\tilde{\sigma}^2}{2}\Delta p_t(x).
\end{align*}
Without loss of generality, we may define the residual diffusivity $\bar{\sigma}^2 = \sigma^2-\tilde{\sigma}^2 \ge 0$ and the residual drift $\bar{A} = A- \tilde{A}$. By subtracting the equations, we obtain a stationary Fokker-Planck equation with the residual parameters
\begin{align}
    0&= -\nabla \cdot (\bar{A}x p_t(x)) + \frac{\bar{\sigma}^2}{2}\Delta p_t(x).
    \label{eq:residual_FP}
\end{align}
First, we note that if the residual diffusion $\bar{\sigma}^2 = \sigma^2-\tilde{\sigma}^2$ is nonzero, it follows from elliptic PDE theory that the residual Fokker-Planck equation \eqref{eq:residual_FP} has at most one stationary distribution \citep[Theorem 4.1.6, Example 4.1.8]{bogachev2022fokker}. Since each marginal $p_t$ solves \eqref{eq:residual_FP}, each marginal is in fact the stationary distribution $p_{eq}$ of a residual process with drift $\bar{A}$ and diffusivity $\bar{\sigma}^2$. Since the residual is a linear SDE, $p_{eq}$ is a Gaussian distribution $\mathcal{N}(0, \Sigma)$ with $\Sigma$ satisfying the Lyapunov equation $\Sigma \bar{A}^\top + \bar{A} \Sigma = -\bar{\sigma}^2 I_d$ \citep[Section 6.2.1]{gajic2008lyapunov}. Thus, if the diffusivities are unequal, then non-identifiability implies that all marginals coincide at the same mean-zero Gaussian distribution. By Lemma \ref{lemma: gaussian_ARI}, all such distributions satisfy auto-rotational invariance.}

\rv{From this analysis, we may also deduce that if the marginals $(p_t)_{t \ge 0}$ are not all the same distribution, then non-identifiability is only possible if both SDEs share the same diffusivity $\sigma^2$, and are thus only distinct in their drifts. In this case, the residual Fokker-Planck equation reduces to the continuity equation:
\begin{align}
    0&= -\nabla \cdot (\bar{A}x p_t(x)).
\end{align}
By Liouville's theorem (see \citet[Theorem 1]{neklyudov2021deterministic}), each observed marginal $X_t \sim p_t$ is stationary under the linear ODE $\dot{x}=\bar{A}x$. Thus, for each marginal $X_t$, we have
\begin{align}
    e^{\bar{A}\theta}X_t \sim X_t \quad \forall \theta \in \R.
\end{align}
We conclude that each marginal $X_t \sim p_t$ is auto-rotationally invariant with respect to the same matrix $B=\bar{A}$.}
\end{proof}

\rv{We make a couple of remarks about this result. First, we recall that observing marginals at equilibrium is a well-documented source of non-identifiability, since equilibrium only depends on the drift to diffusion ratio (see \citet{lavenant2021towards} and \citet[Proposition 4.1]{guan2025gradient}). Our result shows that under isotropic diffusion, observing marginals that vary in time is enough to determine the true diffusivity $\sigma^2$. This follows from the same elliptic PDE argument underlying a recent identifiability result by \citet{guan2025gradient}, which showed that observation at equilibrium is the only source of non-identifiability of the SDE parameters for gradient-flow SDEs, which are defined by conservative drift $-\nabla \Psi$ and isotropic diffusivity $\sigma^2$.}

\rv{However, in contrast to the gradient-flow setting, linear SDEs may have rotational drift, which introduces undetectable rotation as a second source of non-identifiability. Indeed, a classic non-identifiability example from the literature is given by the systems \citep{hashimoto2016learning, shen2024learning}:
\begin{align}
    \mathrm{d}X_t &= \mathrm{d}W_t, &X_0 \sim \delta_0\\
    \mathrm{d}Y_t &= \begin{bmatrix} 0 & 1 \\ -1 & 0  \end{bmatrix}Y_t \: \mathrm{d}t
    + \mathrm{d}W_t, &  Y_0 \sim \delta_0,
\end{align}
where the residual drift is the skew-symmetric matrix $\bar{A}=\begin{bmatrix}
    0 & 1 \\
    -1 & 0
\end{bmatrix}$,
which produces a rotational action that is indistinguishable at the level of the marginals $p_t \sim \mathcal{N}(0, t I_d)$. We present examples with more general rotational geometries in \cref{sec:Noniden_examples}. Our result shows that 
this failure mode, characterized by shared invariance under a generalized rotational action, is the only other source of non-identifiability besides observation at equilibrium, for linear SDEs with isotropic diffusion.}

\rv{\subsection{SDE Identifiability Under Anisotropic Diffusion}}
\label{sec: SDE_param_identifiability_anisotropic}

\rv{We now consider the more general case of linear additive noise SDEs with possibly anisotropic diffusion:
\begin{align}
    dX_t = AX_t \dt + G \dWt,
    \label{eq: anisotropic_linear_sde}
\end{align}
and we derive identifiability conditions for determining the SDE parameters $(A, H = GG^\top)$ from observed marginals $(p_t: t \ge 0)$. We note that anisotropic diffusion significantly complicates identifiability analysis, since the diffusion level can vary in each dimension. Indeed, anisotropic diffusion can produce non-identifiability such that the distributions are stationary in a subset of components, while otherwise exhibiting transience. Since equilibrium distributions for linear additive noise are mean-zero Gaussians \citep[Sections 4.2 and 4.4]{pavliotis2014stochastic}, we can illustrate this by considering a `Gaussian pancake' initial distribution $X_0 \sim p_0$, where
$X_0^{(1)} \sim \mathcal{N}(0, 1)$ and the remaining components
$X_0^{(2:d)}$ are initialized independently. To produce non-identifiability, we consider two different univariate dynamics that preserve $X_t^{(1)} \sim \mathcal{N}(0, 1)$ and we decouple these first component dynamics from the rest of the system:
\begin{align}
\mathrm{d} X_t &= \begin{bmatrix} -\frac{1}{2} & \mathbf{0} \\ \mathbf{0} & A_{\setminus1} \end{bmatrix} X_t \,\mathrm{d}t + \begin{bmatrix} 1 & \mathbf{0} \\ \mathbf{0} & G_{\setminus1}  \end{bmatrix} \mathrm{d}W_t, \quad X_0 \sim p_0 \\
\mathrm{d}Y_t &= \begin{bmatrix} -2 & \mathbf{0} \\ \mathbf{0} & A_{\setminus1} \end{bmatrix} Y_t \,\mathrm{d}t + \begin{bmatrix} 2 & \mathbf{0} \\ \mathbf{0} & G_{\setminus1}  \end{bmatrix} \mathrm{d}W_t, \quad Y_0 \sim p_0.
\end{align}
These SDEs are non-identifiable from marginals since they are both stationary in the first component, while being trivially identical in other components since they share the same independent complementary dynamics $(A_{\setminus1}, G_{\setminus1})$.}

\rv{While we saw that Gaussianity is directly related to non-identifiability in the previous setting with isotropic diffusion, the above construction shows that the addition of Gaussian slices can lead to non-identifiability if we permit anisotropic diffusion. Thus, non-identifiability is feasible even if a marginal $X_t \sim p_t$ fails to be auto-rotationally invariant, as defined in Definition \ref{def:full_ARI}. However, by recalling the classical fact that Gaussians are the only random variables that have finitely many nonzero cumulants \citep{marcinkiewicz1939propriete}, and moreover, that the only nonzero cumulant of a mean-zero Gaussian is its covariance, we may consider a natural relaxation of auto-rotational invariance. Instead of defining the invariance relation over the full law, we can define it on all cumulant orders, except $n=2$.}

\rv{\begin{definition}[Weak auto-rotational invariance]
\label{def:weak_ari}
Let $X$ be a $d$-dimensional r.v, and denote by $T_n^X(v, \ldots, v)$ the order-$n$ tensor defining the $n$th cumulant of $X$, projected in the direction of $v \in \R^d$, i.e. the $n$th cumulant of the univariate r.v. $v^\top X$. Then, $X$ is weakly auto-rotationally invariant if there exists some nonzero matrix $B \in \R^{d \times d}$, such that the mean and higher order cumulants $n \ge 3$ obey the invariance:
\begin{align}
    T_n^X(e^{B^\top \theta}v, \ldots,e^{B^\top \theta}v) = T_n^X(v, \ldots, v) \quad \forall n \in \N \setminus \{2\}, \forall  v \in \R^d, \forall \theta \in \R.
    \label{eq:new_ARI}
\end{align}
\end{definition}}
For convenience, we sometimes drop the ``weak'' qualifier and simply refer to weak ARI as ARI, but
we make the distinction explicit in formal statements and proofs. Indeed, as shown in Lemma \ref{lemma: bounded_weak_equals_full}, if the support of the
random variable is bounded, then weak ARI (Definition \ref{def:weak_ari}) is in fact equivalent to
full ARI (Definition \ref{def:full_ARI}).
Intuitively, weak ARI omits the constraint at order $n=2$, which enables the addition of Gaussian slices. While we saw that degenerate distributions, which are supported in a proper subspace of $\R^d$ are always ARI (and hence weakly ARI), weak ARI also encompasses random variables that are Gaussian in a proper subspace of $\R^d$. In Figure \ref{fig:ARI_distributions}, we visualize three ARI distributions as well as a Gaussian pancake distribution, which is only weakly ARI.


\begin{figure}[h!]
\centering
\begin{minipage}{0.4\textwidth}
        \centering
        \begin{subfigure}[b]{\textwidth}
            \centering
            \includegraphics[width=\textwidth]{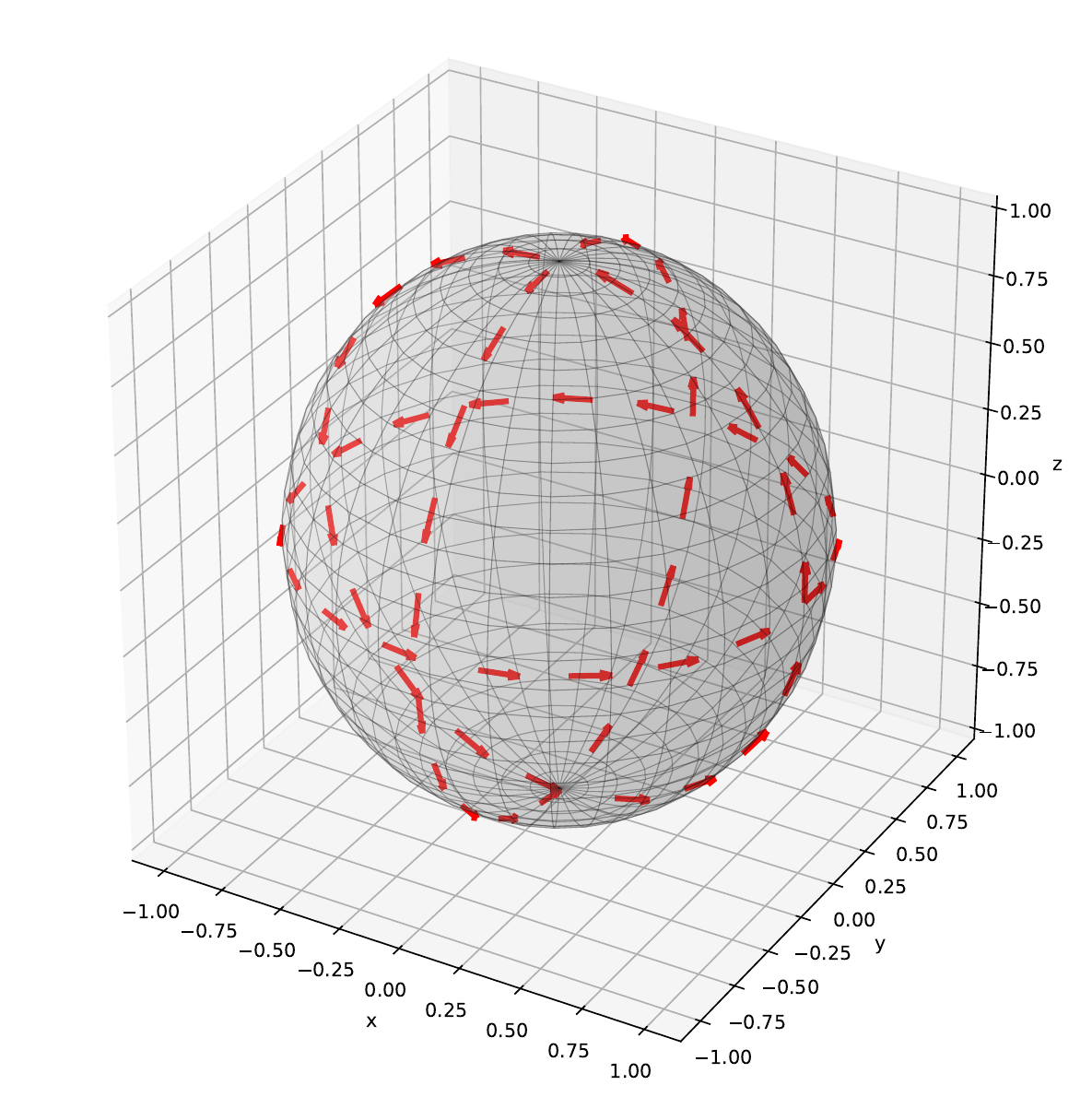}
            \caption{$X_0$ is spherical and invariant under all classical rotations.}
            \label{fig:sphere_vector_field}
        \end{subfigure}
    \end{minipage}%
    \hfill
    \begin{minipage}{0.4\textwidth}
        \centering
        \begin{subfigure}[b]{\textwidth}
            \centering
            \includegraphics[width=\textwidth]{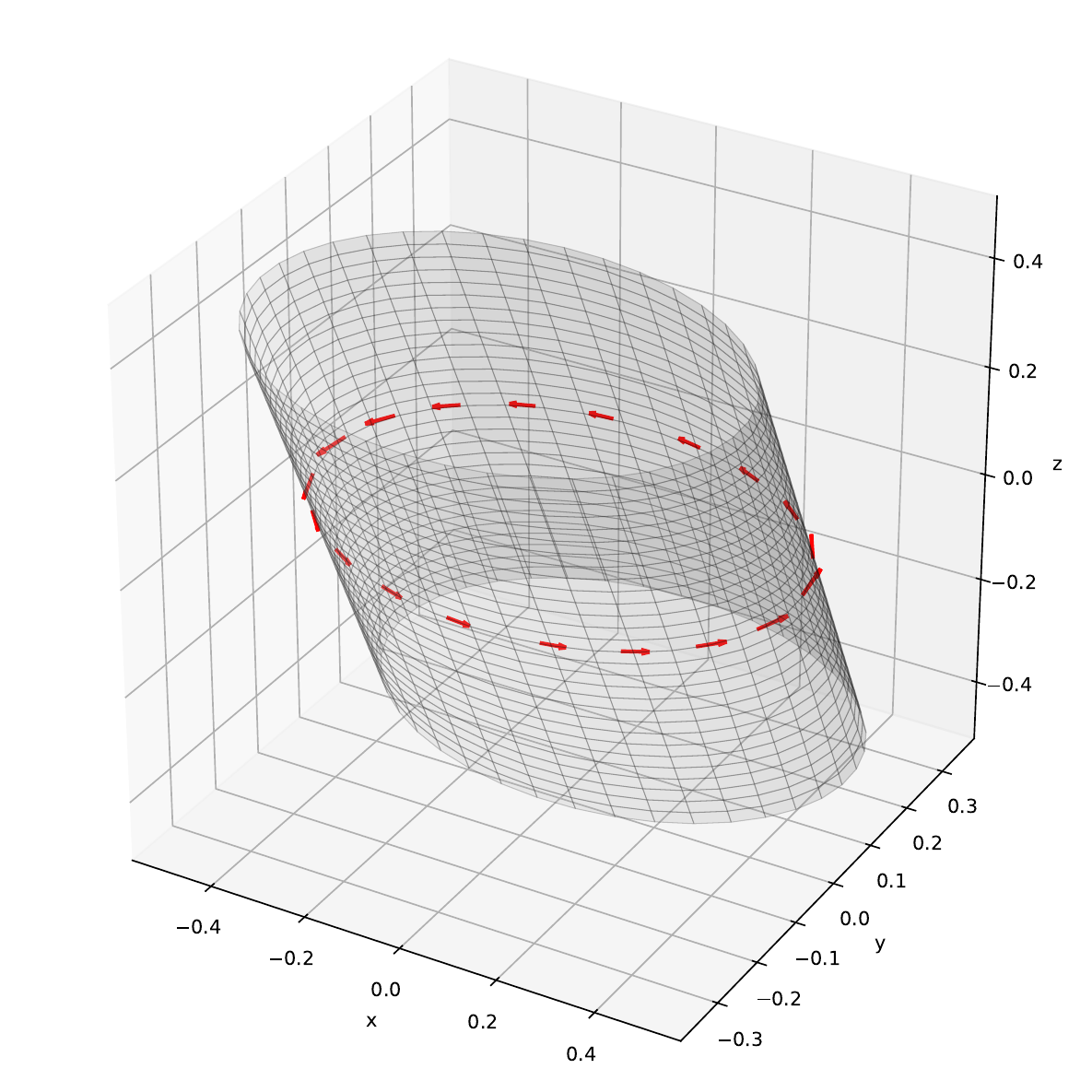}
            \caption{$X_0$ is cylindrical (up to scaling), and is invariant under rotation in the $xy$ plane after transformation.}
            \label{fig:cylinder_vector_field}
        \end{subfigure}
    \end{minipage}
    
    \begin{minipage}{0.4\textwidth}
        \centering
        \begin{subfigure}[b]{\textwidth}
            \centering
            \includegraphics[width=\textwidth]{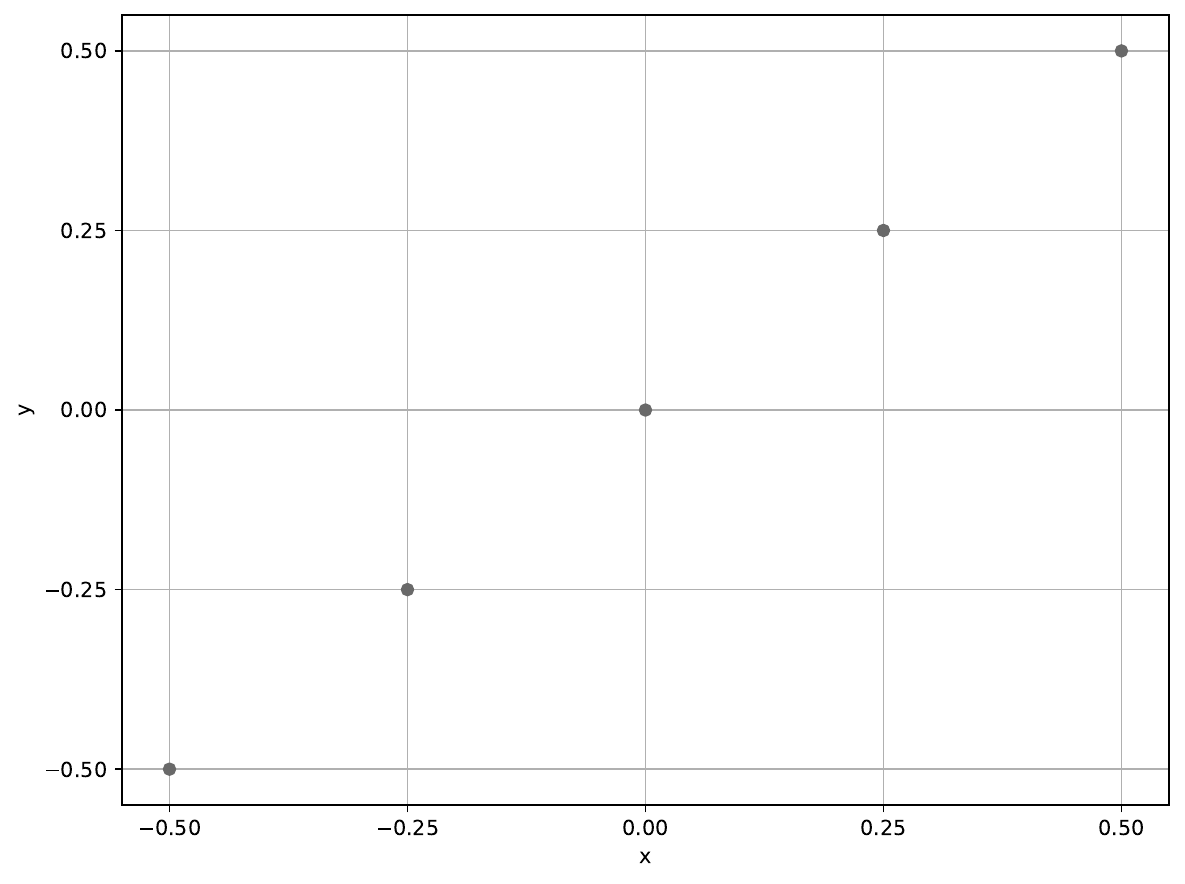}
            \caption{$X_0$ is degenerate and its support lies in the nullspace of $B=\begin{bmatrix}
                1 & -1 \\ 
                0 & 0
            \end{bmatrix}$ }
            \label{fig:degenerate_rv}
        \end{subfigure}
    \end{minipage}
    \hfill
    \begin{minipage}{0.4\textwidth}
        \centering
        \begin{subfigure}[b]{\textwidth}
            \centering
            \includegraphics[width=\textwidth]{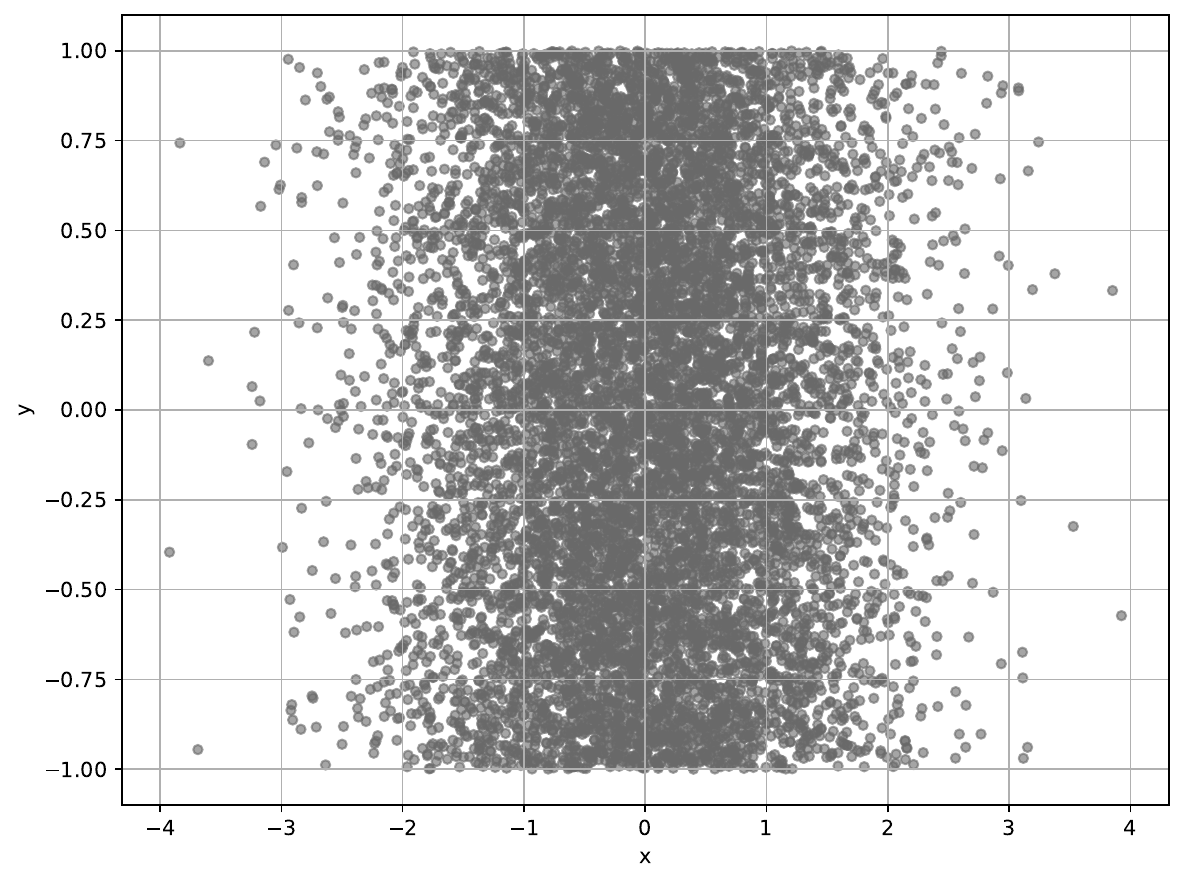}
            \caption{$X_0$ is a Gaussian pancake, which is mean $0$ univariate Gaussian in $x$.}
            \label{fig:gaussian_pancake_rv}
        \end{subfigure}
    \end{minipage}%
\caption{We visualize four distributions that satisfy auto-rotational invariance. The first three distributions obey strong auto-rotational invariance (Definition \ref{def:full_ARI}), while the Gaussian pancake is weakly auto-rotationally invariant (Definition \ref{def:weak_ari}).}
\label{fig:ARI_distributions}
\end{figure}

\rv{We now present our second identifiability theorem, which proves that non-identifiability from the marginals of a linear additive noise SDE is only possible if the initial marginal $X_0 \sim p_0$ is weakly auto-rotationally invariant, and furthermore, that there exist non-identifiable systems for any weakly ARI initial distribution $X_0$. Intuitively, non-identifiability is only feasible if there is a symmetry to exploit, and if the initial distribution has such a symmetry, then it is susceptible to non-identifiability.}

\rv{\begin{theorem}[Identifiability for linear additive noise SDEs]
Let $p_0$ be a probability distribution on $\R^d$ with a well-defined moment generating function, and let $$\Theta = \{(A,H): A \in \R^{d \times d}, G \in \R^{d \times m}, H = GG^\top \}$$ denote the set of drift-diffusion pairs under a linear additive noise model \eqref{eq: linear_additive_noise_SDE}. It follows that every parameter configuration $(A,H) \in \Theta$ induces a unique set of marginals $(p_t : t \ge 0)_{A,H}$ from initialization $X_0 \sim p_0$, if and only if $X_0$ is not weakly auto-rotationally invariant (Definition \ref{def:weak_ari}).
\label{thm: identifiability}
\end{theorem}}
\begin{proof}
\rv{We first prove that non-identifiability can only occur if $X_0$ is weakly ARI. Suppose that two linear additive noise SDEs have distinct parameters $(A, H=GG^\top) \neq (\tilde{A}, \tilde{H}=\tilde{G}\tilde{G}^\top)$, yet equal marginals for all times $t \ge 0$. We note that these have analytic solutions (see \cref{sec:linear_additive_noise}):
\begin{align*}
    X_t &= e^{At}X_0 + \int_0^t e^{A(t-s)}GdW_s\\
    Y_t &= e^{\tilde{A}t}X_0 + \int_0^t e^{\tilde{A}(t-s)}\tilde{G}dW_s,
\end{align*}
where the initial distribution $X_0$ is shared. }

\rv{Then, since the m.g.f is well-defined, equality in distribution $X_t \sim Y_t$ is equivalent to equality in all cumulants between $X_t$ and $Y_t$:
\begin{align}
    T_n^{X_t}(v, \ldots, v) = T_n^{Y_t}(v, \ldots, v) \qquad \forall v \in \R^d, \forall t \ge 0,
\end{align}
where $T_n^{X}(v, \ldots, v)$ denotes the $n$th cumulant of $X$, projected onto $v \in \R^d$. By matching cumulants of all orders, non-identifiability requires that the following infinite system of equations is satisfied for all $v \in \R^d$ and $t \ge 0$:
\begin{align}
    \begin{cases}
T_n^{X_0}(e^{A^\top t}v, \ldots, e^{A^\top t}v)&=T_n^{X_0}(e^{\tilde{A}^\top t}v, \ldots, e^{\tilde{A}^\top t}v) \qquad n \neq 2,\\
T_2^{X_0}(e^{A^\top t}v, e^{A^\top t}v) + v^\top\Sigma_t v &= T_2^{X_0}(e^{\tilde{A}^\top t}v, e^{\tilde{A}^\top t}v) + v^\top\tilde{\Sigma}_t v \qquad n=2,
\end{cases}
\label{eq:all_cumulant_eqs}
\end{align}
Note that the Ito integral terms contribute only to the second order cumulant, since they are mean zero Gaussians \citep{marcinkiewicz1939propriete}. We are also able to move the matrix exponential into the arguments, due to the homogeneity of cumulant tensors. We first focus on all cumulants of order $n \neq 2$ and note that they imply the identity:
\begin{align}
    T_n^{X_0}(e^{\bar{A}^\top t}v, \ldots, e^{\bar{A}^\top t}v)=T_n^{X_0}(v, \ldots, v) \qquad n \neq 2
\label{eq:higher_order_cumulants}
\end{align}
Indeed, if we take the time derivative at $t=0$, we obtain
\begin{align*}
    \frac{d}{dt}\bigg\rvert_{t=0}T_n^{X_0}(e^{A^\top t}v, \ldots, e^{A^\top t}v) &=\frac{d}{dt}\bigg\rvert_{t=0}T_n^{X_0}(e^{\tilde{A}^\top t}v, \ldots, e^{\tilde{A}^\top t} v)\\
    \implies nT_n^{X_0}(A^\top v,v, \ldots, v) &= nT_n^{X_0}(\tilde{A}^\top v, v,\ldots, v)\\
    \implies T_n^{X_0}(\bar{A}^\top v, v, \ldots, v) &=0
\end{align*}
where $\bar{A}=A-\tilde{A}$. Since the above holds for all $v \in \R^d$, we can apply it to any vector in the set $\{u(t)=e^{\bar{A}^\top t}v: t \ge 0, v \in \R^d\}$, which yields
\begin{align*}
    T_n^{X_0}(\bar{A}^\top u(t), u(t), \ldots, u(t))&=0 \\
    \implies \frac{d}{dt}T_n^{X_0}(u(t), \ldots, u(t))=nT_n^{X_0}(\bar{A}^\top u(t), u(t), \ldots, u(t))&=0.
\end{align*}
Thus, along the trajectory $u(t)=e^{\bar{A}^\top t}v$, for all $n \neq 2$, the $n$th cumulant of $X_0$, projected onto $u(t)$, is constant in time, so the equations for matching the higher order cumulants reduce to
\begin{align*}
    T_n^{X_0}(e^{\bar{A}^\top t}v, \ldots, e^{\bar{A}^\top t}v)= T_n^{X_0}(v, \ldots, v) \qquad \forall  n \neq 2,
\end{align*}
which is precisely the condition for weak auto-rotational invariance \eqref{eq:new_ARI} with $B = \bar{A}$.}

We now show that if $X_0$ satisfies the weak ARI condition \eqref{eq:new_ARI} for some matrix $B \neq 0$, then there exist systems with distinct parameter sets $(A, H=GG^\top) \neq (\tilde{A}, \tilde{H}=\tilde{G}\tilde{G}^\top)$, which produce the same marginals for each time, i.e. $(p_t : t \ge 0)_{A,H} = (p_t : t \ge 0)_{\tilde{A},\tilde{H}}$. We denote the initial second moment $\Sigma = \E[X_0X_0^\top]$, and consider the rank-degenerate and full rank cases separately.

First, suppose $\Sigma$ is rank-degenerate, so that some $u \neq 0$ has $\Sigma u = 0$ and hence
$\E[(u^\top X_0)^2] = u^\top\Sigma u = 0$, giving $u^\top X_0 = 0$ a.s. For our non-identifiable counterexample, we consider base parameters 
$A = -\alpha I$ and $H = 2\alpha\Sigma$. The idea is that since the initial distribution and the noise are supported in the hyperplane
$u^\perp$ ($\ran G = \ran H = \ran\Sigma
\subseteq u^\perp$), the marginals will always stay in this subspace. This allows us to alter the drift in the direction of $u$, without changing the evolution. 

In particular, we can pick any $w \neq 0$ and set the drift offset $\bar{A} = A- \tilde{A} = wu^\top \neq 0$, so that
$\ker\bar{A} = u^\perp$, and take the same diffusion $\tilde{G} = G \implies \tilde{H} = H$. Then $\bar{A}X_0 = w(u^\top X_0) = 0$ a.s. and $\bar{A}G = 0$. Then, since $-\alpha I$ commutes with $\bar{A}$, expanding the exponential $e^{\tilde{A}t}X_0$ gives $e^{\tilde{A}t}X_0 = e^{-\alpha t}X_0$ a.s. Similarly, $e^{\tilde{A}(t-s)}G = e^{-\alpha(t-s)}G$. Driven by the same Brownian motion, the two systems therefore produce identical
paths, despite having a drift offset of $\bar{A}$. We provide one such construction in Example \ref{eq: non_identifiability_wang}.

To design a non-identifiable system in the non-degenerate case, where $\Sigma \succ 0$ is full rank, we consider an SDE with base parameters 
\begin{align*}
    A &= -\alpha I \\
    H&=2\alpha \Sigma
\end{align*}
Then, to define the second system, we set the drift difference to be the matrix for which $X_0$ is weakly ARI, i.e. $\bar{A} = A - \tilde{A} = B$, such that the second system has parameters
\begin{align*}
    \tilde{A} &= A - \bar{A} = -\alpha I - \bar{A}\\
    \tilde{H} &= H - \bar{H}= 2\alpha \Sigma + \bar{A} \Sigma + \Sigma \bar{A}^\top
\end{align*}
Note that since $\Sigma \succ 0$, then we can pick $\alpha$ arbitrarily large to ensure that $\tilde{H}$ has positive spectrum, and is a valid p.s.d diffusion matrix. Since $A=-\alpha I$ is isotropic, it follows that we have the analytic solutions 
\begin{align*}
    X_t &= e^{-\alpha t}X_0 + \mathcal{N}(0, \Sigma_t) \qquad \text{s.t. } \Sigma_t = \int_0^t e^{A(t-s)}He^{A^\top (t-s)}ds \\
    Y_t &= e^{-\alpha t}e^{-\bar{A}t}X_0 + \mathcal{N}(0, \tilde{\Sigma}_t) \qquad \text{s.t. } \tilde{\Sigma}_t = \int_0^t e^{\tilde{A}(t-s)}\tilde{H}e^{\tilde{A}^\top (t-s)}ds.
\end{align*}
Then, note that non-identifiability holds if and only if $X_t \sim Y_t$ for all $t \ge 0$. Equivalently, the cumulants of all orders $n$ match for all times and projection vectors $v \in \R^d$:
\begin{align}
    T_n^{(X_t)}(v, \ldots, v) = T_n^{(Y_t)}(v, \ldots, v)\qquad \forall n \in \N, \forall v \in \R^d, \forall t \ge 0
    \label{eq:non_iden_all_cumulants_match}
\end{align}
To show that the above holds, we consider the cases $n \neq 2$ and $n=2$ separately. Matching all cumulant orders $n \neq 2$ reduces to showing that $e^{-\alpha t}X_0$ and $e^{-\alpha t}e^{-\bar{A}t}X_0$ have the same higher order cumulants. This follows directly from $X_0$ being weakly ARI with respect to $\bar{A}$:
\begin{align*}
T_n^{e^{-\alpha t}X_0}(v, \ldots, v) &= e^{-n\alpha t}T_n^{X_0}(v, \ldots, v) \\
&\overset{\eqref{eq:new_ARI}}{=} e^{-n\alpha t}T_n^{X_0}(e^{-\bar{A}^\top t}v, \ldots, e^{-\bar{A}^\top t}v)
\\ &= T_n^{e^{-\alpha t}e^{-\bar{A} t}X_0}(v, \ldots, v).
\end{align*}
Then, to show that the cumulants match for order $n=2$, we note that this corresponds to equality in covariance. Since the case $n \neq 2$ above includes $n=1$, the two systems already share the same mean at all times, so it suffices to show that their second moments agree. We in fact constructed the systems so that the second moment is constant in time, i.e. $\E[X_tX_t^\top] = \Sigma = \E[Y_tY_t^\top]$ for all $t \ge 0$. Indeed, the second moment of a linear additive noise SDE solves the autonomous Lyapunov ODE $\dot{M}_t = AM_t + M_tA^\top + H$, so it suffices to check that $M_t = \Sigma$ is an equilibrium solution. With $A = -\alpha I$, $H = 2\alpha\Sigma$, we obtain
\begin{align*}
    A\Sigma + \Sigma A^\top + H &= -2\alpha \Sigma + 2\alpha \Sigma = 0\\
    \tilde{A}\Sigma + \Sigma\tilde{A}^\top + \tilde{H}&= (A\Sigma + \Sigma A^\top + H) -\left(\bar{A}\Sigma + \Sigma\bar{A}^\top + \bar{H} \right)\\
    &= -\left(\bar{A}\Sigma+ \Sigma\bar{A}^\top + \bar{H} \right) = 0.
\end{align*}
Thus, all cumulants match between both SDEs at all times, and we conclude that they are non-identifiable.
\end{proof}

\rv{\cref{thm: identifiability} shows that weak auto-rotational invariance of the initial distribution is a necessary and sufficient condition for non-identifiability to be feasible. We note that by re-indexing the index of the initial time, it follows that identifiability of the SDE parameters is guaranteed if any of the marginals in the evolution is not weakly ARI. Furthermore, since the SDE parameters are constant in time, non-identifiability not only implies that each marginal is weakly ARI, but also that each marginal is weakly auto-rotationally invariant with respect to the same matrix $B$ in the condition \eqref{eq:new_ARI}.  In other words, non-identifiability requires that a symmetry invariance is shared across all marginals. We now determine the prevalence of these shared symmetries, in order to quantify the risk of non-identifiability. Below, we prove that the SDE parameters are generically identifiable, even if $X_0$ is weakly ARI. In particular, we show that the curve of covariances $(\cov(X_t): t \ge 0)$ uniquely identifies the true SDE parameters $(A,H)$ for almost every parameter configuration.}

\rv{\begin{theorem}[Generic identifiability for linear additive noise SDEs]
\label{thm:stronger}
Suppose $C_0 = \cov(X_0)$ is the covariance of an initial distribution $X_0 \sim p_0$. Then, for Lebesgue-almost every $(A,G) \in \R^{d \times d} \times \R^{d \times m}$ under the linear additive noise model, the covariance curve $(C_t:t\ge 0)_{A,G}$ uniquely determines the SDE parameters $(A,H=GG^\top)$. Thus, the SDE parameters $(A,H)$ are generically identifiable under continuous observation.
\end{theorem}}

\begin{proof}
We wish to prove that the map $(A,G) \to (C_{t}:t \ge 0)_{A,G}$ uniquely determines $(A,H=GG^\top)$ on $(\R^{d \times d} \times \R^{d \times m}) \setminus \mathcal{N}$, where $\mathcal{N}$ has Lebesgue measure $0$. First, we note that if two different parameter configurations trace out the same curve of covariances, i.e.
$(C_{t}:t \ge 0)_{A,G} = (C_{t}:t \ge 0)_{\tilde{A},\tilde{G}}$,
then by Lemma \ref{lemma:non_iden_imply_diff_drift}, their residual parameters
$\bar{A}=A-\tilde{A}$ and
$\bar{H}=GG^\top-\tilde{G}\tilde{G}^\top$
satisfy the Lyapunov equation
\begin{align}
    \bar{A}C_t + C_t \bar{A}^\top = -\bar{H}.
\end{align}
Thus, if we consider a collection of $D$ nonzero times $\{t_i\}_{i=1}^D$, and define $\Delta C_i= C_{t_i} - C_{0} = \cov(X_{t_i}) - \cov(X_0)$, then we can subtract the above equation at times $t=t_i$ and $t=0$ to obtain
\begin{align}
    \bar{A}\Delta C_i + \Delta C_i \bar{A}^\top = 0 \qquad \forall i =1, \ldots ,D. 
    \label{eq:cov_diff_2}
\end{align}
The main fact that we would like to show is that up to a set $\mathcal{N}$ of measure $0$, the condition $(C_{t}:t \ge 0)_{A,G} = (C_{t}:t \ge 0)_{\tilde{A},\tilde{G}}$ forces the residual drift $\bar{A}$ to be $0$. From this, the drifts would be equal, and it would follow that the corresponding diffusion matrices $H=GG^\top$ and $\tilde{H}=\tilde{G}\tilde{G}^\top$ must also be equal, since otherwise, the curve of covariances would deviate (see equations \eqref{eq:deriv_cov_1} and \eqref{eq:deriv_cov_2} in the proof of Lemma \ref{lemma:non_iden_imply_diff_drift}).

To prove the above fact, it suffices to show that the covariance differences $\{\Delta C_i\}_{i=1}^D$ span $\mathrm{Sym}_d$, since it would then follow that $\bar{A}=0$ is the only solution to \eqref{eq:cov_diff_2}. Indeed, assume that the covariance differences span $\mathrm{Sym}_d$. Then, since $\Id \in \mathrm{Sym}_d$, there exists a linear combination of \eqref{eq:cov_diff_2}, which produces $\bar{A} + \bar{A}^\top =0$, i.e., $\bar{A}$ is skew-symmetric. Thus, we can rewrite \eqref{eq:cov_diff_2} as $\bar{A}\Delta C_i=\Delta C_i \bar{A}$. Then, since by hypothesis, every diagonal matrix is in the span of the covariance differences, it follows that $\bar{A}$ commutes with every diagonal matrix. Taking $\Lambda = \mathrm{diag}(\lambda_1, \ldots, \lambda_d)$ with distinct entries, the identity $\bar{A}\Lambda = \Lambda\bar{A}$ reads $\bar{A}_{ij}\lambda_j = \lambda_i\bar{A}_{ij}$ entrywise, forcing $\bar{A}_{ij} = 0$ whenever $i \neq j$. Hence $\bar{A}$ is diagonal, and therefore symmetric. However, $\bar{A}=0$ is the only matrix that is both symmetric and skew-symmetric.

We now show that the covariance differences span $\mathrm{Sym}_d$. For this, note that $\mathrm{Sym}_d$ is a $\frac{d(d+1)}{2}$ dimensional real subspace, since symmetric matrices are determined by the upper triangular block (with main diagonal). We therefore consider \(D=\frac{d(d+1)}{2}\) covariance differences and apply the half-vectorization operator \(\operatorname{vech}\), which stacks the entries on and below the main diagonal of a symmetric matrix. We then form the \(D \times D\) matrix
\[
M(A,G)=
\begin{bmatrix}
\operatorname{vech}(\Delta C_1)^\top \\
\vdots \\
\operatorname{vech}(\Delta C_D)^\top
\end{bmatrix}.
\]
Indeed, without loss of generality, we may regard the covariance differences with respect to $C_0$ as functions of $(A,G)$. To conclude the proof, we show that $\det(M(A,G)) \neq 0$ for Lebesgue-almost every $(A,G) \in \R^{d \times d} \times \R^{d \times m}$. To do so, we use the fact that a real analytic function is either identically zero or nonzero Lebesgue-almost everywhere \citep[Proposition 0]{mityagin2015zero}, and apply this to $\det(M(A,G))$. Indeed,
\[
C_t
=
e^{At}C_0e^{A^\top t}
+
\int_0^t e^{As}GG^\top e^{A^\top s}\,ds
\]
is analytic in $(A,G)$, so it follows that the determinant of $M(A,G)$ is also analytic. We therefore just need to show that for a given initial covariance $C_0$, there exists a single parameter configuration $(A,G)$ such that $\det(M(A,G))\neq0$. We provide an example in Proposition~\ref{prop:exhibit} of the Appendix.

\end{proof}


To provide some intuition, we note that the converse direction of \cref{thm: identifiability} was shown by adversarially choosing SDE parameters that leave the covariance stationary, thus enabling non-identifiability by fixing the geometry across all marginals. However, even if the initial distribution $X_0$ is weakly ARI, \cref{thm:stronger} shows that almost every parameter configuration will produce an evolution that breaks the symmetries required for non-identifiability.

We conclude with a few remarks about the practical implications of our identifiability results and connections to related theory. First, we note the limitation that our identifiability results assume continuous observation of the marginals $X_t \sim p_t$. We note that the assumption of continuous observation cannot generally be relaxed, due to the aliasing problem. For example, let $p_0 = \textrm{Unif}[-1,1]^2$ be a uniform square, which is clearly not auto-rotationally invariant. If we consider a rotational action from the skew-symmetric drift $A=\begin{bmatrix} 0 & \gamma \\ -\gamma & 0\end{bmatrix}$ and $H=0$, and only sample the marginals at discrete times $\{0, \Delta t, \cdots , k\Delta t, \cdots \}$, then it would be impossible to determine whether the square has rotated by $\theta$ in between two times, or by  $2\pi k + \theta$ for any $k \in \N$. This is precisely the Nyquist aliasing phenomenon in signal processing: discrete sampling at rate $\frac{1}{\Delta t}$ cannot resolve frequencies above $\frac{\pi}{\Delta t}$. However, we note that aliasing counterexamples require a strict synchronicity, and moreover that generic SDE parameter configurations will also alter the covariance geometry, as we showed in \cref{thm:stronger}. We therefore expect that identifiability results can be proven in the finite-sample setting under certain assumptions, but note that new techniques must be used to prove them, since we cannot directly leverage the differential structure of the evolution equations.

Second, by proving that non-identifiability requires auto-rotational invariance, which defines a shared invariance structure on the cumulants \eqref{eq:new_ARI}, our results show that higher order cumulants provide crucial information for identifiability. This bears interesting parallels to identifiability theory in causal discovery for linear non-Gaussian acyclic models (LiNGAM) on static data, where the presence of higher order cumulants from non-Gaussian noise enables the causal graph to be identified in the absence of latent confounders \citep{shimizu2006linear}. In each of these settings, we seek identifiability from data composed of linear effects with additive independent noise. In the static case, we are unable to resolve cycles in the graph, and imposing non-Gaussianity in the independent noise provides sufficient cumulant information to identify the DAG. Meanwhile, for linear SDEs, the temporal component allows us to identify cycles in the graph. Indeed, as we will see in the next section, the drift matrix ($A$) trivially determines all directed edges in the endogenous set. However, the noise is necessarily Gaussian, since the diffusion is driven by Brownian motion. Thus, identifiability in both settings is related to higher order cumulants. In the static setting, the source of the higher order cumulants comes from non-Gaussian noise, whereas in the SDE setting, the Brownian noise is Gaussian, and thus, the higher order cumulants come from the initial distribution.

\section{Causal Graph Identification from SDE Parameters}
\label{sec:causal_graph}

\rv{In this section, we establish theory for recovering the causal graph of an additive noise SDE from its identified drift and (observational) diffusion parameters. Thus, given our identifiability results from the previous section, these insights demonstrate that under mild conditions, we can theoretically identify the SDE and the causal graph of a system that is governed by a linear additive noise SDE, directly from its temporal marginals $(p_t)_{t \ge 0}$.}

\rv{
First, we review the definition of a dynamic structural causal model (DSCM), which we adapt from \citet[Definition 1]{boeken2024dynamic}.}

\rv{\begin{definition}[DSCM]
Let $V$ be a set of $d$ endogenous variables $\{ X^{(1)}, \dots, X^{(d)}\}$ and let $W$ be a set of exogenous variables, $\{ X^{(d+1)},\dots, X^{(d+m)}\}$, e.g. noise sources that influence the endogenous variables. Then, a DSCM is defined for each of the $d$ endogenous variables, such that 
\begin{align}
    X_t^{(j)} &= X_0^{(j)} + \int_0^t a_j(s, X^{\alpha(j)}_s) \, \mathrm{d}h_j(s, X^{\beta(j)}_s),
    \label{eq: explicit_SEM_SDE}
\end{align}
where $\alpha(j)\subset V \cup W$ represents the variables used as arguments for the integrand $a_j(t, X_t)$, and $\beta(j) \subset V \cup W$ represents the variables used as arguments for the integrator $h_j(t, X_t)$. For a $d$-dimensional additive noise SDE \eqref{eq: additive_noise_SDE} driven by an $m$-dimensional Brownian motion, the exogenous set is simply $W=\{X^{(d+1)}, \dots, X^{(d+m)}\}=\{W^{(1)}, \dots, W^{(m)}\}$, and the DSCM is defined by
\begin{align}
    X_t^{(j)}&= X_0^{(j)} + \int_0^t b_j(X^{\alpha(j)}_s) \, \mathrm{d}s +  \sum_{k=1}^{m} G_{j, k}W_t^{(k)}.
    \label{eq: DSCM_additive_SDE}
\end{align}
In particular, $\alpha(j)\subset V$ depends only on the endogenous variables, and $\beta(j)$ corresponds to the components $k$ of the Brownian motion (index $k+d$ in the expanded vector $X$), such that $G_{j,k} \neq 0$.
\end{definition}}
\rv{We can then define a causal graph $\mathcal{G}$ from the DSCM. We adapt \citep[Definition 2]{boeken2024dynamic}, with the only modification being that we permit self-edges $j \to j$.}
\rv{\begin{definition}[Causal graph from DSCM]
Given a DSCM \eqref{eq: explicit_SEM_SDE}, we define a causal graph $\mathcal{G}$ with vertices $V \cup W$, and edges $E = \{i \to j: j \in V, i \in \alpha(j) \cup \beta(j) \}$.
\end{definition}}
\rv{Intuitively, an edge $i \to j$ indicates that component $X_t^{(i)}$ directly influences the evolution of component $X_t^{(j)}$, and we only include edges that point towards variables in the endogenous set. Since the diffusion $G$ is state-independent for additive noise SDEs, this means that endogenous variables can only influence one another through the drift function $b$. However, the diffusion still determines incoming edges to endogenous variables. In particular, latent confounders are modeled by off-diagonal entries of the diffusion $G$, since multiple variables would be correlated due to shared driving noise (see \citep[Figure 1]{boeken2024dynamic}).} 

\rv{This causal framework is well-defined given access to the drift $b$ and diffusion $G$. However, due to the symmetry of Brownian motion, one can at best identify $H= GG^\top$ from observational data. Given the parameters $b$ and $H= GG^\top$ from an additive noise SDE, we show that one can recover all endogenous edges from the drift $b$ and also use non-zero entries in the observational diffusion $H= GG^\top$ to determine pairwise latent confounders.}

\begin{lemma}
Let $X_t$ evolve according to a $d$-dimensional additive noise SDE: $\mathrm{d}X_t = b(X_t) \: \mathrm{d}t+ G\mathrm{d}W_t$, and let $\mathcal{G}= (V \cup W, E)$ be its causal graph.
\begin{enumerate}[label=\alph*.]
\item  There exists a directed edge $i \to j$ in $\mathcal{G}$ if and only if $b_j(X_t)$ depends on $X_t^{(i)}$.
\item There exists a latent confounder $W^{(k)}$ over $(i,j)$ in $\mathcal{G}$, if $H_{i,j} = (GG^\top)_{i,j} \neq 0$. 
\end{enumerate}
\label{prop: identify_causal_graph_SDE_add_noise}
\end{lemma}

\begin{proof}
The first claim follows immediately from the definition of the dynamic  structural causal model \eqref{eq: DSCM_additive_SDE}. 
Then, to prove the second claim, we note that if $H_{i,j} = G_{i,\cdot} \cdot G^\top_{\cdot,j}  \neq 0$, then it follows that for some column $k \in [m]$ in the diffusion matrix $G$, we have that $G_{i,k}, G_{j,k} \neq 0$. From this, we conclude that $X^{(i)}$ and $X^{(j)}$ both share a dependence on the noise from the $k$th component of the Brownian motion, $W_t^{(k)}$.





\end{proof}

\rv{Lemma \ref{prop: identify_causal_graph_SDE_add_noise} establishes a bridge between causal graph identification and SDE parameter identification, where we can only hope to recover the observational diffusion. Indeed, by combining Theorem \ref{thm:stronger} and Lemma \ref{prop: identify_causal_graph_SDE_add_noise}, we see that we can almost always identify all directed edges within the endogenous set (due to recovery of the drift $A$) from the marginals of a linear additive noise SDE. Furthermore, the nonzero entries from the identified diffusion matrix, $H_{i,j}$, $i \neq j$, always identify pairwise latent confounding from shared noise. We emphasize that even when the drift is assumed to be linear, additive noise SDEs can encode arbitrary causal structures between endogenous variables (ex. via nonzero entries in the drift $A$). This makes our identification theory particularly relevant for GRN inference, since our setting can model cyclic behaviour, such as feedback loops and repressilator dynamics, which are fundamentally excluded in common gradient-flow models
\citep{weinreb2018fundamental, hashimoto2016learning, lavenant2021towards, terpin2024learning, guan2025gradient}. We however note that not all noise relations can be learned from the observational diffusion, as we discuss in Examples \ref{ex: 3_bidirected_vs_multiedge} and \ref{ex: cancellation_for_pairwise_latents} in Appendix \ref{sec: SDEs_causality}.
We include additional details and discussion about the causal interpretation of SDEs in Appendix \ref{sec: SDEs_causality}.}

\section{Our Parameter Estimation Method}
\label{sec:methods}
In this section, we introduce our method, Alternating Projection Parameter Estimation from $X_0$ (APPEX), which estimates a process' drift, diffusion, and causal graph from observed temporal marginals. Figure \ref{fig:APPEX_visualized} illustrates the general strategy of the APPEX algorithm. We focus on linear additive noise SDEs, since Theorem \ref{thm: identifiability} guarantees identifiability given that $X_0$ is not auto-rotationally invariant. However, our algorithm can be used to estimate any additive noise SDE in the form of Equation \ref{eq: additive_noise_SDE}. In this section, we give an intuitive explanation at the population level, using the process' true temporal marginals $X_{t_i} \sim p_i$. In Section \ref{sec:experiments}, we perform experiments using empirical marginals $\hat{p}_i$, with finite samples per marginal.

\begin{figure}
    \centering
\includegraphics[width=1\linewidth]{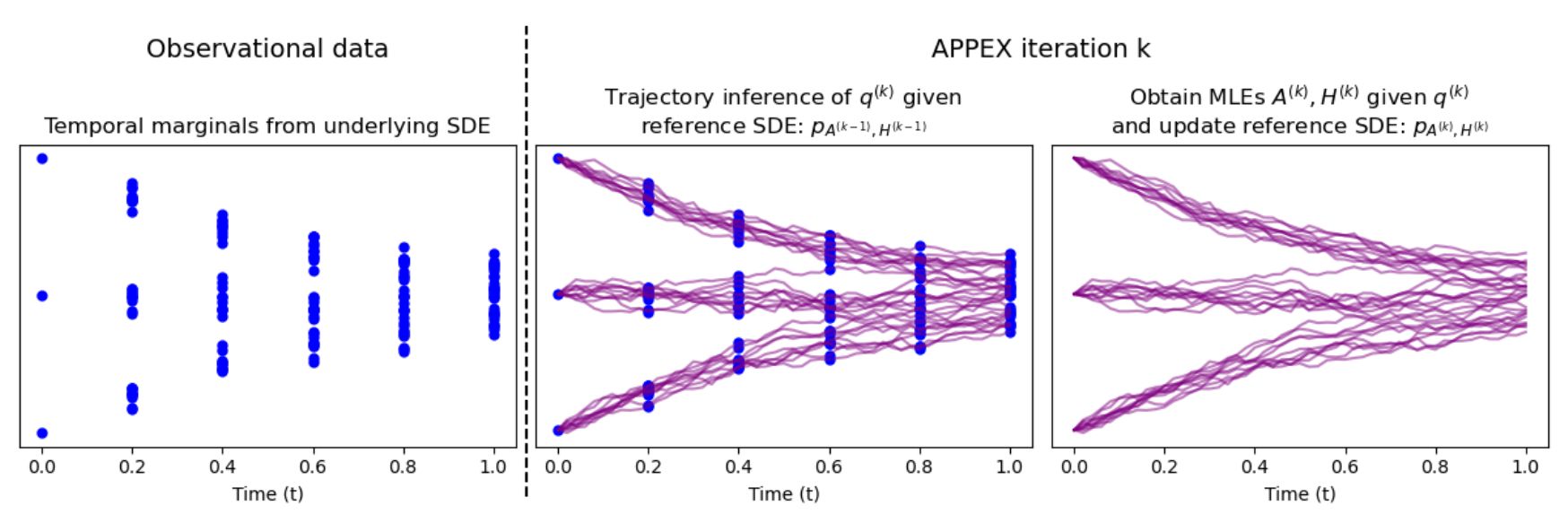}
    \caption{\rv{Visualization of our parameter estimation algorithm APPEX on a toy example}. Given observed temporal marginals from an underlying SDE (left), APPEX alternates between trajectory inference (middle) and MLE parameter estimation (right) in order to find the SDE parameters that best represent the temporal marginal observations.}
    \label{fig:APPEX_visualized}
\end{figure}

Let $p_0, ..., p_{N-1}$ denote a set of observed temporal marginals, and let $(A^{(0)}, H^{(0)})$ denote initial guesses for the drift and diffusion. The idea is to use an alternating optimization algorithm to obtain increasingly better estimates for the drift and diffusion matrices. Formally, we consider the spaces
\begin{align}
    \mathcal{D} &=\{q: q_0 = p_0, ..., q_{N-1} = p_{N-1} \}
    \label{eq: marginal_constraints} \text{ and }\\
    \mathcal{A} &= \left\{ p : \exists\, (A, H) \in \mathbb{R}^{d \times d} \times S^d_{\succeq 0} \text{ s.t } p(x, t+\Delta t \mid y, t) = \mathcal{N}\Big(e^{A\Delta t}y,\, \int_0^{\Delta t} e^{A s}\,H\,e^{A^\top  s}\, ds\Big) \right\}.
    \label{eq: projection_spaces_for_method}
\end{align}
$\mathcal{D}$ is the set of laws on paths, which coincide with the $N$ temporal marginals at their respective times. $\mathcal{A}$ is the set of laws given by linear additive noise SDEs. To find the underlying linear additive noise SDE in the intersection $\mathcal{A} \cap \mathcal{D}$, our method alternates between information projections, $\argmin_{q \in \mathcal{D}}D_{KL}(q \| p)$ and moment projections, $\argmin_{p \in \mathcal{A}}D_{KL}(q \| p)$. The information projection step corresponds to trajectory inference, whose goal is to find the law on paths $q \in \mathcal{D}$, which best aligns with the reference measure $p$. The moment projection then corresponds to maximum likelihood parameter estimation, since it amounts to determining the law on paths $p \in \mathcal{A}$ over linear additive noise SDEs, which best aligns with the law of the estimated trajectories $q$. Indeed, in the infinite-data limit, MLE is equivalent to minimizing relative entropy, i.e. KL divergence \citep{akaike1998information, white1982maximum}. At the population level, an iteration entails the updates:
\begin{align}
q^{(k)} &= \argmin\limits_{q \in \mathcal{D}}D_{KL}(q \| p_{A^{(k-1)}, H^{(k-1)}})\label{eq: entropy_min_to_reference_sde}\\
A^{(k)} &= \argmin\limits_{A \in \R^{d \times d}}D_{KL}(q^{(k)} \| p_{A, H^{(k-1)}})\label{eq: entropy_min_drift}\\
H^{(k)} &= \argmin\limits_{H \in S^d_{\succeq 0}}D_{KL}(q^{(k)} \| p_{A^{(k)}, H}).
\label{eq: entropy_min_diffusion}
\end{align} 

At iteration $k$, we start by using the SDE parameters estimated from the previous iteration to define a reference SDE $p_{A^{(k-1)}, H^{(k-1)}} \in \mathcal{A}$. Using this reference SDE and the set of observed temporal marginals, we perform trajectory inference \eqref{eq: entropy_min_to_reference_sde}, to estimate the path distribution $q^{(k)} \in \mathcal{D}$ that minimizes relative entropy to the reference SDE, while satisfying the marginal constraints \eqref{eq: marginal_constraints} (see the middle plot of Figure \ref{fig:APPEX_visualized}). We then perform a two-step maximum likelihood estimation, in order to estimate the drift and diffusion from the estimated trajectories. This in turn updates the reference SDE $p_{A^{(k)}, H^{(k)}} \in \mathcal{A}$ for the next iteration. To perform APPEX on a general additive noise SDE \eqref{eq: additive_noise_SDE}, the spaces $\mathcal{A}$ and the drift maximum likelihood parameter estimation step \eqref{eq: entropy_min_drift} would be modified accordingly. We show in Corollary \ref{cor: general_MLE_additive_Noise} of the appendix that there is a closed form diffusion MLE for general additive noise SDEs, once the drift has been estimated.


Previous works implement similar iterative schemes \citep{zhang2024joint, shen2024multi, vargas2021solving}, which alternate between trajectory inference and MLE parameter estimation. In contrast to our method, each of these works assumes that the diffusion of the process is known, and hence invariant across iterations. One reason for this assumption is to ensure finite KL divergence between the estimated trajectories and estimated SDEs for all iterations \citep{vargas2021solving}. Given continuously observed marginals of a $d$-dimensional process, the KL divergence between two laws on paths $q, p$ is taken over the path space $\Omega = C([0,T], \R^d)$, such that
$$
D_{KL}(q\| p)= \int_{\Omega} \log\left(\frac{\mathrm{d}q}{\mathrm{d}p}(\omega) \right)\mathrm{d}q(\omega).
$$
Thus, $D_{KL}(q\| p)$ will only be finite between two laws on paths if we can define the Radon-Nikodym derivative $\frac{\mathrm{d}q}{\mathrm{d}p}$ over path space $\Omega$. By Girsanov's theorem, this is only ensured when both processes share the same diffusion \citep{shen2024multi, vargas2021solving}. However, if we only consider measurements from a finite number of observed marginals, the KL divergence of the discretized processes $q^N, p^N$ over the path space projected on the $N$ measurement times, $\Sigma_\mathcal{D}= C(\{t_i\}_{i=0}^{N-1}, \R^d)$, can be decomposed as follows, see \citet{benamou2019entropy}:
\begin{align}
\label{eq:discretized_KL}
D_{KL}(q^{N}\| p^{N})=\sum_{i=0}^{N-2}D_{KL}(q_{i, i+1}\| p_{i, i+1}) - \sum_{i=1}^{N-2}D_{KL}(q_i\| p_i),
\end{align}
where $p_{i, i+1}, q_{i, i+1}$ are the joint probability measures of $q, p$, restricted to times
$t_i$ and $t_{i+1}$, and similarly $q_i, p_i$ are the marginals of $q, p$ at time $t_i$. We note that the above equality holds since both $p^N$ and $q^N$ are Markovian. Then, since $D_{KL}(q_i \| p_i)\leq D_{KL}(q_{i, i+1}\| p_{i, i+1})$ holds by the data-processing inequality, it follows that $D_{KL}(q^{N}\| p^{N})$ will be finite as long as $D_{KL}(q_{i, i+1}\| p_{i, i+1}) = \int_{\R^d \times \R^d} \log\left(\frac{\mathrm{d}q_{i, i+1}}{\mathrm{d}p_{i, i+1}}(x) \right) q_{i, i+1}(x) \, \mathrm{d}x < \infty$ for each $i=0, \ldots, N-2$.  In particular, if $p$ and $q$ are the laws of two different drift-diffusion SDEs with distinct non-degenerate diffusions, then the KL divergence  $D_{KL}(q^N\| p^N)$ over the discretized path space $\Sigma_\mathcal{D}$ will be finite. Since we consider a setting with a finite number of temporal marginal observations in this work, we may consider laws on paths discretized over these observations. This allows us to both consider diffusions which are not known in advance, and to improve our diffusion estimates with respect to KL divergence after each iteration.



\subsection{Trajectory Inference Via Anisotropic Entropy-Regularized Optimal Transport}
To begin each iteration of APPEX, we infer trajectories given the set of temporal marginals $\{p_i\}_{i=0}^{N-1}$ and a reference SDE $p_{A^{(k-1)}, H^{(k-1)}} \in \mathcal{A}$, by solving an associated Schr\"odinger Bridge problem. Under certain conditions, the SB problem \eqref{eq: schrodinger_bridge} can be solved using the variational entropic optimal transport (EOT) formulation
\begin{align}
\pi^* = \argmin_{\pi \in \Pi(\mu,\nu)} \int c(x,y) \, \mathrm{d}\pi(x,y) + \epsilon ^2D_{KL}(\pi \| \mu \otimes \nu).
\label{eq: entropic_OT_problem}
\end{align}
Intuitively, $\pi^*$ transports the probability measure $\mu$ to $\nu$ while minimizing a total cost, based on the function $c(x,y)$ and the entropic regularization $\epsilon^2 > 0$. Objectives  \eqref{eq: schrodinger_bridge} and \eqref{eq: entropic_OT_problem} are equivalent \citep{peyre2019computational, zhang2024joint}, under the reference measure $d\mathcal{K}(x,y)=K(x,y)d\mu(x)d\nu(y)$ with density $K(x,y) \propto e^{\frac{-c(x,y)}{\epsilon^2}}$. The convention in the literature is to set $c(x,y)= \frac{\|y-x\|^2}{2}$, to define an isotropic Gaussian law $\mathcal{K}$ \citep{janati2020entropic, lavenant2021towards,  chizat2022trajectory, zhang2024joint}. The EOT problem \eqref{eq: entropic_OT_problem} can then be solved via Sinkhorn's algorithm. 


However, the standard EOT problem \eqref{eq: entropic_OT_problem} only considers a scalar regularization parameter $\epsilon^2$, which means that it can only directly model reference SDEs with isotropic diffusion (see Remark \ref{rmk: EOT_for_traj_inf} in the appendix).
To generalize trajectory inference for SDEs with anisotropic diffusion, we formalize ``anisotropic entropy-regularized optimal transport" (AEOT), by considering transition kernels $K$ with custom mean and covariance, e.g. parameterizing $K$ with $A$ and $H=GG^\top$ rather than the scalar $\epsilon^2>0$. Indeed, we can consider more general Gaussian transition kernels, $K_{\theta}(x,y) = \exp({\frac{(y-\mu(x, \theta))^\top \Sigma(\theta)^{-1}(y-\mu(x, \theta))}{-2}})$, in order to model a transition $y|x \sim \mathcal{N}(\mu(x, \theta), \Sigma(\theta))$, whose covariance $\Sigma(\theta)$ is possibly anisotropic. Under this formulation, the cost is given by the inner product $c_\theta(x,y) =\frac{(y-\mu(x, \theta))^\top \Sigma(\theta)^{-1}(y-\mu(x, \theta))}{2}$. Trajectory inference would amount to solving the AEOT problem:
\begin{align}
\pi^* = \argmin_{\pi \in \Pi(\mu,\nu)} \int \frac{(y-\mu(x, \theta))^\top \Sigma(\theta)^{-1}(y-\mu(x, \theta))}{2} \, \mathrm{d}\pi(x,y) + D_{KL}(\pi \| \mu \otimes \nu),
\label{eq: generalized_entropic_OT_problem}
\end{align}
where the entropic regularization is captured in the matrix $\Sigma(\theta)$ and function $\mu(x, \theta)$. The AEOT problem over a pair of marginals \eqref{eq: generalized_entropic_OT_problem} can be solved via Sinkhorn's algorithm with inputs $\mu, \nu, K_\theta$. Proposition \ref{prop: md_Sinkhorn_optimal} in the Appendix then shows that the solution to the multi-marginal trajectory inference step \eqref{eq: entropy_min_to_reference_sde} is the joint distribution given by the Markov concatenation of couplings (see Definition \ref{def:Markov_concatenation_couplings} in the Appendix)
\begin{align*}
    \pi^* = \pi_{N-2,N-1} \circ \ldots \circ \pi_{0,1},
\end{align*}
where $\pi_{i,i+1}$ is the AEOT solution \eqref{eq: generalized_entropic_OT_problem} with marginals $\mu = p_i, \nu = p_{i+1}$ and a transition kernel determined by the reference SDE.


\begin{remark}[Application to empirical marginals]
    Proposition \ref{prop: md_Sinkhorn_optimal} is formalized for exact temporal marginals $p_i$. However, in practice, we observe empirical measures $\hat{p}_i$, which converge in distribution to $p_i$, in the limit of infinite data. It is for these empirical measures $\hat{p}_i$ that we actually compute the estimated couplings $\hat{\pi}_{i,i+1}$ in practice. This corresponds to numerically estimating the AEOT solution
\begin{align}
\hat{\pi}_{i,i+1} = \argmin_{\pi \in \Pi(\hat{p}_i,\hat{p}_{i+1})} \int \frac{(y-\mu(x, \theta))^\top \Sigma(\theta)^{-1}(y-\mu(x, \theta))}{2} \, \mathrm{d}\pi(x,y) + D_{KL}(\pi \| \hat{p}_i \otimes \hat{p}_{i+1}).
\label{eq: generalized_entropic_OT_problem_empirical}
\end{align}
By the main theorem in \citet{ghosal2022stability}, it holds that $\hat{\pi}_{i,i+1}$ converges in distribution almost surely to the minimizer of (\ref{eq: generalized_entropic_OT_problem}) with $\mu=p_i$ and $\nu=p_{i+1}$ as the number of samples goes to infinity for each time $i$ and $i+1$. Combining this with the previous proposition, we see that estimating (\ref{eq: generalized_entropic_OT_problem_empirical}) for each pair of times $i$ and $i+1$ is asymptotically equivalent to the KL minimization step in (\ref{eq: entropy_min_to_reference_sde}).

We also note that \eqref{eq: generalized_entropic_OT_problem_empirical} is practically solved using Sinkhorn's algorithm with empirical marginals $\hat{p_i}, \hat{p}_{i+1}$ and the discretized transition kernel $K_\theta$ of the reference SDE. For example, if we have $M$ samples per empirical marginal, then $\hat{p}_i \sim \text{Unif}(x_{t_i}^{(j)}: j=1, \ldots M)$ and $\hat{p}_{i+1}\sim \text{Unif}(x_{t_{i+1}}^{(j)}: j=1, \ldots M)$ would both be discrete uniform distributions over their samples, and $K_\theta \in \R^{M \times M}$ would be a square matrix, such that entry $K_{\theta_{j,k}}$ is obtained by applying the given transition kernel on the data points $x_{t_i}^{(j)}$ and $x_{t_{i+1}}^{(k)}$.
\end{remark}


\subsection{Parameter Estimation Via MLE}

To optimize objectives \eqref{eq: entropy_min_drift} and \eqref{eq: entropy_min_diffusion} for each iteration of APPEX, we require maximum likelihood estimators for the SDE parameters, given multiple observed trajectories from $[0,T]$. In the context of iteration $k$ of APPEX, these are the trajectories sampled from the law on paths $q^{(k)}$ obtained from the trajectory inference step. We derive closed-form maximum likelihood estimators for the linear additive noise SDE from  Equation \ref{eq: linear_additive_noise_SDE}, given multiple trajectories in Proposition \ref{prop: MLE_SDE_params}. Closed form drift MLE solutions are generally unavailable for additive noise SDEs, however, we derive a closed form diffusion estimate (which depends on the drift) for general additive noise SDEs in Corollary \ref{cor: general_MLE_additive_Noise} of the appendix.

\begin{proposition}[MLE estimators for drift and diffusion from trajectories] 
Given $M$ trajectories over $N$ equally spaced times: $\{ X_{i\Delta t}^{(j)} : i \in 0, ..., N-1, \text{ } j \in 0, ..., M-1\}$ sampled from the linear additive noise SDE \eqref{eq: linear_additive_noise_SDE}, the maximum likelihood solution for linear drift is approximated by 
\begin{align}
\hat{A}&= \frac{1}{\Delta t}\left(\sum_{i=0}^{N-2} \sum_{j=0}^{M-1}(\Delta X_{i}^{(j)})  X_{i}^{{(j)}^\top}\right)\left(\sum_{i=0}^{N-2} \sum_{j=0}^{M-1} X_{i}^{(j)}X_{i}^{{(j)}^\top}\right)^{-1}
\\
&\overset{\substack{M \to \infty \\[2pt]}}{\longrightarrow}
 \frac{1}{\Delta t}\left(\sum_{i=0}^{N-2}   \E_{p_{i,i+1}}[(\Delta X_{i})X_{i}^\top]\right)\left(\sum_{i=0}^{N-2}   \E_{p_i}[X_{i}X_{i}^\top]\right)^{-1},
\label{eq: mle_drift}
\end{align}
where $p_{i,i+1}$ is the joint measure over $X_{t_i}$ and $X_{t_{i+1}}$. Similarly, the maximum likelihood solution for diffusion is approximated by
\begin{align}
\hat{H}&= \frac{1}{MT}\sum_{i=0}^{N-2}\sum_{j=0}^{M-1} \left( (\Delta X_i^{(j)} - AX_{{i}}^{(j)}\Delta t)(\Delta X_i^{(j)} - AX_{{i}}^{(j)}\Delta t)^\top \right)\\
&\overset{\substack{M \to \infty \\[2pt]}}{\longrightarrow}
\frac{1}{T}\sum_{i=0}^{N-2}   \E_{p_{i,i+1}}[\left(\Delta X_{i}  -AX_i\Delta t\right) \left(\Delta X_{i}  -AX_i\Delta t\right) ^\top ].
\label{eq: mle_diffusion}
\end{align}
\label{prop: MLE_SDE_params}
\end{proposition}
\begin{proof}
See Section \ref{sec:mle} of the Appendix. We note that estimators $\hat{A}$ and $\hat{H}$ were derived using the discretized transition kernel, $X_{i+1}| X_i \sim \mathcal{N}(X_i + AX_i\Delta t, H\Delta t)$. 
\end{proof}

We note that the MLE drift estimator for $A$ does not depend on the diffusion $H$, but the MLE estimator for $H$ depends on $A$. We therefore estimate drift first in each iteration of APPEX.

\subsection{The APPEX Algorithm}
We summarize APPEX's implementation in Algorithm \ref{alg: iterative_SDE_param_est}. In each iteration, trajectory inference \eqref{eq: entropy_min_to_reference_sde} is performed by solving the AEOT problem \eqref{eq: generalized_entropic_OT_problem}, and the drift and diffusion estimates are updated via MLE. By default, we use Sinkhorn's algorithm to solve AEOT and we use closed form MLE solutions when available.
\begin{algorithm}
\caption{Parameter estimation for an additive noise SDE from temporal marginals with APPEX}
\begin{algorithmic}[1]
\State \textbf{Input:} Observed marginals $\hat{p}_i$, $i = 0, \ldots, N-1$, number of iterations $K$, $\Delta t$
\State \textbf{Result:} Estimated drift function $\hat{b}$ and additive noise $\hat{H}$
\State \textbf{Initialize:} $\hat{b} \gets 0$, $\hat{H} \gets \sigma^2 \Id$, $k \gets 0$
\While{$k < K$}
    \For{$i = 0, \ldots, N-2$}
        \State $\hat{\pi}_{i, i+1} \gets \text{Anisotropic-Entropy-Regularized-Optimal-Transport}(\hat{b}, \hat{H}, \hat{p}_{i}, \hat{p}_{i+1}, \Delta t)$
    \EndFor
    \State $\text{Sample-Trajectories} \gets \hat{\pi}_{N-2, N-1} \circ \ldots \circ \hat{\pi}_{0, 1}(\hat{p}_0)$
    \State $\hat{b} \gets \text{MLEfit}(\text{Sample-Trajectories})$
    \State $\hat{H} \gets \text{MLEfit}(\text{Sample-Trajectories}, \hat{b})$
    \State $k \gets k + 1$
\EndWhile
\State $\mathcal{G} \gets$ \text{Estimate-Causal-Graph($\hat{b}, \hat{H}, \epsilon)$}
\end{algorithmic}
\label{alg: iterative_SDE_param_est}
\end{algorithm}

\begin{remark}[Application to causal discovery]
\rv{By Lemma \ref{prop: identify_causal_graph_SDE_add_noise}, we can estimate the causal graph from the SDE parameters learned by APPEX. For nonlinear drift $b(X_t)$, we would include the edge $i \to j$ if the $j$th component of the drift, $b_j$, is a function of $X^{(i)}$. This can for example be assessed by analyzing the Jacobian matrix of $b$. For linear drift, this simplifies to analyzing each entry $A_{j,i}$ and choosing a threshold to determine the presence of directed edges. For edges related to latent confounders (from shared diffusion), one can choose a threshold for $|H_{i,j}|$ to determine whether there is an unobserved confounder causing $X^{(i)}$ and $X^{(j)}$. While thresholding is one of the simplest methods to extract a causal graph from learned SDE parameters, we note that many other methods can be used for this task, and thus leave the implementation as general as possible.}
\end{remark}
Because the trajectory inference and parameter estimation subprocedures are optimal with respect to minimizing KL divergence, we can show that APPEX's estimates progressively fit the data better with each iteration. Furthermore, this data-fitting objective is uniquely minimized by the true parameters $(A,H)$, if identifiability conditions hold. 
\begin{lemma}
\label{lemma: APPEX_decrease_KL}
Suppose that $A,H$ are the true parameters of a linear additive noise SDE with temporal marginals $p_0, ..., p_{N-1}$. Let $q^{(k)} \in \mathcal{D}$ be the law of the estimated trajectories at iteration $k$ of APPEX, and $p_{A^{k}, H^{k}} \in \mathcal{A}$ be the law of the estimated linear additive noise SDE. Then, the relative entropy between the estimated law of the trajectories and the law of the estimated linear additive noise SDE is decreasing with each iteration \citep{shen2024learning}:
$$D_{KL}(q^{(k+1)} \| p_{A^{k+1}, H^{k+1} }  ) \le  D_{KL}(q^{(k)} \| p_{A^{k}, H^{k} })  \ \ \forall k \ge 0. $$ 
Furthermore, if $p_0$ is not weakly auto-rotationally invariant, then as we let the number of observed marginals $N \to \infty$ for any sampling scheme that would reduce the maximum time gap between marginals to $0$, the relative entropy between the estimated law of the trajectories and the law of the estimated SDE is uniquely minimized by the underlying SDE $p_{A,H}$:
$$\inf\limits_{q \in \mathcal{D}, p \in \mathcal{A}}D_{KL}(q\|p) = 0 \iff q=p = p_{A, H}.$$
\end{lemma}

\begin{proof} Proposition \ref{prop: md_Sinkhorn_optimal} shows that $q^{(k+1)}$ minimizes relative entropy to $p_{A^{k}, H^{k}}$. Similarly, any MLE solution $p_{A^{(k+1)}, H^{(k+1)} }$ minimizes relative entropy to $q^{(k+1)}$. Proposition \ref{prop: MLE_SDE_params} approximates the MLE parameters $A^{k+1}$ and $H^{(k+1)}$. By construction of the iterative scheme,
\[
\begin{aligned}
D_{KL}\!\left(q^{(k+1)} \,\middle\|\, p_{A^{(k+1)}, H^{(k+1)}}\right)
&\le
D_{KL}\!\left(q^{(k+1)} \,\middle\|\, p_{A^{(k+1)}, H^{(k)}}\right) \\
&\le
D_{KL}\!\left(q^{(k+1)} \,\middle\|\, p_{A^{(k)}, H^{(k)}}\right)
\le
D_{KL}\!\left(q^{(k)} \,\middle\|\, p_{A^{(k)}, H^{(k)}}\right),
\end{aligned}
\]
which proves the first claim of the lemma. If $X_0$ is not weakly auto-rotationally invariant, then $p_{A,H}$ is the law of the unique linear additive noise SDE, which obeys the marginal constraints $p_0, \ldots p_{N-1}$ as $N \to \infty$ (Theorem \ref{thm: identifiability}). Equivalently, we have  $\mathcal{A} \cap \mathcal{D} = \{ p_{A, H} \}$,
$$\inf\limits_{q \in \mathcal{D}, p \in \mathcal{A}}D_{KL}(q\|p) = 0 \iff q=p = p_{A, H}.$$ 
\end{proof}


\section{Experiments}
\label{sec:experiments}
In a wide range of experiments on simulated data from linear additive noise SDEs, we find that APPEX accurately identifies the underlying SDE from its temporal marginals and that it outperforms the popular Waddington-OT method \citep{schiebinger2019optimal}. In Section \ref{sec: revisit_SDEs}, we assess the validity of Theorem \ref{thm: identifiability} by showing that APPEX can learn the drift and diffusion parameters of classical non-identifiable SDEs, if $X_0$ is not auto-rotationally invariant. We then demonstrate APPEX's effectiveness on randomly generated higher-dimensional SDEs in Section \ref{sub: higher_dim_exps}. Finally, in Section \ref{sec: causal_discovery}, we show that APPEX can use its parameter estimates to recover the underlying causal graph, thus supporting Lemma \ref{prop: identify_causal_graph_SDE_add_noise}. An empirical experiment for consistency is also provided in Appendix \ref{sec:consistency}, which shows that as the number of samples per time marginal increases, APPEX's drift and diffusion estimates converge to the true parameters.

For each experiment, we simulate data for linear additive noise SDEs using Euler-Maruyama discretization with $dt_{EM}=0.01$, such that we generate $M=500$ trajectories, each observed across $100$ time steps. Each $d$-dimensional trajectory is initialized with $p_0\sim \text{Unif}\{x_i\}_{i=1}^{d}$ where $\{x_i\}_{i=1}^{d}$ are random linearly independent vectors with each entry having a magnitude uniformly sampled between $2$ and $10$. Since these are discrete random variables whose support spans the full space, it follows from Proposition \ref{prop: discrete_auto_rotate} and Theorem \ref{thm: identifiability} that the drift and diffusion are theoretically identifiable from marginals. To model the ``marginals-only'' setting, we subsample at the rate $dt=0.05$ to produce $N=20$ marginals with $M=500$ observations per time. 

To perform parameter estimation, we use $30$ iterations of our APPEX algorithm, such that the first reference SDE is an isotropic Brownian motion $\mathrm{d}X_t = \sigma \mathrm{d}W_t$, even when the true process has anisotropic noise. To model the realistic setting where diffusion is not precisely known, we assume that the initial guess of the diffusion's trace is within an order of magnitude of the ground truth trace $\text{tr}(H)$. Specifically, we initialize $A^{(0)} = 0, H^{(0)}=\sigma^2 \Id \ s.t. \ \sigma^2 \sim (\text{tr}(H)/d)10^{\text{Unif}(-1,1)}$. Because the first reference SDE is $\sigma \mathrm{d}W_t$, we note that the first iteration of APPEX is equivalent to performing MLE on trajectories that were inferred using the Waddington-OT (WOT) method from \citet{schiebinger2019optimal}, based on standard entropy regularized OT \eqref{eq: entropic_OT_problem} with $\epsilon^2 = \sigma^2 \Delta t$. APPEX and WOT are distinguished by the fact that APPEX benefits from further iterations, and allows the reference SDE to have anisotropic diffusion and non-zero drift. For each iteration of APPEX, we use Sinkhorn's algorithm to solve the anisotropic entropy-regularized optimal transport problem for trajectory inference, and we use the closed form MLE solutions from Proposition \ref{prop: MLE_SDE_params}. Due to time complexity and numerical stability, we use linearized discretizations for the Gaussian transition kernels. The Python code for reproducing the experimental results and figures is available at \url{https://github.com/guanton/APPEX}. Analogous code in R is available at \url{https://github.com/HydroML/X0isAllYouNeed}.

\subsection{Revisiting Non-Identifiable SDEs}
\label{sec: revisit_SDEs}
We consider three non-identifiable pairs of SDEs from the literature (summarized in Appendix \ref{sec:Noniden_examples}), but revise $X_0$ to be non-auto-rotationally invariant, such that by Theorem \ref{thm: identifiability}, the SDE parameters are now theoretically identifiable from observed marginals. $10$ replicates of each experiment were performed, such that the replicates featured data from different valid $p_0$ initializations and different initial diffusivities $\sigma^2$. We track the mean absolute error (MAE), plotted in Figure \ref{fig:mae_3_exps}, between the true drift/diffusion parameters and their estimates.
\begin{figure}[h!]
    \centering
    \begin{subfigure}{0.49\textwidth}
        \centering
        \includegraphics[width=\textwidth]{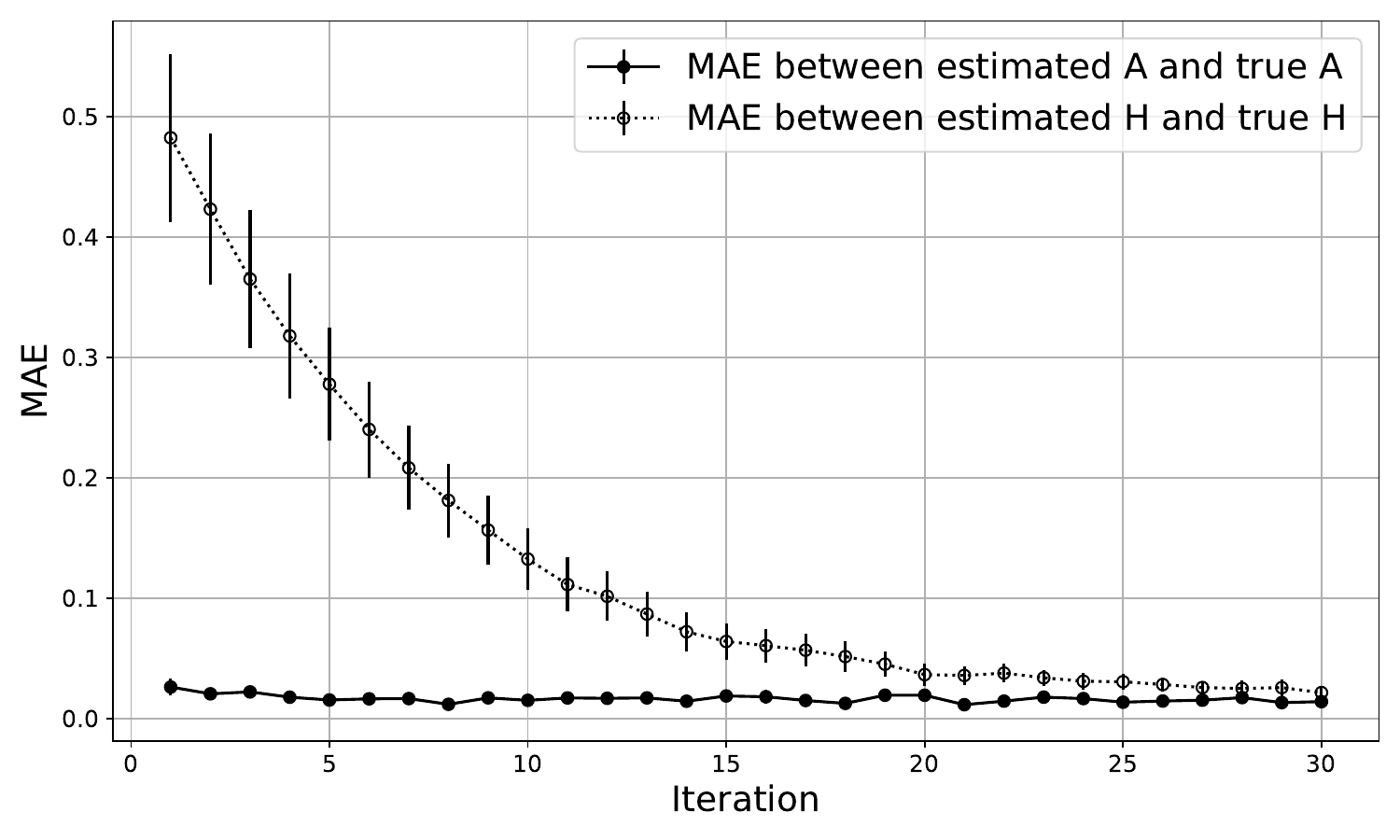}
        \caption{$\mathrm{d}X_t = -X_t \: \mathrm{d}t+ \mathrm{d}W_t$ from Example~\ref{eq: non_identifiability_same_stationary} }
    \end{subfigure}
    \begin{subfigure}{0.49\textwidth}
        \centering
        \includegraphics[width=\textwidth]{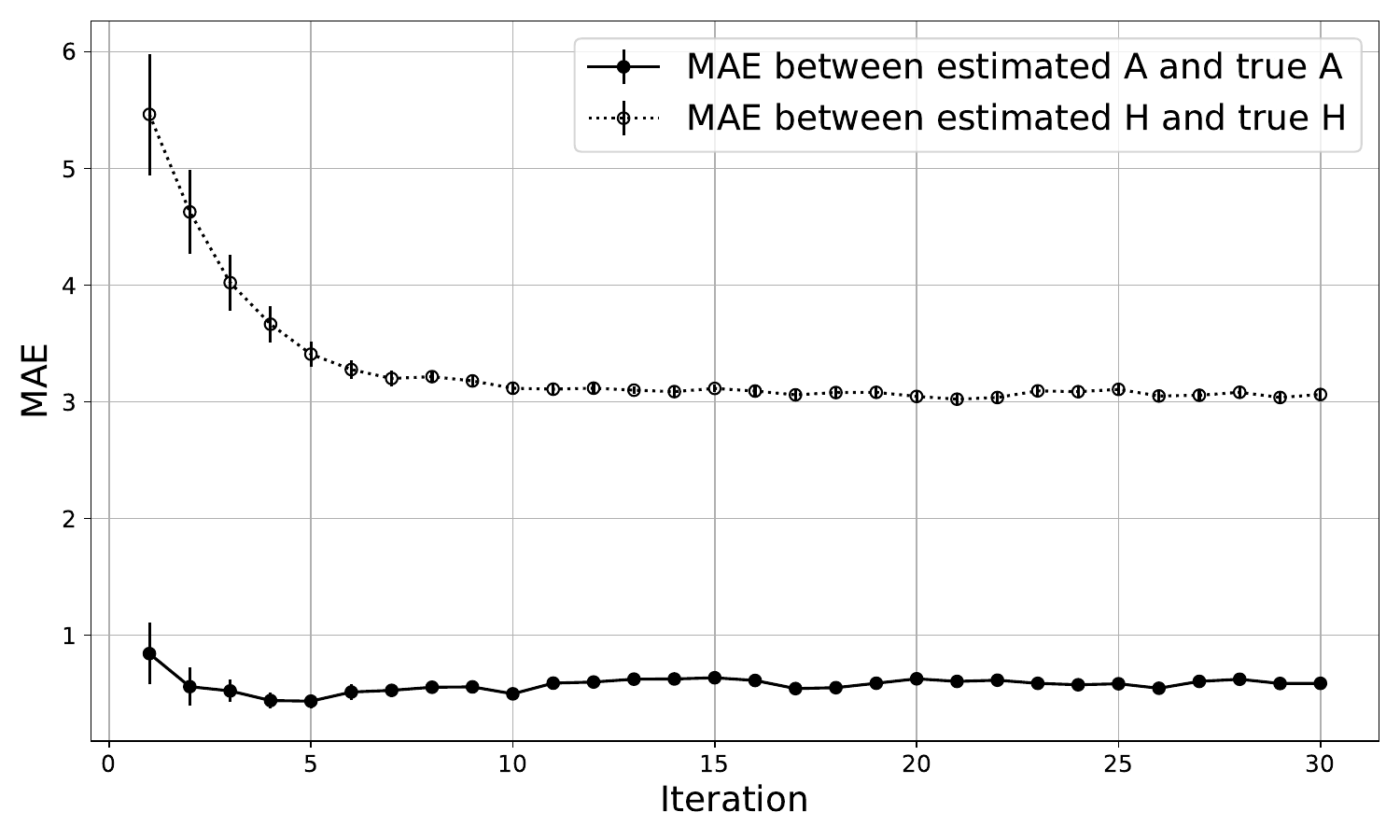}
        \caption{$\mathrm{d}X_t = -10X_t \: \mathrm{d}t+ \sqrt{10}\mathrm{d}W_t$ from Ex. ~\ref{eq: non_identifiability_same_stationary}}
    \end{subfigure}
    
    \vspace{0.3cm} 

    \begin{subfigure}{0.49\textwidth}
        \centering
        \includegraphics[width=\textwidth]{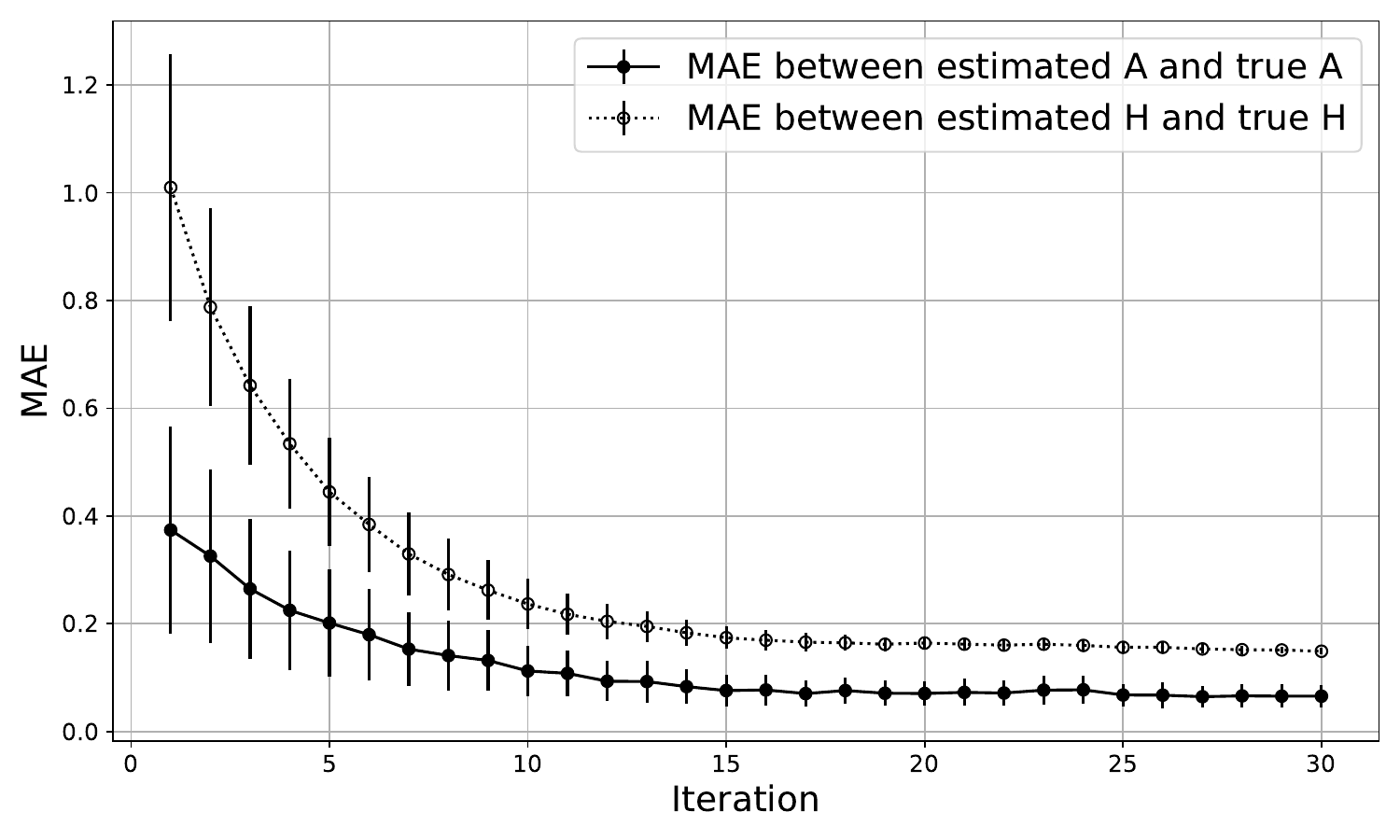}
        \caption{$\mathrm{d}X_t = \begin{bmatrix} 1 & 0 \\ 0 & 1\end{bmatrix}\mathrm{d}W_t$ from Ex. ~\ref{eq: non_identifiability_rotation}}
    \end{subfigure}
    \begin{subfigure}{0.49\textwidth}
        \centering
        \includegraphics[width=\textwidth]{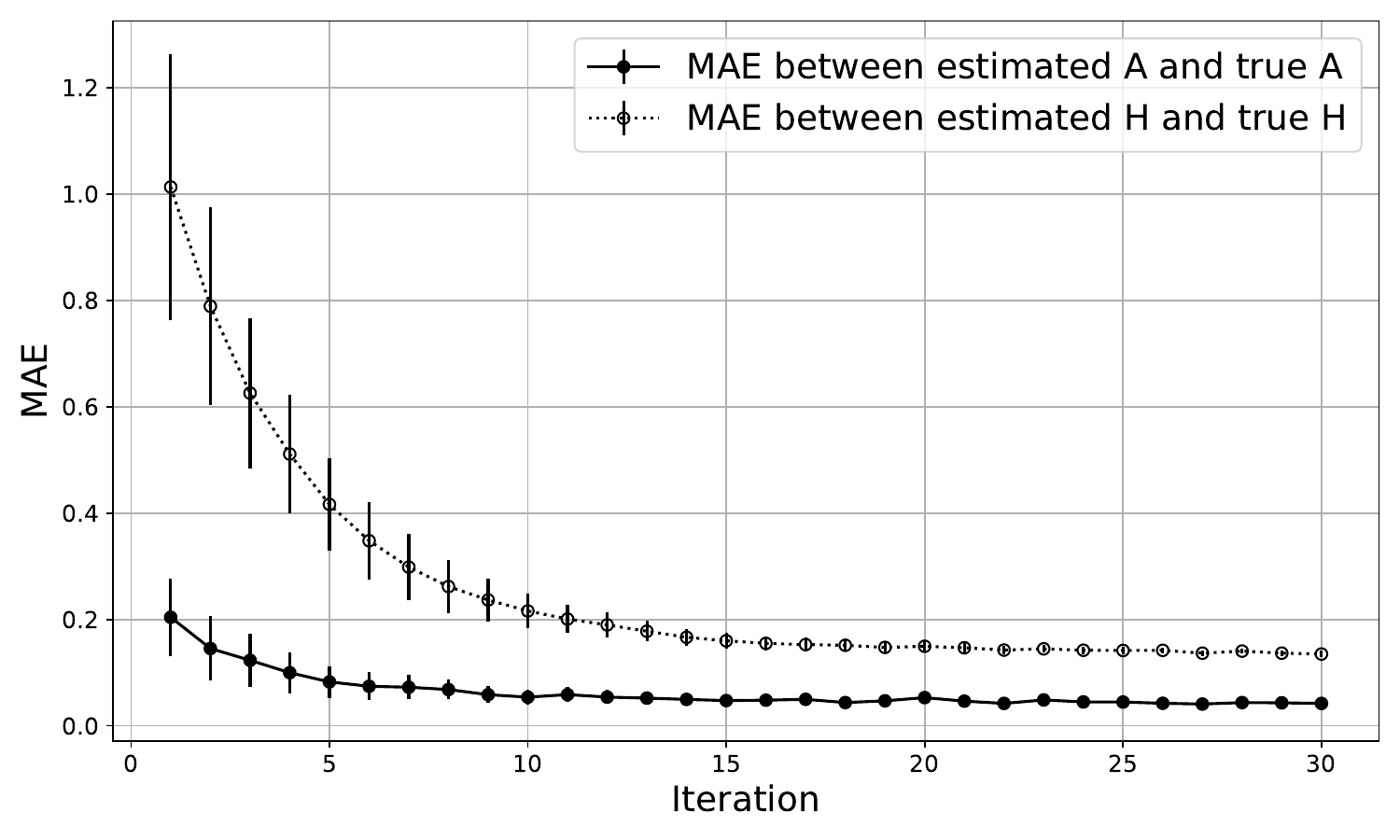}
        \caption{ $\mathrm{d}X_t = \begin{bmatrix} 0 & 1 \\ -1 & 0 \end{bmatrix} X_t \mathrm{d}t+  \begin{bmatrix} 1 & 0 \\ 0 & 1\end{bmatrix} \mathrm{d}W_t$ from Ex. ~\ref{eq: non_identifiability_rotation}}
    \end{subfigure}
    
    \vspace{0.3cm} 

    \begin{subfigure}{0.49\textwidth}
        \centering
        \includegraphics[width=\textwidth]{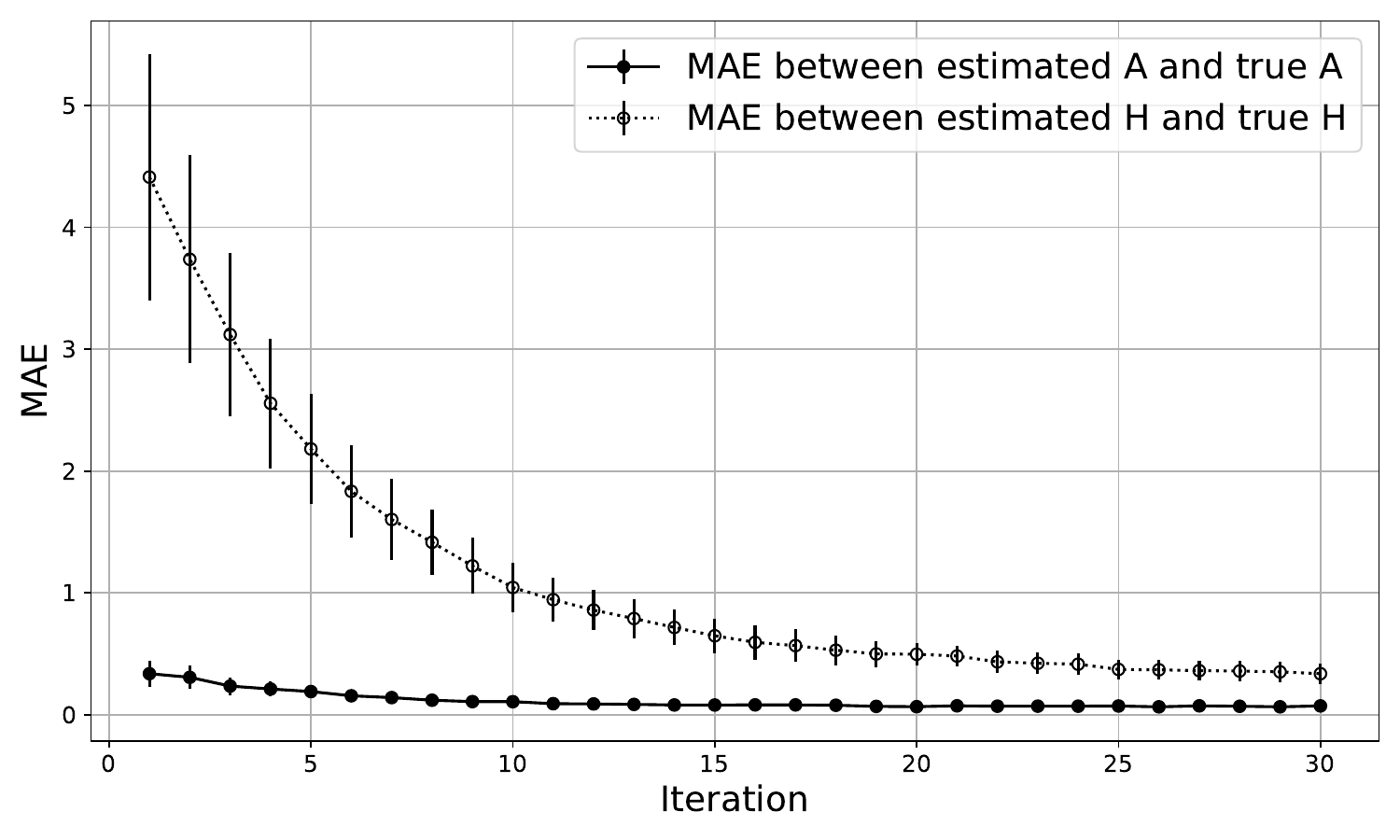}
        \caption{$\mathrm{d}X_t = \begin{bmatrix} 1 & 2 \\ 1 & 0 \end{bmatrix} X_t \mathrm{d}t+ \begin{bmatrix} 1 & 2 \\ -1 & -2 \end{bmatrix} \mathrm{d}W_t$ from Ex.\ref{eq: non_identifiability_wang}}
    \end{subfigure}
    \begin{subfigure}{0.49\textwidth}
        \centering
        \includegraphics[width=\textwidth]{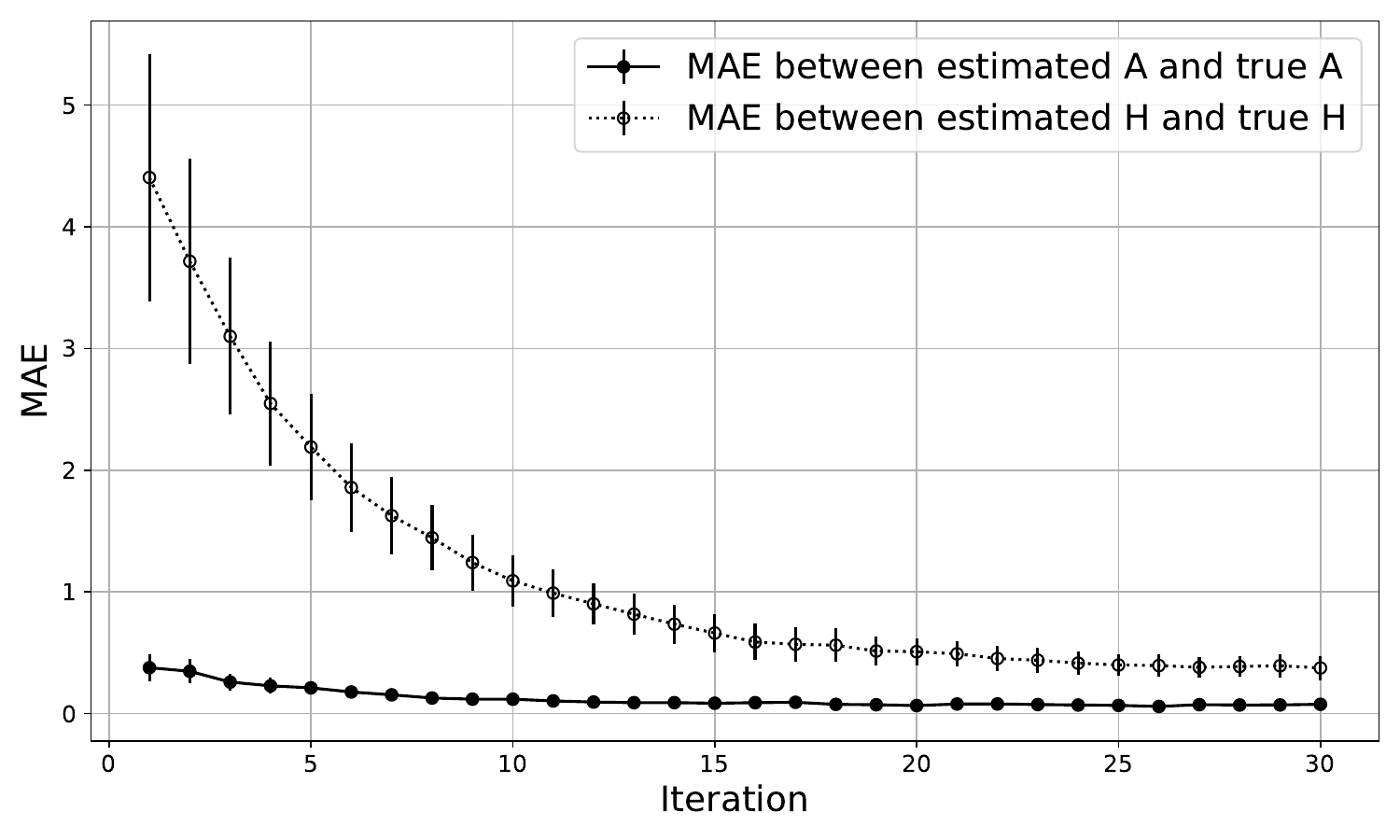}
        \caption{$\mathrm{d}X_t = \begin{bmatrix} \frac{1}{3} & \frac{4}{3} \\ \frac{2}{3} & -\frac{1}{3} \end{bmatrix} X_t \mathrm{d}t+ \begin{bmatrix} 1 & 2 \\ -1 & -2 \end{bmatrix} \mathrm{d}W_t$ from Ex.\ref{eq: non_identifiability_wang}}
    \end{subfigure}
    
    \caption{The mean absolute error for estimates of $A$ and $H$ using APPEX is shown per iteration for all three pairs of SDEs from Section \ref{sec:Noniden_examples}.}
\label{fig:mae_3_exps}
\end{figure}

The results are consistent with Theorem \ref{thm: identifiability} and Lemma \ref{lemma: APPEX_decrease_KL}. In each example, APPEX is able to estimate the drift and diffusion by iteratively improving upon both estimates, as evidenced by decreasing MAEs as the number of iterations increases. In particular, we observe the worst performance from the first iteration, which is equivalent to performing parameter estimation following trajectory inference by WOT. The most difficult SDE to infer was $\mathrm{d}X_t = -10X_t \: \mathrm{d}t+ \sqrt{10}\mathrm{d}W_t$, due to the high noise scale. We also note that APPEX succeeded in estimating the degenerate diffusion from the SDEs in Example \ref{eq: non_identifiability_wang}, despite the  fact that degenerate diffusion can result in infinite KL divergence with respect to a misspecified reference SDE.

We emphasize that if we observe marginals such that the SDE is non-identifiable, for example, inferring the Ornstein-Uhlenbeck SDEs $\mathrm{d}X_t = -X_t \: \mathrm{d}t+ \mathrm{d}W_t$ from snapshots sampled from its equilibrium distribution $\mathcal{N}(0, \frac12)$, then no method can recover the true parameters. Indeed, in this example, an infinite number of linear additive noise SDEs share the same equilibrium distribution (as long as the drift to diffusion ratio is $1$), and are hence consistent with the observed data. When using APPEX to estimate the SDE in this non-identifiable setting, we indeed saw that it recovered different drift and diffusion estimates, with ratio close to $1$ (see Figure~\ref{fig:OU_nonidentifiable}).

\subsection{Higher Dimensional Randomly Generated Linear Additive Noise SDEs}
\label{sub: higher_dim_exps}
We now consider parameter estimation on a broad range of higher dimensional random linear additive noise SDEs. For each dimension $d=3, \ldots, 10$, we generate $10$ random SDEs. To create each drift matrix $A$, we randomly sample each of its $d^2$ entries from $\text{Unif}(-5,5)$, and verify that its principal eigenvalue is less than $1$. This allows us to consider non-stationary processes, while obeying reasonable growth conditions. To create each diffusion $G$, we randomly initialize each of its $d^2$ entries via $\text{Unif}(-1,1)$ and then set $H=GG^\top$. For numerical stability, we perform Sinkhorn on the logarithmic scale for this experiment. We evaluate each estimate’s performance using MAE and correlation with the true parameters, where correlation is computed by vectorizing the corresponding matrices. \rv{We report the mean and standard error of the MAE across the 10 runs.}

\begin{table}[h!]
  \centering
  \begin{tabular}{|c|c|c|c|c|}
    \hline
    \textbf{$d$} & \multicolumn{2}{c|}{\textbf{A Estimation}} & \multicolumn{2}{c|}{\textbf{GG$^\top $ Estimation}} \\
    \cline{2-5}
    & \textbf{WOT} & \textbf{APPEX} & \textbf{WOT} & \textbf{APPEX} \\
    \hline
    3 & $0.351 \pm 0.04$ & $\mathbf{0.237 \pm 0.04}$ & $0.793 \pm 0.205$ & $\mathbf{0.147 \pm 0.030}$ \\
    \hline
    4 & $0.730 \pm 0.067$ & $\mathbf{0.328 \pm 0.041}$ & $1.549 \pm 0.439$ & $\mathbf{0.415 \pm 0.070}$ \\
    \hline
    5 & $0.912 \pm 0.060$ & $\mathbf{0.602 \pm 0.195}$ & $2.174 \pm 0.702$ & $\mathbf{0.362 \pm 0.039}$ \\
    \hline
    6 & $1.43 \pm 0.170$ & $\mathbf{0.358 \pm 0.020}$ & $18.010 \pm 8.636$ & $\mathbf{0.256 \pm 0.046}$ \\
    \hline
    7 & $1.480 \pm 0.132$ & $\mathbf{0.360 \pm 0.015}$ & $6.807 \pm 1.724$ & $\mathbf{0.345 \pm 0.037}$\\
    \hline
    8 & $1.862 \pm 0.137$ & $\mathbf{0.460 \pm 0.015}$ & $5.472 \pm 1.266$& $\mathbf{0.359 \pm 0.019}$\\
    \hline
    9 & $1.803 \pm 0.222$ & $\mathbf{0.487 \pm 0.016}$ & $8.134 \pm 3.024$ & $\mathbf{0.454 \pm 0.122}$ \\
    \hline
    10 & $1.670 \pm 0.241$ & $\mathbf{0.439 \pm 0.019}$ & $35.122 \pm 28.529$ & $\mathbf{0.317 \pm 0.025}$ \\
    \hline
  \end{tabular}
  \caption{Mean absolute error (MAE) of estimated drift and diffusion for dimensions $d$. \rv{Each box reports the mean $\pm$ the standard error.}}
  \label{tab: mae_table}
\end{table}

\begin{table}[h!]
  \centering
  \begin{tabular}{|c|c|c|c|c|}
    \hline
    $d$ & \multicolumn{2}{c|}{\textbf{A Estimation}} & \multicolumn{2}{c|}{\textbf{GG$^\top $ Estimation}} \\
    \cline{2-5}
    & \textbf{WOT} & \textbf{APPEX} & \textbf{WOT} & \textbf{APPEX} \\
    \hline
    3 & $0.996 \pm 0.001$ & $\mathbf{0.998 \pm 0.001}$ & $0.837 \pm 0.048$ & $\mathbf{0.985 \pm 0.005}$ \\
    \hline
    4 & $0.943 \pm 0.015$ & $\mathbf{0.987 \pm 0.005}$ & $0.729 \pm 0.039$& $\mathbf{0.865 \pm 0.031}$ \\
    \hline
    5 & $0.921 \pm 0.016$ & $\mathbf{0.952 \pm 0.030}$  & $0.728 \pm 0.040$ & $\mathbf{0.909 \pm 0.018}$ \\
    \hline
    6 & $0.794 \pm 0.040$ & $\mathbf{0.986 \pm 0.001}$ & $0.530 \pm 0.056$ & $\mathbf{0.961 \pm 0.007}$ \\
    \hline
    7 & $0.792 \pm 0.029$ & $\mathbf{0.988 \pm 0.001}$ & $0.595 \pm 0.037$ & $\mathbf{0.946 \pm 0.012}$ \\
    \hline
    8 & $0.699 \pm 0.035$ & $\mathbf{0.981 \pm 0.002}$ & $0.611 \pm 0.042$ & $\mathbf{0.949 \pm 0.006}$ \\
    \hline
    9 & $0.740 \pm 0.033$ & $\mathbf{0.978 \pm 0.002}$& $0.615 \pm 0.025$ & $\mathbf{0.919 \pm 0.033}$ \\
    \hline
    10 & $0.760 \pm 0.049$ & $\mathbf{0.983 \pm 0.001}$ & $0.641 \pm 0.041$ & $\mathbf{0.960 \pm 0.006}$ \\
    \hline
  \end{tabular}
    \caption{Correlation between estimated 
  and true drift and diffusion for dimensions $d$. \rv{Each box reports the mean $\pm$ the standard error.}}
  \label{tab: corr_table}
\end{table}

\begin{figure}[!h]
    \centering
    \begin{subfigure}{0.49\textwidth}
        \centering
        \includegraphics[width=\textwidth]{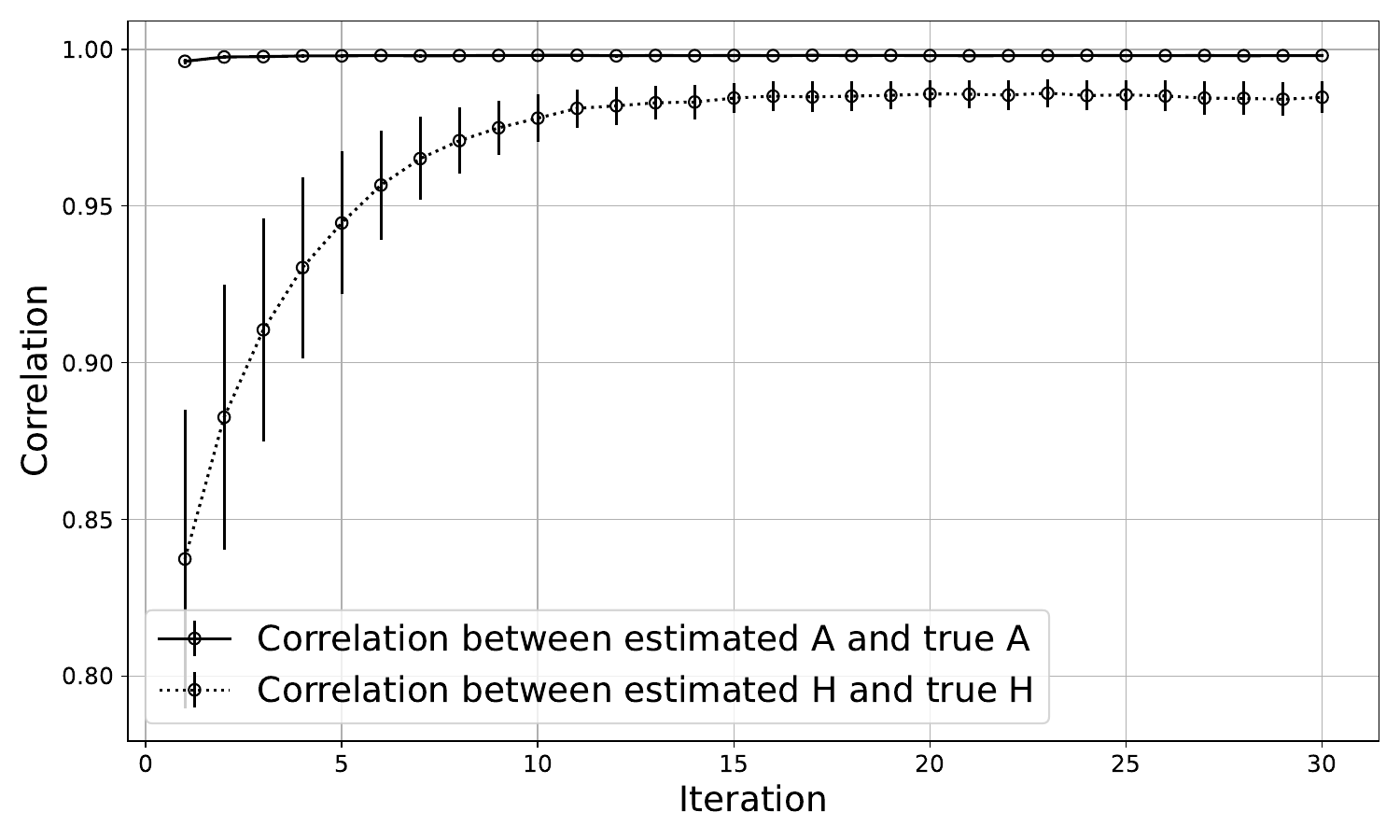}
        \caption{$d=3$}
    \end{subfigure}
    \begin{subfigure}{0.49\textwidth}
        \centering
        \includegraphics[width=\textwidth]{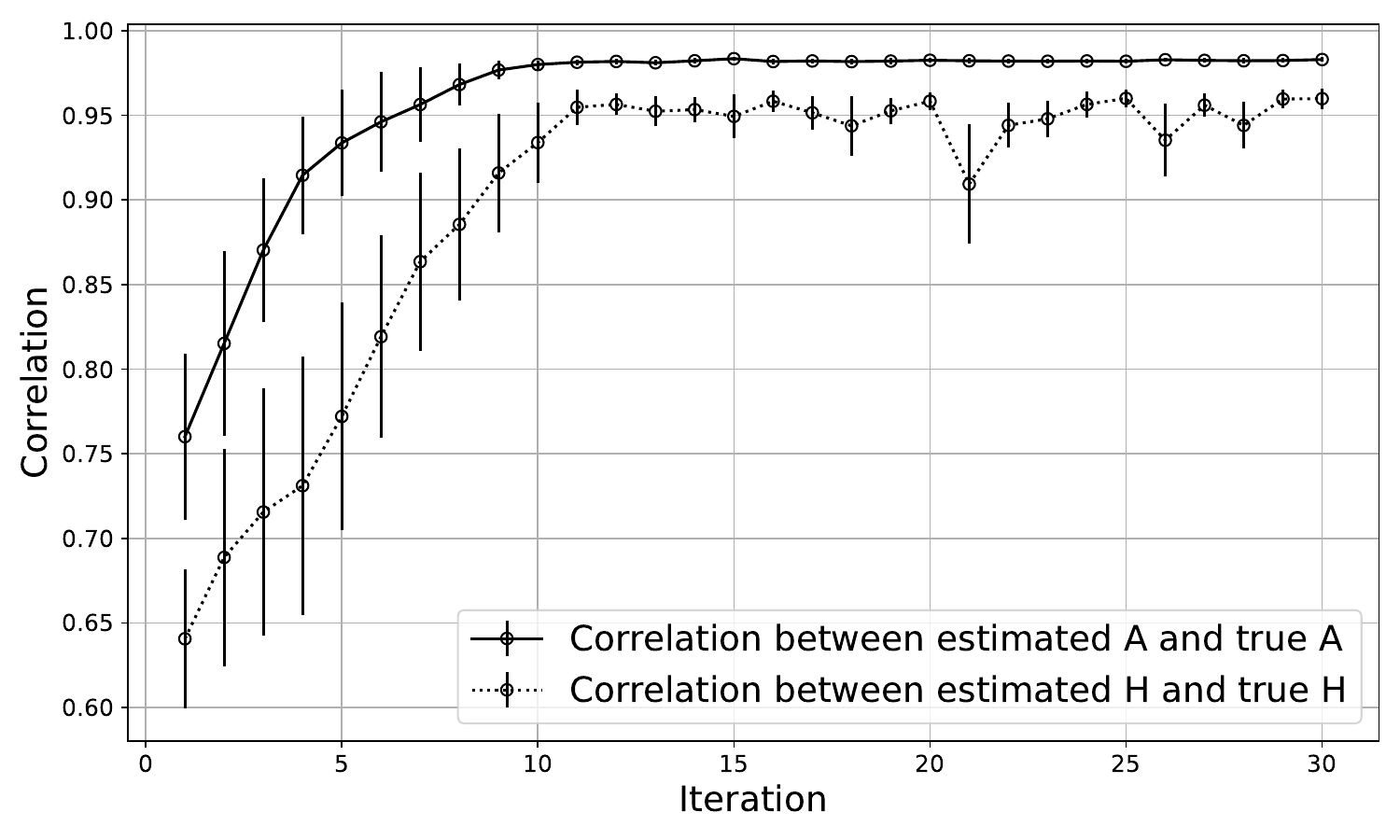}
        \caption{$d=10$}
    \end{subfigure}
    
    \caption{The correlation between the estimated and true SDE parameters is plotted per iteration across $10$ random linear additive noise SDEs for dimensions $3$ and $10$. \rv{Each point reports the mean of the 10 runs, while the bars represent the standard error.}}
    \label{fig:corr_exp_rand}
\end{figure}

The results, given in Tables \ref{tab: mae_table} and \ref{tab: corr_table}, demonstrate that APPEX estimates both SDE parameter sets robustly across all settings. Importantly, APPEX can handle arbitrary additive noise structures $H=GG^\top$, as evidenced by the fact that its drift and diffusion estimates have low MAE and high correlation with respect to the true parameters. This is significant because previous literature has focused on the setting of isotropic noise, but in practice, we expect correlated noise structures, as well as unequal noise along the main diagonal. In contrast, although WOT estimates the drift somewhat decently, it is unable to estimate the diffusion accurately, particularly in higher dimensions, since it is constrained to isotropic noise in its reference SDE. Figure \ref{fig:corr_exp_rand} shows how APPEX is able to re-orient incorrect diffusion priors to closely match the true diffusion of the underlying SDE. While running increasing numbers of iterations within APPEX generally leads to monotonic increases in parameter accuracy (as measured by correlation in Figure \ref{fig:corr_exp_rand}), we observe occasional instability due to the combination of finite observations, numerical limitations during trajectory inference from capping the maximum number of iterations of Sinkhorn, and the fact that the algorithm is building on previous (and wrong) estimates.

\subsection{Causal Discovery}
\label{sec: causal_discovery}

We conclude with experiments that test our method's ability to recover the underlying causal graph. We first consider causal discovery from systems without latent confounders, and then consider causal discovery from systems with latent pairwise confounders through correlated diffusion. 

In both settings, we generate random SDEs of dimension $d=3,5,10$ whose causal structure is determined by Erdos-Renyi graphs $\mathcal{G}(d,p)$, such that each of the $d(d-1)$ possible directed edges are included with probability $p$. As in Section \ref{sub: higher_dim_exps}, we ensure that the maximal eigenvalue of the randomly generated ground truth drift matrix is at most $1$ to prevent unbounded growth. Lemma \ref{prop: identify_causal_graph_SDE_add_noise} states that directed edges $e=i \to j \in E$  in the causal graph $\mathcal{G} = ([d], E, \tilde{E})$ are characterized by the condition $A_{j,i} \neq 0$. To simulate ground truth edges, we follow standard convention, by simulating edge weights $A_{j,i}$ uniformly via $\text{Unif}(-5, -0.5) \cup \text{Unif}(0.5, 5)$ \citep{runge2021necessary, reisach2021beware}. We infer the presence of an edge $i \to j$, if in the estimated drift, $\hat{A}_{j,i} > \epsilon$ (corresponding to a positive edge weight) or $\hat{A}_{j,i} <- \epsilon$ (corresponding to a negative edge weight). We choose our threshold $\epsilon = 0.5$ to match the minimum edge weight magnitude
$|A_{j,i}| \ge 0.5$ in our simulations. Inference accuracy of directed edges is then measured according to the structural Hamming distance between $\mathcal{G}$ and the estimated graph $\mathcal{\hat{G}}$, which considers the classification of positive directed edges, negative directed edges, and absent edges. The structural Hamming distance adds $1$ for every misclassified edge:
\begin{align}
d(\mathcal{G}, \hat{\mathcal{G}}) = \sum_{(i, j) \in [d] \times [d]} \mathbf{1}\left\{ \operatorname{sgn}(A_{j,i}) \neq \operatorname{sgn}(\hat{A}_{j,i})\mathbf{1}_{|\hat{A}_{j,i}| > \epsilon}  \right\},
\label{eq: hamming_A}
\end{align}
where the sum ranges over all ordered pairs including $i = j$, so that self-regulatory edges $i \to i$ are also included. If the observed system does not have any latent confounders from correlated diffusion, we enforce this by setting the diffusion matrix $G$ to be strictly diagonal. This ensures that the only edges in $\mathcal{G}$ are uni-directional edges from the drift. We set each of the $d$ diagonal entries of $G$ via $\text{Unif}(0.1,1)$. Conversely, when modeling pairwise latent confounders, we introduce correlated diffusion by selecting a random subset of columns in $G$ to have precisely two nonzero entries. In particular, we first  initialize all entries along the main diagonal of $G$ to be $1$, and for each of the selected columns, we set one non-diagonal entry $G_{i,j}=1$, to model the pairwise latent confounder $U_{i,j}$. All other entries are set to $0$. The number of selected columns with pairwise latent confounders is sampled from $\text{Unif}\{1, \cdots \lfloor 2d/3 \rfloor\}$, e.g. we sample from $\text{Unif}\{1, 2\}$ for $d=3$, $\text{Unif}\{1, 2, 3\}$ for $d=5$, and $\text{Unif}\{1, \dots 6\}$ for $d=10$. Lemma \ref{prop: identify_causal_graph_SDE_add_noise} states that the condition $H_{i,j} = G_{i,\cdot}G_{j,\cdot}^\top \neq 0$ identifies pairwise latent confounders $U_{i,j}$. Since all entries of $G$ are nonnegative in this construction, no cancellation occurs in $H = GG^\top$, and the confounded pairs are exactly those with $H_{i,j} \neq 0$, for which $H_{i,j} \ge 1$. We read the ground-truth adjacency directly off the construction, and infer the presence of
$U_{i,j}$ if the estimated diffusion satisfies $|\hat{H}_{i,j}| > \epsilon_H$, with
$\epsilon_H = 0.5$. Inference of edges from pairwise latent confounders is measured according to the structural Hamming distance, which adds $1$ for every misclassified pairwise latent confounder:
\begin{align}
d_{\text{latent}}(\mathcal{G}, \hat{\mathcal{G}}) = \sum_{1 \le i < j \le d}
\mathbf{1}\left\{ \mathbf{1}\{H_{ij} \neq 0\} \neq \mathbf{1}\{|\hat{H}_{ij}| > \epsilon_H\} \right\}.
\label{eq: hamming_H}
\end{align}

Table \ref{tab: causal_sufficiency_table} reports the mean structural Hamming distances for directed edges from causal graphs estimated by WOT and APPEX, in the setting without latent confounders. APPEX outperforms WOT across all tested dimensions $d \in \{3, 5, 10\}$ and edge probabilities $p \in \{0.1, 0.25, 0.5\}$, with the discrepancy becoming more pronounced as the causal structure grows in complexity at higher dimensions and edge densities. We observe that APPEX can consistently infer complex causal relationships, whereas WOT's inference is often limited, due to misspecified diffusion. For example, the first row in Figure  \ref{fig:causal_sufficiency_plots} demonstrates that APPEX recovers the system's negative feedback structure, while WOT introduces additional cycles in order to compensate for a poorly specified reference diffusion.

Table \ref{tab: latent_table} reports the mean structural Hamming distances for \rv{drift-related edges and edges from pairwise latent confounders} from causal graphs estimated by WOT and APPEX, in the setting with correlated diffusion. The probability $p$ of connecting a directed edge $i \to j$ was fixed at $p=0.25$. As in the setting with uncorrelated diffusion, APPEX outperforms WOT in all evaluated settings, particularly for higher-dimensional systems. 





\begin{table}[h!]
  \centering
  \begin{tabular}{|c|c|c|c|c|c|c|}
    \hline
    $d$ & \multicolumn{2}{c|}{$p = 0.1$} & \multicolumn{2}{c|}{$p = 0.25$} & \multicolumn{2}{c|}{$p = 0.5$} \\
    \cline{2-7}
    & \textbf{WOT} & \textbf{APPEX} & \textbf{WOT} & \textbf{APPEX} & \textbf{WOT} & \textbf{APPEX} \\
    \hline
    3 & $0.40  \pm 0.40 $ & $\mathbf{0.00 \pm 0.00}$ & $1.80 \pm 0.76$ & $\mathbf{0.30 \pm 0.30}$ & $1.40 \pm 0.65$ & $\mathbf{0.00 \pm 0.00}$ \\
    \hline
    5 & $3.80 \pm 2.13$ & $\mathbf{0.10 \pm 0.10}$ & $3.56 \pm 1.70$ & $\mathbf{0.00 \pm 0.00}$ & $4.40 \pm 1.14$ & $\mathbf{0.60 \pm 0.22}$ \\
    \hline
    10 & $38.3 \pm 8.53$ & $\mathbf{0.50 \pm 0.40}$ & $48.6\pm 6.55$ & $\mathbf{0.80 \pm 0.39}$ & $37.4 \pm 4.66$& $\mathbf{3.30 \pm 0.70}$ \\
    \hline
  \end{tabular}
  \caption{Average Structural Hamming Distance (lower is better) with varying dimensions $d$ and random edge probabilities \( p \). Each box reports the mean $\pm$ the standard error.}
  \label{tab: causal_sufficiency_table}
\end{table}

\begin{table}
  \centering
  \begin{tabular}{|c|c|c|c|c|}
    \hline
    $d$ & \multicolumn{2}{c|}{Hamming distance for \rv{drift} edges} & \multicolumn{2}{c|}{Hamming distance for latent confounders} \\
    \cline{2-5}
    & \textbf{WOT} & \textbf{APPEX} & \textbf{WOT} & \textbf{APPEX} \\
    \hline
    3 & $1.20  \pm 0.51$ & $\mathbf{0.7 \pm 0.40}$ & $1.00 \pm 0.30$ & $\mathbf{0.00 \pm 0.00}$ \\
    \hline
    5 & $4.80\pm 1.47$ & $\mathbf{1.60 \pm 0.87}$ & $3.50 \pm 0.64$ & $\mathbf{0.10 \pm 0.10}$ \\
    \hline
    10 & $38.00 \pm 6.30$ & $\mathbf{3.00 \pm 0.77}$ & $23.80 \pm 3.86$ & $\mathbf{2.90\pm 1.39}$  \\
    \hline
  \end{tabular}
  \caption{Average Structural Hamming Distance for directed edges, and latent confounders (lower is better) with varying dimensions $d$ and random edge probability $p=0.25$. Each box reports the mean $\pm$ the standard error.}
  \label{tab: latent_table}
\end{table}
\begin{figure}
    \centering
    \par\vspace{-5mm}
    \includegraphics[width=0.86\linewidth]{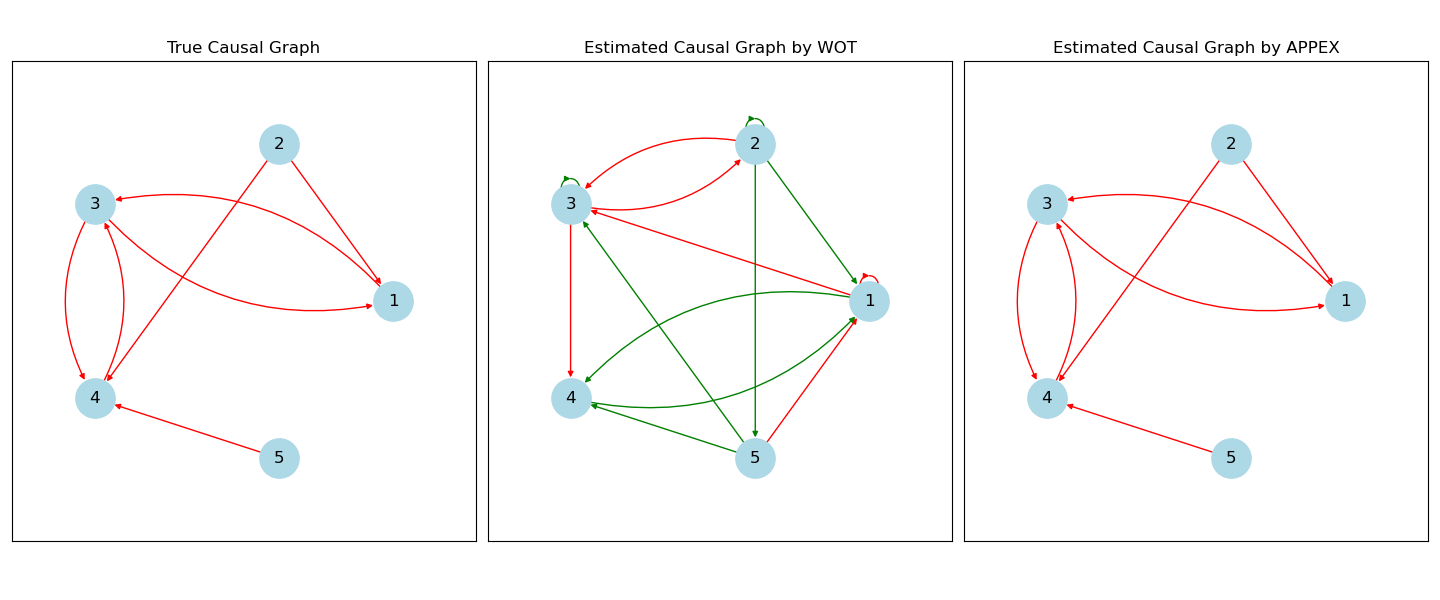}
    \par\vspace{-5mm}
    \includegraphics[width=0.86\linewidth]{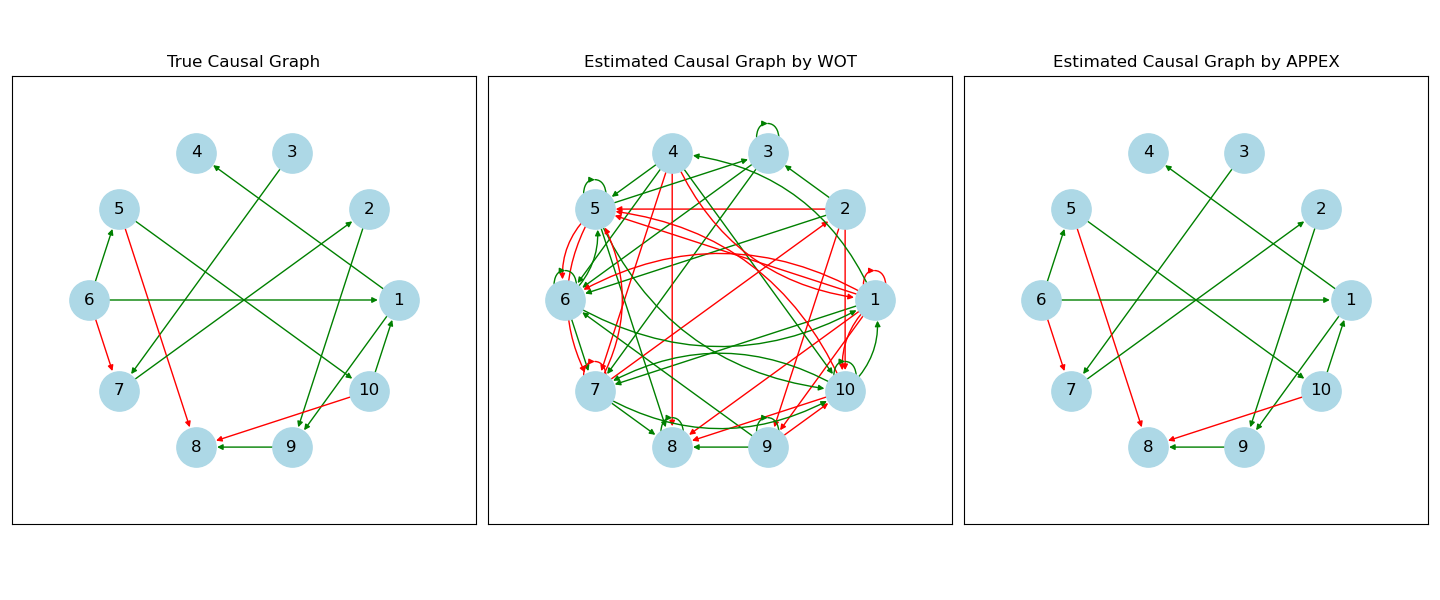}
    \par\vspace{-5mm}
     \includegraphics[width=0.86\linewidth]{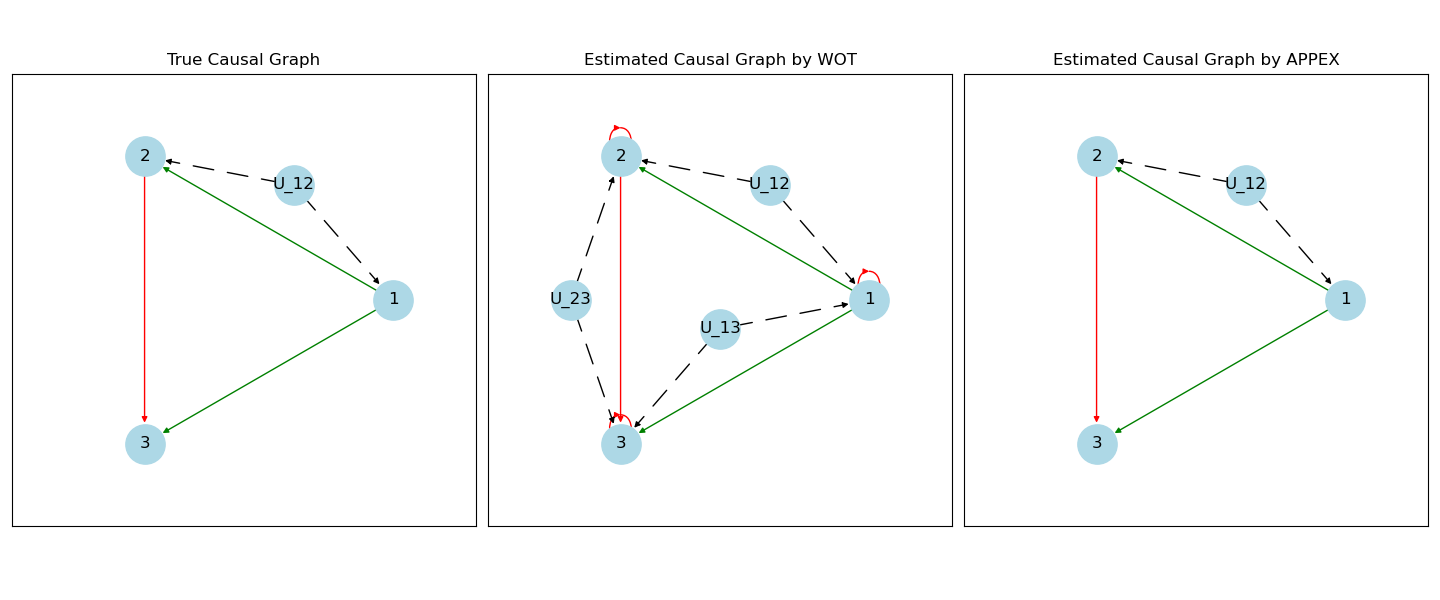}
     \par\vspace{-5mm}
    \includegraphics[width=0.86\linewidth]{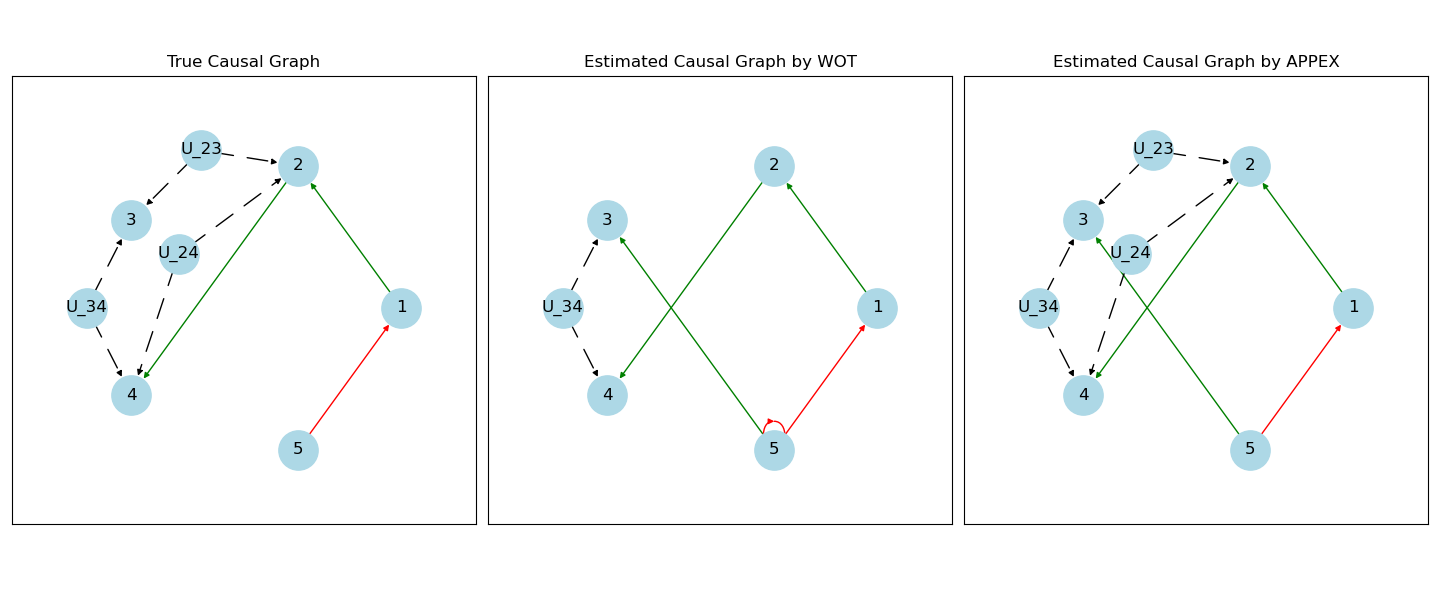}
    \par\vspace{-5mm}
    \caption{The true and estimated causal graphs by WOT and APPEX for four random SDEs of varying dimensions are illustrated. Red edges $i \to j$ represent negative edge weights, such that $X^{(i)}$ negatively regulates  $X^{(j)}$. Green edges $i \to j$ represent positive edge weights, such that $X^{(i)}$ positively regulates  $X^{(j)}$. Dashed edges represent the effect of a latent confounder $U_{i,j}$ on observed variables $i,j$ through correlated diffusion.}
    \label{fig:causal_sufficiency_plots}
\end{figure}


\section{Discussion}
\label{sec:Conclusions}
Many recent works have focused on inferring the drift or diffusion from an SDE's temporal marginals, yet the problem of non-identifiability remained largely unaddressed \citep{weinreb2018fundamental, lavenant2021towards, chizat2022trajectory, shen2024learning, zhang2024joint, terpin2024learning, brogat2024learning}. \rv{To resolve this issue, we first provide necessary and sufficient identifiability conditions for linear additive noise SDEs in Theorem \ref{thm: identifiability}, which shows that non-identifiability of the drift and diffusion parameters is feasible if and only if the initial distribution $X_0$ is weakly auto-rotationally invariant. Building off Theorem \ref{thm: identifiability}, we introduce Theorem \ref{thm:stronger}, which proves that the set of drift and diffusion parameters that produce non-identifiable observations has measure zero, even if the initial distribution obeys auto-rotational invariance. Our theory therefore clarifies that non-identifiability of linear drift and diffusion from the process's marginals can only be produced under exceedingly rare circumstances. We complement our identifiability theory by showing that directed edges and latent confounders in the system's causal graph can be recovered from the drift and diffusion parameters, for both linear and nonlinear additive noise SDEs.} This provides a theoretical foundation for learning the precise dynamics and causal structure of a linear additive noise process, solely from its temporal marginals, without relying on the usual prior knowledge assumptions and model restrictions.

To build on our theoretical contributions, we also introduce APPEX, a method that simultaneously learns drift and diffusion from temporal marginals. APPEX iterates towards the optimal parameter set by alternating between trajectory inference, via anisotropic entropy-regularized optimal transport, and MLE parameter estimation. Since each of these subprocedures are asymptotically optimal (Lemma \ref{lemma: APPEX_decrease_KL}), given a finite set of the process's true marginals, APPEX improves its estimates with each iteration. When tested on hundreds of randomly generated linear additive noise SDEs between dimensions $3$ to $10$, APPEX is able to effectively recover the process's drift, diffusion, and causal graph, while significantly outperforming Waddington-OT, which is one of the most popular inference methods in single-cell biology. 

Our theoretical and algorithmic contributions offer a number of insights to advance statistical inference methods from temporal marginals. First, our framework removes common model restrictions to expand the set of dynamics that can be inferred from marginals. This includes negative feedback loops and repressilator dynamics, which both arise from rotational drift, and correlated noise, which arises from anisotropic diffusion. These features are common in real-world datasets, and are especially relevant for accurate trajectory and GRN inference from single-cell data \citep{weinreb2018fundamental,santos2024learning}. Second, we demonstrate that joint identification of the unique drift and diffusion parameters is theoretically and practically possible in the ``marginals-only" setting. Previous and concurrent works either estimate only one parameter given the other \citep{weinreb2018fundamental,lavenant2021towards, chizat2022trajectory, shen2024learning,zhang2024joint,forrow2024consistent} or estimate both without establishing identifiability \citep{terpin2024learning, brogat2024learning}. We note that even if a user is only interested in one parameter, APPEX offers the ability to estimate it, regardless of how the other parameter is initialized. This is particularly significant given that a misspecified diffusion typically leads to a misidentified drift, and vice-versa. \rv{Importantly, the identifiability conditions for linear additive noise SDEs are generic, since non-identifiability requires both a pathological initial distribution, which satisfies (weak) auto-rotational invariance, as well as pathologically aligned SDE parameters, which preserve the symmetry invariance across observed marginals. These theoretical insights provide a principled path for developing new and improved methods that can learn the full SDE purely from observational snapshots.}




\subsection{Limitations and Future Work}
Although this paper presents several breakthroughs for the identification of SDEs from marginal snapshots, we acknowledge several limitations and opportunities for future work. 


First, our identifiability result is restricted to time homogeneous linear additive noise SDEs. Although this is a popular model for real-world processes, including for GRN inference \citep{rohbeck2024bicycle, zhang2024joint}, it would be beneficial to characterize identifiability conditions under more general models. We note that characterizing identifiability is particularly challenging for non-additive noise models, where the noise depends on the process $X_t$, since the same marginal observations can be explained by a nonlinear additive noise SDE and a linear multiplicative noise SDE \citep{coomer2022noise}. For this reason, it may be easier to consider identifiability for parametric families of nonlinear drift under the additive noise model. We note that we have already proven results about causal discovery for additive noise SDEs in Section \ref{sec:causal_graph} and that APPEX is also formulated for general additive noise SDEs.  Although the MLE for nonlinear drift may not have a closed form solution, optimizers like BFGS \citep{lavenant2021towards, chizat2022trajectory, wang2024generator} and neural networks \citep{shen2024learning} have been successfully used for SDE parameter estimation on real and synthetic data. 

Second, our theory considers the infinite data setting, given the process's true marginals $\{p_i\}_{i=0}^{N-1}$, and we have not studied asymptotic convergence rates. In particular, our identifiability result was established given continuously observed marginals, which corresponds to obtaining infinite samples from infinite marginals. Similarly, the theory surrounding the optimality of APPEX assumes that we have infinite samples from a finite set of marginals. We leave the asymptotic analysis of APPEX for future work. Directions for future work include the sample complexity of entropic optimal transport for trajectory inference, which requires asymptotically more samples per marginal as we decrease the entropic regularization (decreasing time granularity in the SDE setting) \citep[Theorem 3]{genevay2019sample}, and the fact that the MLE of the AR(1) process from multiple trajectories is slightly biased \citep{ledolter2009estimation}. We note that APPEX could be combined with previous methods that relax the marginal constraints (see \citet{lavenant2021towards}). This could improve estimates in low data settings, particularly in the biological setting where taking measurements destroys samples. 

Third, even in the asymptotic setting, the nonconvex nature of APPEX's alternating projection optimization problem makes the prospect of proving convergence difficult. Although our experiments indicate that APPEX's estimates approach the true solution, regardless of the initial guesses for the drift $A$ and observational diffusion $H$, we have not yet proved convergence due to the non-convexity of the optimization space. A formal proof of convergence and consistency would further strengthen APPEX's theoretical foundation.

Finally, we note that the trajectory inference step of APPEX, which we solve using Sinkhorn's algorithm, is computationally intensive. For $M$ samples per marginal, each iteration of the standard dense Sinkhorn algorithm requires $O(M^2)$ arithmetic operations once the transport kernel has been constructed. Consequently, $J$ iterations require $O(JM^2)$ operations. The required number of iterations depends on the stopping tolerance, the cost matrix, and the entropy regularization parameter \citep{dvurechensky2018computational}. 
We note that accurate trajectory inference relies on a large number of samples ($M$), a small entropic regularization $\epsilon^2 \sim \delta t$, and potentially high-dimensional data. Our method is therefore impractical on very large datasets, which are prevalent in biological applications. It would therefore be advantageous to solve the associated Schr\"odinger bridge problem using more efficient methods. We note that recent work makes use of closed form solutions to the SB problem when $X_0$ is a degenerate Dirac distribution, to increase computational efficiency \citep{huang2024one}. We also note that the SB problem can be expressed via forwards and backwards drifts \citep{vargas2021solving}, and can therefore be solved via generative diffusion models \citep{de2021diffusion, shi2024diffusion}. Thus, when a closed form solution is not known, generative diffusion models may be more efficient solvers for the trajectory inference problem. Concurrent work estimates drift and diffusion using other methods, including JKO energy minimization \citep{terpin2024learning} and reproducing kernel Hilbert spaces to minimize distance to the Fokker-Planck solution \citep{brogat2024learning}. We have not yet compared the accuracy and computational efficiency of these methods against APPEX, and a systematic comparative analysis would be an interesting direction for future research.

\acks{The authors would like to thank United Therapeutics for supporting this research. GS and AW also acknowledge the support of the Burroughs Wellcome Fund.}


\newpage

\appendix

\section{Properties of the Linear Additive Noise SDE}
\label{sec:linear_additive_noise}
Let $X_t$ evolve according to a linear additive noise SDE \eqref{eq: linear_additive_noise_SDE}. Then its solution can be obtained using the integrating factor $e^{-At}$ and is given by \citep[Sec 4.3]{sarkka2019applied}: 
\begin{align}
X_t = e^{At}X_0 + \int_0^t e^{A(t-s)}GdW_s.
\end{align}
By the independent increments of $W_t$, we may interpret $X_t$ as being composed of two independent r.v.'s. The first term amounts to a transformation of the initial distribution $X_0 \to e^{At}X_0$ under the noiseless linear ODE. The second term is an Ito integral, which is in fact a mean-$0$ Gaussian with time-dependent variance $\Sigma_t = \int_0^t e^{A(t-s)}He^{A^\top (t-s)}ds$. Indeed, since the integrand is deterministic, Gaussianity follows from \citep[Theorem A.7]{oksendal2013stochastic}, and the mean and covariance respectively follow from the non-anticipating property and Ito's isometry.

\section{Non-Identifiable Linear Additive Noise SDEs}
\label{sec:Noniden_examples}

\begin{figure}[h!]
    \centering
    
    \begin{subfigure}[t]{\textwidth}
        \centering
        \begin{minipage}{0.6\textwidth}
            \centering
            \includegraphics[width=\textwidth]{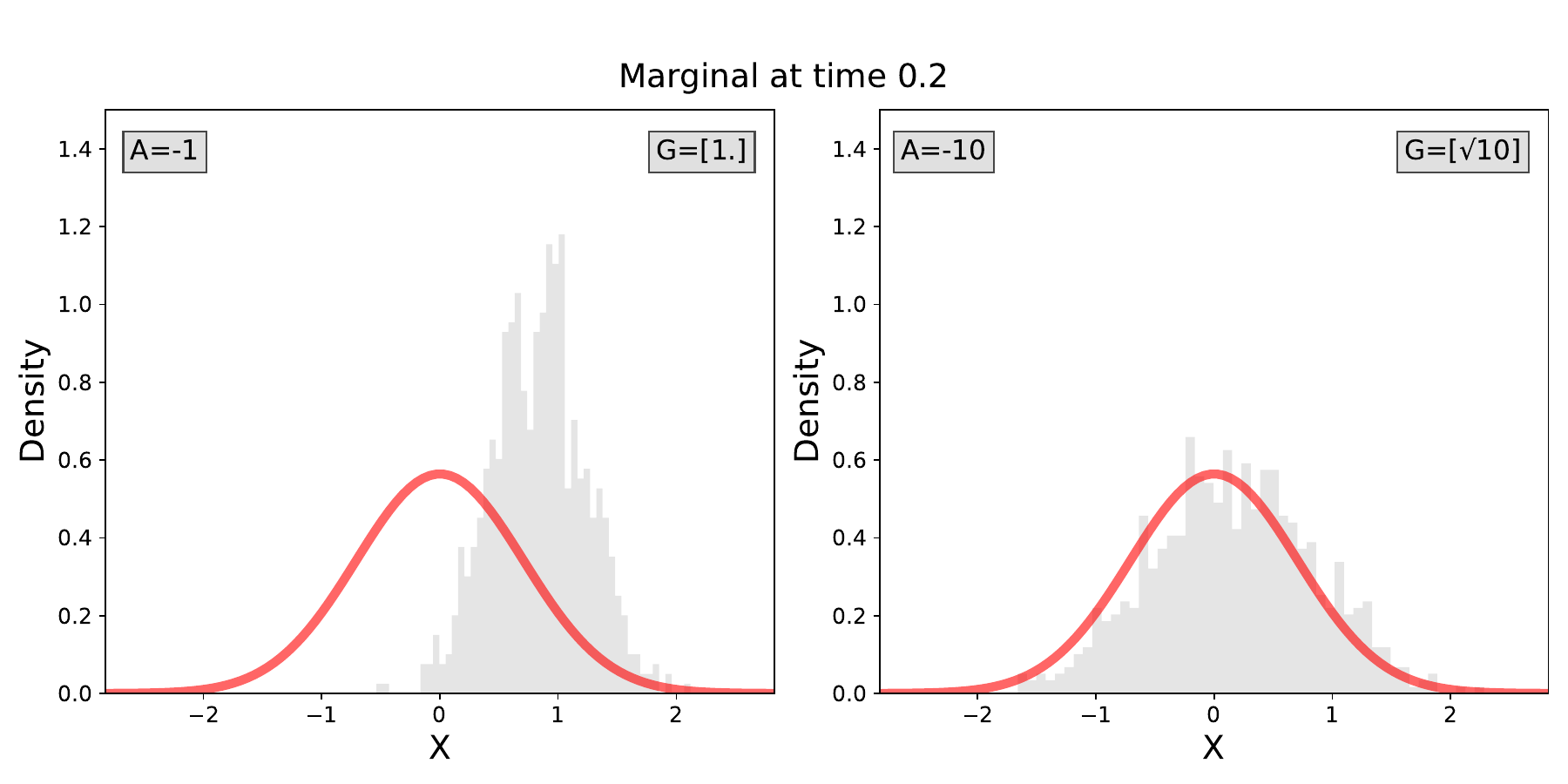}
        \end{minipage}%
        \begin{minipage}{0.4\textwidth}
            \vspace{-10pt} 
            \subcaption{This corresponds to Example~\ref{eq: non_identifiability_same_stationary} with $X_0 \sim \delta_1$. The SDE with $A = -10, H = 10$ (right) has converged to the stationary distribution (red) at time $t = 0.2$, while the SDE with $A =-1, H = 1$ (left) has not yet sufficiently drifted or diffused. }
        \end{minipage}
        \label{fig:intervene_ex1}
    \end{subfigure}

    \begin{subfigure}[t]{\textwidth}
        \centering
        \begin{minipage}{0.6\textwidth}
            \centering
            \includegraphics[width=\textwidth]{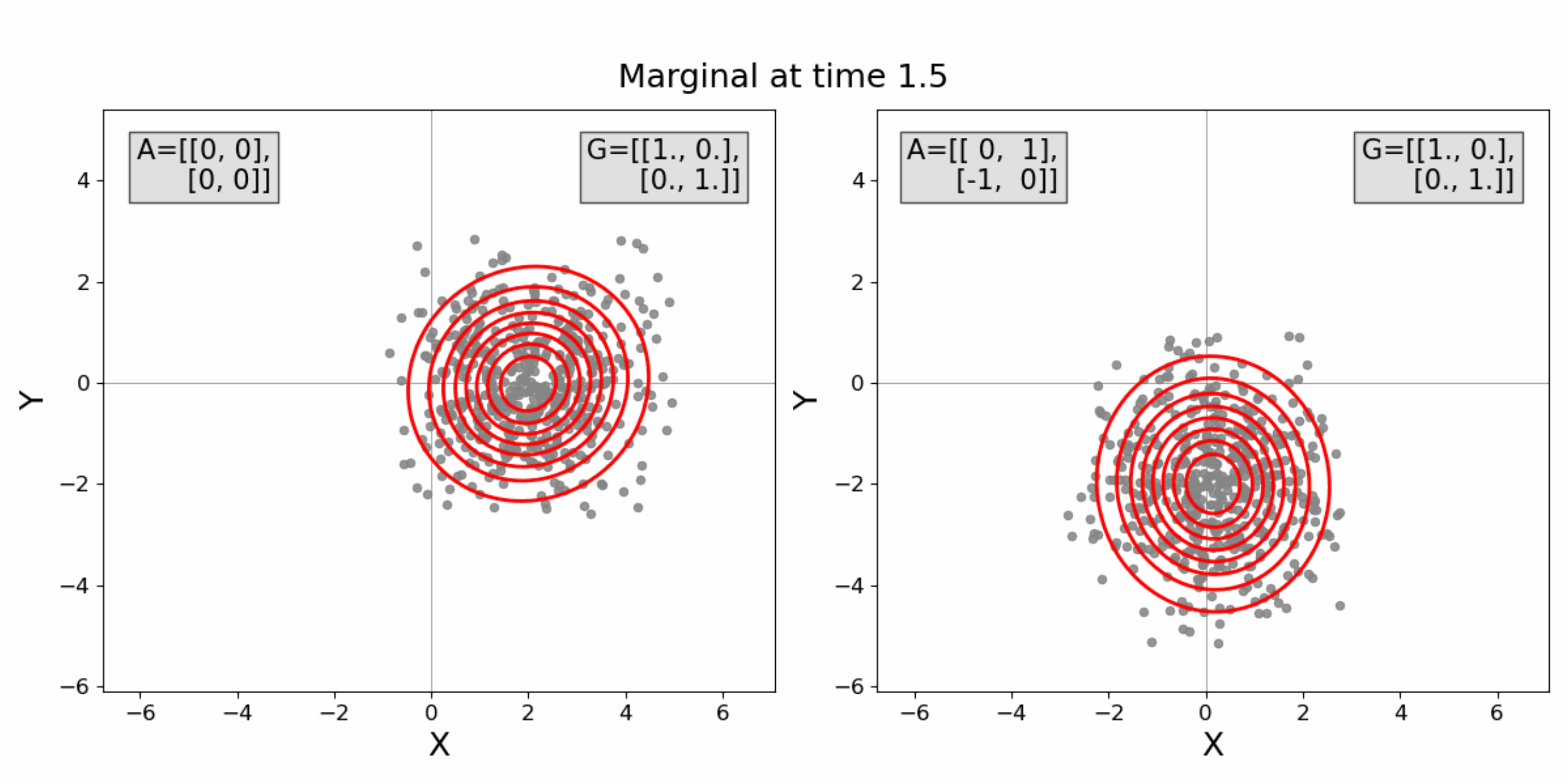}
        \end{minipage}%
        \begin{minipage}{0.4\textwidth}
            \vspace{-10pt} 
            \subcaption{This corresponds to Example~\ref{eq: non_identifiability_rotation}, with $X_0 \sim \text{Unif}\{ (2,0), (2,0.1)\}$. The first SDE is irrotational and maintains mean $(2,0.05)$. The second SDE (right) rotates about the origin. At time $t=1.5$, the mean has shifted to approximately $(0,-2)$.}
        \end{minipage}
        \label{fig:intervene_ex2}
    \end{subfigure}

    \begin{subfigure}[t]{\textwidth}
        \centering
        \begin{minipage}{0.6\textwidth}
            \centering
            \includegraphics[width=\textwidth]{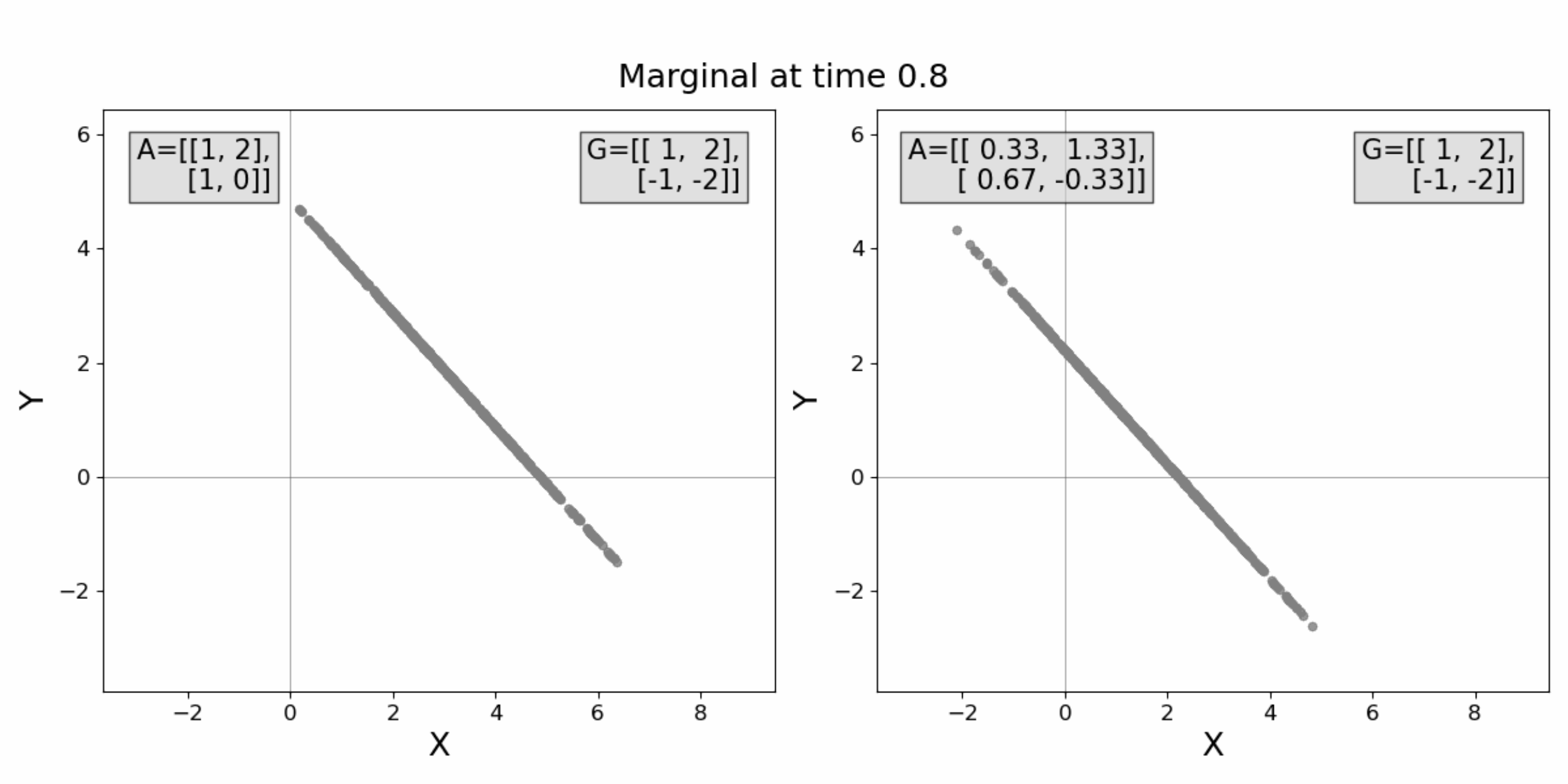}
        \end{minipage}%
        \begin{minipage}{0.4\textwidth}
            \vspace{-10pt} 
            \subcaption{This corresponds to Example~\ref{eq: non_identifiability_wang}, with $X_0\sim \text{Unif}\{(1,0)^\top, (0,1)^\top\}$. The first SDE (left) has eigenvalue $2$ for $(2,1)$, while the second (right) has eigenvalue $1$, leading to more extreme intercepts in the left plot at $t = 0.8$.}
        \end{minipage}
        \label{fig:intervene_ex3}
    \end{subfigure}

    \caption{Marginals are plotted at various times for SDE pairs following various non-auto-rotationally invariant initial distributions. Importantly, these initializations enable previously non-identifiable systems to be distinguished. }
    \label{fig:initialise_on_examples}
\end{figure}

 We summarize three classical examples of non-identifiability  from the literature \citep{lavenant2021towards, shen2024multi, wang2024generator,hashimoto2016learning,weinreb2018fundamental}. In all examples, the processes $X_t$ and $Y_t$ share the same time marginals. In Figure \ref{fig:initialise_on_examples}, we visualize how changing the initial distribution $X_0$ to be non auto-rotationally invariant resolves the non-identifiability.

\begin{example}[Starting at stationary distribution \citep{lavenant2021towards}]
\label{eq: non_identifiability_same_stationary}
This example is motivated by the fact that multiple SDEs share the same stationary distribution. Hence, if $p_0$ is initialized at the stationary distribution, then these SDEs will be non-identifiable from their marginals. We provide a simple example.
\begin{align}
    \mathrm{d}X_t &= -X_t \: \mathrm{d}t+ \mathrm{d}W_t, &  X_0 \sim \mathcal{N}(0, \frac{1}{2})\\
    \mathrm{d}Y_t &= -10Y_t \: \mathrm{d}t+ \sqrt{10}\mathrm{d}W_t, & Y_0 \sim \mathcal{N}(0, \frac{1}{2})
\end{align}
Here, both SDEs have the same stationary distribution $\mathcal{N}(0, \frac{1}{2})$ despite having significantly different individual trajectories. Indeed, the stationary distribution of a $1$-dimensional $0$-mean Ornstein-Uhlenbeck (OU) process with drift $-\lambda X_t$ and diffusion $\sigma$ is  $\mathcal{N}(0, \frac{\sigma^2}{2\lambda})$, which depends only on the drift:diffusivity ratio $\lambda/\sigma^2$ \citep{lavenant2021towards}. For multivariate OU processes with drift $A$ and observational diffusion $H$, the stationary distribution $\mathcal{N}(0, \Sigma)$ depends only on the relationship $\Sigma A^\top + A \Sigma = -H$ \citep{mandt2016variational}. We note that this example also generalizes easily to the case where one component of $X_0$ is a mean-$0$ Gaussian. 
\end{example}


\begin{example}[Rotation around process mean \citep{shen2024multi}]
\label{eq: non_identifiability_rotation}

In this example, the first SDE governing $X_t$ is $2$-dimensional Brownian motion and the second SDE governing $Y_t$ adds a divergence-free rotational vector field $(x,y) \to (y, -x)$ about $(0,0)$, which is undetectable if the initial distribution $p_0$ is rotationally invariant.
\begin{align}
    \mathrm{d}X_t &= \mathrm{d}W_t, &X_0 \sim \mathcal{N}(0, I_d)\\
    \mathrm{d}Y_t &= \begin{bmatrix} 0 & 1 \\ -1 & 0  \end{bmatrix}Y_t \: \mathrm{d}t
    + \mathrm{d}W_t, &  Y_0 \sim \mathcal{N}(0, I_d)
\end{align}
Hence, if $p_0$ is an isotropic distribution with mean $(0,0)$, then $X_t$ and $Y_t$ are non-identifiable from one another \citep{shen2024multi}. This can also be shown directly with the Fokker-Planck equation \eqref{eq: FP_linear_additive_noise}, since $\nabla \cdot \big((Ax) p(x,t)\big) = \nabla p(x,t) \cdot Ax + p(x,t) \nabla \cdot (Ax) =  \nabla p(x,t) \cdot Ax = 0$, since $\nabla \cdot (Ax) = \mathrm{tr}(A) = 0$ and the level sets of $p$ are parallel to the rotational vector field $Ax = \begin{bmatrix} 0 & 1 \\ -1 & 0  \end{bmatrix}x$.
\end{example}
\begin{example}[Degenerate rank \citep{wang2024generator}]
\label{eq: non_identifiability_wang} In this example motivated by \citet{wang2024generator}, trajectories (and therefore temporal marginals) of distinct SDEs are non-identifiable if the process is degenerate, limiting observations to a strict subspace of $\R^d$.
\begin{align}
    \mathrm{d}X_t &= \begin{bmatrix} 1 & 2 \\ 1 & 0  \end{bmatrix} X_t \: \mathrm{d}t
    +  \begin{bmatrix} 1 & 2 \\ -1 & -2  \end{bmatrix}\mathrm{d}W_t, &X_0 = \begin{bmatrix}
    1 \\ -1
    \end{bmatrix}\\
    \mathrm{d}Y_t &=  \begin{bmatrix} 1/3 & 4/3 \\ 2/3 & -1/3  \end{bmatrix} Y_t \: \mathrm{d}t
    + \begin{bmatrix} 1 & 2 \\ -1 & -2  \end{bmatrix}\mathrm{d}W_t, &Y_0 = \begin{bmatrix}
    1 \\ -1
    \end{bmatrix}
\end{align}

The drift matrices $\begin{bmatrix} 1 & 2 \\ 1 & 0  \end{bmatrix}, \begin{bmatrix} 1/3 & 4/3 \\ 2/3 & -1/3  \end{bmatrix}$ each have eigenvector $X_0 = \begin{bmatrix}
    1 \\ -1
    \end{bmatrix}$ with eigenvalue $-1$. Moreover, the diffusion is rank-degenerate with column space span$(\begin{bmatrix} 1 \\ -1
        
    \end{bmatrix})$. Thus, both SDEs will have identical behaviour along span$(\begin{bmatrix} 1 \\ -1
        
    \end{bmatrix})$ and are only differentiated by their behaviour elsewhere. Given $X_0 = \begin{bmatrix}
    1 \\ -1
    \end{bmatrix}$, the processes will stay within span$(\begin{bmatrix} 1 \\ -1
        
    \end{bmatrix})$ and evolve identically. We note that this non-identifiability holds even when we observe trajectories \citep{wang2024generator}, whereas the first two examples are identifiable from trajectories but not identifiable from marginals.\\
\end{example}

\subsection{Properties of Auto-Rotationally Invariant Random Variables}
\label{sec: rotations_appendix}

In this section, we present additional theory related to $\Sigma$-generalized rotations and auto-rotational invariance. The first result that we present provides intuition about the action of a $\Sigma$-generalized rotation, and how it relates to classical rotations.
\begin{lemma}
Let \(\Sigma \succ 0\) and \(B \in \R^{d \times d}\) be such that
\[
B\Sigma + \Sigma B^\top  = 0.
\]
Then, $R = \Sigma^{-1/2} B\,\Sigma^{1/2}$  is skew-symmetric, and for all \(\theta\in\R\),
\[
e^{B\theta} = \Sigma^{1/2}\,e^{R\theta}\,\Sigma^{-1/2}.
\]
In other words, the $\Sigma$-generalized rotation \(e^{B\theta}\) acts as a classical rotation \(e^{R\theta}\) after the change of variables \(y = \Sigma^{-1/2}x\), which rescales $\R^d$ along the principal axes of $\Sigma$.
\label{lem:conjugate_rotation}
\end{lemma}

\begin{proof}
Since \(\Sigma \succ 0\), there exists a unique and invertible symmetric positive definite square root \(\Sigma^{1/2}\). We then define
\[
R = \Sigma^{-1/2} B\,\Sigma^{1/2}, \qquad R^\top =  \Sigma^{1/2} B^\top\,\Sigma^{-1/2}.
\]
Then, since \(B\Sigma + \Sigma B^\top  = 0\), we can multiply on the left by \(\Sigma^{-1/2}\) and on the right by \(\Sigma^{-1/2}\) to yield
\[
0 = \Sigma^{-1/2} B\,\Sigma^{1/2} + \Sigma^{1/2}\,B^\top \,\Sigma^{-1/2} = R + R^\top,
\]
which proves that \(R\) is skew-symmetric, and therefore \(e^{R\theta}\) defines a classical rotation in \(\R^d\), such that
\[
e^{R\theta} = e^{\Sigma^{-1/2}B\Sigma^{1/2}\theta}=\Sigma^{-1/2}e^{B\theta}\Sigma^{1/2},
\]
where the last equality can be seen by writing $e^{R\theta} = \sum_{n=0}^{\infty} \theta^n \frac{R^n}{n!}$ and noting that $R^n = (\Sigma^{-1/2}B\Sigma^{1/2})^n = \Sigma^{-1/2}B^n\Sigma^{1/2}$. Rearranging yields $e^{B\theta} = \Sigma^{1/2}\,e^{R\theta}\,\Sigma^{-1/2}$ as desired.

To see that the change of variables \(y = \Sigma^{-1/2} x\) amounts to an anisotropic scaling, write the eigendecomposition \(\Sigma = QDQ^\top\) with \(Q\) orthogonal and \(D\) diagonal with positive entries. Then \(\Sigma^{-1/2} = QD^{-1/2}Q^\top\), i.e. the map rescales $\R^d$ by $D^{-1/2}$ along the principal axes of $\Sigma$, sending the covariance ellipsoid $\{x: x^\top\Sigma^{-1}x = 1\}$ to the unit sphere.
\end{proof}

We now show that under auto-rotational invariance (Definition \ref{def:full_ARI}), the corresponding $\Sigma$ that describes the geometry of the $\Sigma$-generalized rotations naturally corresponds to the distribution's second moment.
\begin{lemma}
Let $X$ be a r.v. with second moment $\Sigma = \E[XX^\top]$ and suppose that
$e^{B\theta}X \sim X$ for all $\theta \in \R$. Then, $B\Sigma + \Sigma B^\top = 0$.
\label{lem: lyapunov_condition}
\end{lemma}
\begin{proof}
Let
$C(\theta) = \E[(e^{B\theta}X)(e^{B\theta}X)^\top]$. $e^{B\theta}X \sim X$ implies that for all $\theta \in R$,
$C(\theta)$ stays constant, such that $C(\theta) = \E[XX^\top] = \Sigma$. Directly computing $C(\theta)$
and its derivative yields
\begin{align*}
C(\theta) &= e^{B\theta}\Sigma e^{B^\top \theta} = \Sigma\\
C'(\theta) &= Be^{B\theta}\Sigma e^{B^\top \theta} + e^{B\theta}\Sigma B^\top e^{B^\top \theta} = 0
\end{align*}
and we plug in $\theta=0$ to obtain
\begin{align*}
C'(0) &= B\Sigma + \Sigma B^\top = 0.
\end{align*}
\end{proof}

We note that the above derivation also follows through when we let $\Sigma = \cov(X)$ and redefine $C(\theta)$ accordingly. It is however often more convenient to work with the non-centered geometry $\Sigma = \E[XX^\top]$. For example, we can show that if $\Sigma = \E[XX^\top] \succ 0$, the invariance confines the orbits of $e^{B\theta}$ to the ellipsoids determined by $\Sigma$, and forces the density to be constant on these orbits.

\begin{lemma}
\label{lemma:level_sets}
Let $X$ be a random variable with $\Sigma = \E[XX^\top] \succ 0$ and continuous density $p$,
and suppose $e^{B\theta}X \sim X$ for all $\theta \in \R$. Then $p$ is constant along each orbit
$\mathcal{O}_x = \{e^{B\theta}x : \theta \in \R\}$, and $\mathcal{O}_x$ lies inside the ellipsoid
$\{y : y^\top\Sigma^{-1}y = x^\top\Sigma^{-1}x\}$.
\end{lemma}

\begin{proof}
First, we show that the orbit $\mathcal{O}_x$ lies inside the ellipsoid. As
established in the proof of Lemma \ref{lem: lyapunov_condition}, we have 
$e^{B\theta}\Sigma e^{B^\top\theta} = \Sigma$ for every $\theta \in \R$. Inverting and replacing
$\theta$ by $-\theta$ yields
$e^{B^\top\theta}\Sigma^{-1}e^{B\theta} = \Sigma^{-1}$, hence
\[
(e^{B\theta}x)^\top \Sigma^{-1}(e^{B\theta}x) = x^\top \Sigma^{-1}x
\qquad \forall x \in \R^d,\ \theta \in \R,
\]
so $\mathcal{O}_x$ is contained in the ellipsoid $\{y : y^\top\Sigma^{-1}y = x^\top\Sigma^{-1}x\}$.

Then, to show that $p$ is constant on $\mathcal{O}_x$, we take determinants in
$e^{B\theta}\Sigma e^{B^\top\theta} = \Sigma$, which gives $(\det e^{B\theta})^2 = 1$. Hence
$e^{B\theta}$ preserves Lebesgue measure, so $e^{B\theta}X$ has density
$p \circ e^{-B\theta}$. Since $e^{B\theta}X \sim X$, and both densities are continuous, they
agree everywhere, i.e. $p(e^{B\theta}x) = p(x)$ for every $x \in \R^d$ and $\theta \in \R$.
\end{proof}

We now classify select random variables that satisfy auto-rotational invariance. This includes all mean zero Gaussians and all random variables with degenerate second moment.

\begin{lemma}
\label{lemma: gaussian_ARI}
Let $X \sim \mathcal{N}(0, \Sigma)$ with $\Sigma \succeq 0$. Then $X$ is
auto-rotationally invariant.
\end{lemma}
\begin{proof}
It suffices to consider $d \ge 2$, since all mean-zero Gaussians are auto-rotationally invariant in $d=1$.

If $\Sigma$ is rank-deficient, the claim follows from Lemma
\ref{lemma: degenerate_X_isotropy}. Otherwise $\Sigma \succ 0$, and since $d \ge 2$ we may pick
a nonzero skew-symmetric $C \in \R^{d \times d}$ and set $B = \Sigma^{1/2}C\Sigma^{-1/2} \neq 0$.
Then
\[
B\Sigma + \Sigma B^\top = \Sigma^{1/2}(C + C^\top)\Sigma^{1/2} = 0,
\]
so $\frac{d}{d\theta}\big(e^{B\theta}\Sigma e^{B^\top\theta}\big) = 0$ and hence
$e^{B\theta}\Sigma e^{B^\top\theta} = \Sigma$ for all $\theta$. Since $e^{B\theta}X$ is Gaussian
with mean $0$ and covariance $e^{B\theta}\Sigma e^{B^\top\theta} = \Sigma$, we conclude
$e^{B\theta}X \sim X$ for all $\theta$.
\end{proof}

\begin{lemma}
Let $X$ be a $d$-dimensional r.v. with rank-degenerate second moment
$\Sigma = \E[XX^\top]$. Then
$$
e^{B\theta}X \sim X \quad \forall \theta \in \R
$$
admits a non-trivial solution $B \neq 0$.
\label{lemma: degenerate_X_isotropy}
\end{lemma}
\begin{proof}
Since $\Sigma$ is rank-degenerate, it has a non-trivial kernel, and thus some $u \neq 0$ with $\Sigma u = 0$. Then, it follows that 
$\E[(u^\top X)^2] = u^\top \Sigma u = 0$, from which we conclude that $u^\top X = 0$ a.s. 

To design our non-trivial solution $B$, we then pick $B = wu^\top$ for
any $w \neq 0$. Then $BX = w(u^\top X) = 0$ a.s., so
$e^{B\theta}X = X$ a.s.
for every $\theta$, and in particular $e^{B\theta}X \sim X$.
\end{proof}


The following result demonstrates the connection between weak auto-rotational invariance (Definition \ref{def:weak_ari}) and full auto-rotational invariance (Definition \ref{def:full_ARI}).
\begin{lemma}
\label{lemma: bounded_weak_equals_full}
Let $X$ be a $d$-dimensional random variable with bounded support. Then $X$ is weakly
auto-rotationally invariant if and only if it is auto-rotationally invariant.
\end{lemma}
\begin{proof}
Full ARI (Definition \ref{def:full_ARI}) implies weak ARI (Definition \ref{def:weak_ari}) since equality of the law of $e^{B\theta}X$ to the law of $X$, implies that all their cumulants agree. 

To prove the converse, suppose $X$ is weakly ARI with matrix $B \neq 0$. Since the support is bounded, the moment generating function $M(v) = \E[e^{v^\top X}]$ is entire and positive, so the cumulant generating function $K = \log M$ is real-analytic on $\R^d$ with
$K(v) = \sum_{n \ge 1}\frac{1}{n!}T_n^{X}(v,\ldots,v)$ near the origin. By the definition of weak ARI, every
cumulant of order $n \neq 2$ is invariant under the map $v \mapsto e^{B^\top\theta}v$. Since the cumulant generating function of $e^{B\theta}X$ evaluated at $v$ is $K(e^{B^\top\theta}v)$, it follows that the cumulant generating functions of $X$ and $e^{B\theta}X$ differ only in their quadratic terms, i.e., their covariances:
\begin{align}
    K(e^{B^\top\theta}v) - K(v) = \tfrac{1}{2}v^\top Q_\theta v, \qquad
    Q_\theta = e^{B\theta}\cov(X)e^{B^\top\theta} - \cov(X).
    \label{eq: cgf_gap}
\end{align}
This equality holds not just near the origin, but actually on all of $\R^d$ since both sides are real-analytic. We now show that the quadratic difference $Q_\theta$ is in fact $0$.

Let
$h(v) = \sup\{v^\top x : x \in \mathrm{supp}(X)\}$, which is finite since the support is bounded. Then $K(sv) \le s\,h(v)$
for $s > 0$, while $K(sv) \ge s\,v^\top\E[X]$ by Jensen's inequality, so $K(sv) = O(s)$ along
every ray. Applying this to both terms on the left of \eqref{eq: cgf_gap} gives
$\frac{s^2}{2}v^\top Q_\theta v = O(s)$. Thus, if we divide by $s^2$ and let $s \to \infty$, then
$v^\top Q_\theta v = 0$ for all $v$, i.e. $Q_\theta = 0$. From this, we conclude that $K(e^{B^\top\theta}v) = K(v)$ for
all $v$ and $\theta$, so we have full auto-rotational invariance: $e^{B\theta}X \sim X$.
\end{proof}
A simple class of random variables that fail weak auto-rotational invariance (which we use in our numerical experiments) are discrete random variables, whose support spans the full space $\R^d$.
\begin{proposition}
Let $X_0 \sim p_0$ be a discrete random variable with bounded support. Then $X_0$ is weakly
auto-rotationally invariant if and only if the support of $p_0$ fail to span $\R^d$.
\label{prop: discrete_auto_rotate}
\end{proposition}
\begin{proof}
If support fails to span $\R^d$, then some $u \neq 0$ satisfies $u^\top X_0 = 0$ a.s., so $\Sigma = \E[X_0X_0^\top]$ is rank-degenerate and by Lemma \ref{lemma: degenerate_X_isotropy}, $X_0$ is auto-rotationally invariant, and hence weakly auto-rotationally invariant.

Conversely, suppose for contradiction that the support spans $\R^d$ and that $X_0$ is weakly ARI with matrix $B \neq 0$. We note that since the support is bounded, then, by Lemma \ref{lemma: bounded_weak_equals_full}, it suffices to derive a contradiction from the fact that $X_0$ satisfies full ARI while also having a support that spans $\R^d$. Since $e^{B\theta}X_0 \sim X_0$ for all
$\theta$, the map $e^{B\theta}$ must permute the atoms of $p_0$. For any atom $x$, the orbit $\{e^{B\theta}x : \theta \in \R\}$ must therefore consist of atoms. But the orbit is a connected set, since it is a continuous image of $\R$. Since the set of atoms is countable, and a connected subset of $\R^d$ with more than one point is uncountable, it follows that the orbit is a single point, i.e. $e^{B\theta}x = x$ for all $\theta$. Differentiating at $\theta = 0$ then
gives $Bx = 0$. Since this holds for every atom and the atoms span $\R^d$, we conclude $B = 0$, which contradicts the ARI assumption that $B \neq 0$.
\end{proof}

\section{Additional Proofs}





\rv{\begin{lemma}
If the covariances of two linear additive noise SDEs match for all times, i.e. $\cov(X_t) = \cov(Y_t)$ $\forall $ $t \ge 0$, such that $X_t$ evolves with parameters $(A,H)$ and $Y_t$ evolves with parameters $(\tilde{A}, \tilde{H})$, then their residual parameters $\bar{A}=A-\tilde{A}$ and $\bar{H}=H-\tilde{H}$ obey the Lyapunov equation
\begin{align}
    \bar{A}C_t + C_t \bar{A}^\top + \bar{H} = 0.
    \label{eq: differential_Lyapunov}
\end{align}
Thus, for any distinct times $t_1 \neq t_2$, we have
\begin{align}
    \bar{A} (C_{t_2}-C_{t_1}) + (C_{t_2}-C_{t_1})  \bar{A}^\top = 0.
    \label{eq:cov_diff}
\end{align}
\label{lemma:non_iden_imply_diff_drift}
\end{lemma}
\begin{proof}
Recall that given SDE parameters $(A,H)$, the covariance $C_t = \cov(X_t)$ obeys
\begin{align}
    C_t = e^{At}C_0 e^{A^\top t} + \int_{0}^t e^{A(t-s)}He^{A^\top (t-s)}ds.
    \label{eq:covariance_of_lin_sde}
\end{align}
Differentiating \eqref{eq:covariance_of_lin_sde} shows that the covariance of each system obeys its Lyapunov ODE, so for every $t \ge 0$:
\begin{align}
\frac{d}{dt}C_t&= AC_t + C_t A^\top + H 
\label{eq:deriv_cov_1}
\\
\frac{d}{dt}C_t&= \tilde{A}C_t+ C_t\tilde{A}^\top + \tilde{H}.
\label{eq:deriv_cov_2}
\end{align}
Since the covariances of $X_t$ and $Y_t$ are equal at all time points, so are their derivatives. Subtracting yields the desired equation \eqref{eq: differential_Lyapunov} at every $t \ge 0$.
\end{proof}}

\rv{\begin{proposition}
\label{prop:exhibit}
For any initial covariance $C_0 \in \mathrm{Sym}_d$, any $m \ge 1$, and any distinct 
times $t_1, \ldots, t_D > 0$ (where $D = \frac{d(d+1)}{2}$), there exist parameters 
$(A^*, G^*)$ such that $\det(M(A^*, G^*)) \neq 0$.
\end{proposition}}

\begin{proof}
\rv{Let $A^* = \mathrm{diag}(-2^1, -2^2, \ldots, -2^d)$. Since $A^*$ is diagonal, 
the covariance ODE decouples entry-by-entry:
\[
\dot{C}_{ij} = \mu_{ij}\, C_{ij} + (GG^\top)_{ij}, 
\qquad \mu_{ij} = \lambda_i + \lambda_j = -2^i - 2^j.
\]
The rates $\{\mu_{ij}\}_{i \leq j}$ are all distinct, since distinct subsets 
of $\{2^1, \ldots, 2^d\}$ have distinct sums. The scalar solution gives
\[
\Delta C_k^{(ij)} 
= C_{ij}(t_k) - C_{ij}(0) 
= \alpha_{ij}\bigl(e^{\mu_{ij}\, t_k} - 1\bigr),
\]
where $\alpha_{ij} = C_{ij}(0) + (GG^\top)_{ij}/\mu_{ij}$. Ordering the $D$ pairs
$(i,j)$ with $i \leq j$ as $\ell = 1, \ldots, D$, the matrix $M$ factors as
\[
M = F \cdot \mathrm{diag}(\alpha_\ell)_{\ell=1}^D,
\]
where $F_{k\ell} = e^{\mu_\ell\, t_k} - 1$. Therefore,
\begin{equation}
\label{eq:det_factorization}
\det(M) = \det(F) \cdot \prod_{\ell=1}^{D} \alpha_\ell.
\end{equation}}

\rv{We show that both factors are nonzero. For the first factor, consider the $D+1$ functions $\{1, e^{\mu_1 t}, \ldots, e^{\mu_D t}\}$ 
evaluated at the $D+1$ distinct points $\{0, t_1, \ldots, t_D\}$.  
Since the exponents $0, \mu_1, \ldots, \mu_D$ are all distinct, 
these functions form a Chebyshev system 
\citep[Ch.~1]{karlin1966tchebycheff}, meaning that their evaluation 
matrix at any set of distinct points is non-singular. In particular, 
the $(D{+}1) \times (D{+}1)$ matrix
\[
\widetilde{V} = \begin{bmatrix}
1 & 1 & \cdots & 1 \\
1 & e^{\mu_1 t_1} & \cdots & e^{\mu_D t_1} \\
\vdots & \vdots & \ddots & \vdots \\
1 & e^{\mu_1 t_D} & \cdots & e^{\mu_D t_D}
\end{bmatrix}
\]
is non-singular. Subtracting the first row from each subsequent row gives
\[
\begin{bmatrix}
1 & 1 & \cdots & 1 \\
0 & e^{\mu_1 t_1}-1 & \cdots & e^{\mu_D t_1}-1 \\
\vdots & \vdots & \ddots & \vdots \\
0 & e^{\mu_1 t_D}-1 & \cdots & e^{\mu_D t_D}-1
\end{bmatrix},
\]
whose determinant equals $\det(F)$ by cofactor expansion along the first 
column. Since row operations preserve non-singularity, $\det(F) \neq 0$ 
for any choice of distinct $t_k > 0$.}

Then, to ensure that the second factor is nonzero, we take $G^* = \sqrt{c}\,\mathbf{1}e_1^\top$, where $\mathbf{1} \in \R^d$ is the vector of all $1$'s, $e_1 \in \R^m$ is the first standard basis vector, and $c > 0$. Since every entry of $G^*(G^*)^\top$ equals $c$, we have $\alpha_{ij} = C_{ij}(0) - c/(2^i + 2^j)$, which vanishes for at most one value of $c$ for each of the $D$ pairs $(i,j)$. Choosing $c > 0$ outside these finitely many values gives $\alpha_\ell \neq 0$ for all $\ell$. Thus, by \eqref{eq:det_factorization}, $\det(M(A^*, G^*)) \neq 0$, as desired.

\end{proof}

\section{SDEs and Causality}
\label{sec: SDEs_causality}


Although nonzero entries in the observational diffusion $H = GG^\top$ inform the presence of latent confounders between pairs of variables, we note that it is impossible for $H$ alone to determine latent confounders over $p>2$ endogenous variables, $X_t^{(i_1)}, \cdots , X_t^{(i_p)}$.
Indeed, given only $H = GG^\top$, multiple causal interpretations may be possible, because structurally different matrices $G, \tilde{G}$ can obey $H = GG^\top = \tilde{G}\tilde{G}^\top$. 

\begin{example}[Non-identifiability of higher order latent confounders]
\label{ex: 3_bidirected_vs_multiedge}
Consider two matrices $G_1, G_2$, which share the same observational diffusion $H$:
\begin{align*}
    G_1 = \begin{bmatrix} 0 & 1 & 1 \\ 1 & 0 & 1 \\ 1 & 1 & 0 
\end{bmatrix}, \quad
G_2 = \begin{bmatrix}
    \frac{4}{3} & \frac{1}{3} & \frac{1}{3}\\
    \frac{1}{3} & \frac{4}{3} & \frac{1}{3}\\
    \frac{1}{3} & \frac{1}{3} & \frac{4}{3}
\end{bmatrix}, \quad H=G_1G_1^\top  = G_2G_2^\top =\begin{bmatrix} 2 & 1 & 1 \\ 1 & 2 & 1 \\ 1 & 1 & 2 
\end{bmatrix}
\end{align*}
Given $G_1$ as the additive noise of an SDE, each pair of variables shares a noise source, since each row shares a nonzero column entry with another row. However, there is no driving noise source that is shared among all three variables, since each column contains a $0$. In contrast, the causal interpretation under $G_2$ represents driving noise shared by all three variables.

Since $G_1$ and $G_2$ admit different causal graphs, despite having the same observational diffusion $H$, this shows that $H$ provides information about the existence of unobserved confounders between pairs of variables, but cannot provide further causal structure about the confounder with respect to the other endogenous variables $[d] \setminus \{i, j\}$.
\end{example}
Similarly, we can have $H_{i,j}=0$, while variables $X^{(i)}$ and $X^{(j)}$ share a noise source $W_t^{(k)}$.
\begin{example}[Zeros in diffusion $H$ can result from pairwise latent confounders]
\label{ex: cancellation_for_pairwise_latents}
\begin{align*}
G=\begin{bmatrix}
    1 & -1 \\
    1 & 1
\end{bmatrix} \implies H = \begin{bmatrix}
    2 & 0 \\
    0 & 2
\end{bmatrix}
\end{align*}
From $G$, $X^{(1)}$ and $X^{(2)}$ both depend on $W^{(1)}$ and $W^{(2)}$, however, this cannot be determined from the observational diffusion $H$, due to cancellations.
\end{example}

We note that if the noise is not additive, then this further complicates the causal interpretation. In this case, the driving noise $\sigma(X_t)$ may be a function of the endogenous variables, i.e. $\beta(j) \cap V \neq \emptyset$. Thus, unlike the additive noise setting in Lemma \ref{prop: identify_causal_graph_SDE_add_noise}, directed edges $i \to j$ may be informed by the driving noise, via $\beta(j)$, rather than just the drift, via $\alpha(j)$. However, since only $\sigma(X_t)\sigma(X_t)^\top $ is observable rather than $\sigma(X_t)$ itself, observational data under such a model would lend itself to multiple interpretations of the causal graph. The idea is similar to Example \ref{ex: 3_bidirected_vs_multiedge}, where we saw one interpretation feature three pairwise latent confounders and another feature a single confounder over all variables, but under non-additive noise, the ambiguity extends to edges $i \to j$ within the endogenous set. This is illustrated in Example 5.5 in \citet{hansen2014causal}.

The causal framework also highlights the impracticality of common model assumptions within the scRNA-seq literature, namely isotropic diffusion $H=\sigma^2 \Id$ and gradient-field drift $\nabla \psi$. Except in exceedingly rare cases of complete cancellations when taking $H=GG^\top$ (see Example \ref{ex: cancellation_for_pairwise_latents}), isotropic diffusion cannot model latent confounders from shared noise. Moreover, while we expect GRNs to contain feedback loops, i.e. cycles in $\mathcal{G}$, imposing irrotational drift prevents most cycles from being considered. Lemma \ref{prop: irrotational_implies_symmetric} shows that imposing irrotationality on a linear model is equivalent to imposing a symmetric drift, which determines a symmetric causal graph. However, as illustrated in \citet[Fig. 6]{weinreb2018fundamental}, symmetric GRNs cannot capture important relationships, such as negative feedback loops or repressilator dynamics. Indeed, consider the $2$-cycle $i \leftrightarrow j$. By Lemma \ref{prop: identify_causal_graph_SDE_add_noise}, $i \to j$ if $A_{j,i} \neq 0$ and  $j \to i$ if $A_{i,j} \neq 0$. However, if the model is irrotational, then by Lemma \ref{prop: irrotational_implies_symmetric}, $A_{i,j} = A_{j,i}$, which can only model a positive feedback loop.

\begin{lemma}
\label{prop: irrotational_implies_symmetric}
Let $X_t$ evolve according to a linear additive noise SDE \eqref{eq: linear_additive_noise_SDE}. Then, the drift $AX_t$ is irrotational if and only if $A$ is symmetric, i.e. $A = A^\top$.
\end{lemma}
\begin{proof}
First, we note that for any symmetric matrix $A$, the vector field $Ax$ can be expressed as the gradient of the scalar potential given by the quadratic form $\phi(x) = x^\top A x/2$ \citep[(96)-(97)]{petersen2008matrix}. Hence, $Ax$ is irrotational if $A = A^\top$.

For the opposite direction, we recall the fact that any gradient field from $\R^n \to \R^n$ has a symmetric Jacobian \citep{williamson1974multivariable, campbell2014reduction}. Since the Jacobian of a linear vector field $Ax : \R^n \to \R^n$ is $A$, we conclude that all irrotational linear vector fields from $\R^n \to \R^n$ are given by symmetric matrices.
\end{proof}

\section{Entropy-Regularized Optimal Transport}
\label{sec:EOT}
The entropy-regularized optimal transport problem also admits a dual formulation, in terms of finding a pair of  potentials $(f,g)$ with respect to the Gaussian transition kernel $K(x,y) \propto e^{\frac{-c(x,y)}{\epsilon^2}}\propto e^{\frac{-\|y-x\|^2}{2\epsilon^2}}$ \citep{janati2021advances, nutz2021introduction}:
\begin{align}
\pi^*(x,y) &= e^{{f^*(x)+g^*(y)}}K(x,y) \, \mathrm{d}\mu(x)\mathrm{d}\nu(y)\\
f^*, g^* &= \argmax_{f \in \mathcal{L}^1(\mu), g \in \mathcal{L}^1(\nu)} \E_{\mu}(f) + \E_{\nu}(g)- \left(\int_{\R^d \times \R^d} e^{f(x)+g(y)}K(x,y) \, \mathrm{d}\mu(x) \mathrm{d}\nu(y) - 1\right).
\label{eq: entropic_OT_dual_formulation}
\end{align}
In particular, given marginals $\mu,\nu$ and the transition kernel $K$, Sinkhorn's algorithm uses the dual formulation \eqref{eq: entropic_OT_dual_formulation} to find the $\epsilon$-entropy regularized optimal transport solution $\pi^*$, by solving for $f^*, g^*$ via iterative projections \citep{peyre2019computational}.


\begin{definition}[Markov
concatenation of couplings]
\label{def:Markov_concatenation_couplings}
Given Polish spaces $X_{1},X_{2},X_{3}$
and couplings $\pi_{1,2}\in\mathcal{P}(X_{1}\times X_{2})$ and $\pi_{2,3}\in\mathcal{P}(X_{2}\times X_{3})$
with identical marginals $\mu_{2}$ on $X_{2}$, the \emph{Markov concatenation
}$\pi_{2,3} \circ \pi_{1,2}$ is a multi-coupling
in $\mathcal{P}(X_{1}\times X_{2}\times X_{3})$ given by 
\[
(\pi_{2,3} \circ \pi_{1,2})(\mathrm{d}x_{1},\mathrm{d}x_{2},\mathrm{d}x_{3})=\pi_{1,2}(\mathrm{d}x_{1}\mid x_{2})\mu_{2}(\mathrm{d}x_{2})\pi_{2,3}(\mathrm{d}x_{3}\mid x_{2}).
\]
Here $\pi_{1,2}(\mathrm{d}x_{1}\mid x_{2})$ is the disintegration of $\pi_{1,2}$
with respect to $\mu_{2}$ and $\pi_{2,3}(\mathrm{d}x_{3}\mid x_{2})$ is the
disintegration of $\pi_{2,3}$ with respect to $\mu_{2}$. 
\end{definition}

The interpretation
of the Markov concatenation is as follows: a random ``trajectory'' according to
$\pi_{23}\circ\pi_{12}$ corresponds to taking the first two steps
distributed according to $\pi_{12}$, then the third step distributed
according to ``$\pi_{23}$ conditional on the second marginal of
$\pi_{12}$''. The existence of the Markov concatenation is guaranteed
by the disintegration theorem, and Markov concatenations appear naturally
in the time-discretized version of trajectory inference via Schr\"odinger
bridges \citep{lavenant2021towards}. In particular, given Polish spaces $X_1,\ldots,X_4$ and couplings $\pi_{12},\pi_{23},\pi_{34}$, it holds that Markov concatenation is associative, and so we can unambiguously define the iterated Markov concatenation $\pi_{34}\circ\pi_{23}\circ\pi_{12}$ (see \citet{benamou2019entropy} Section 3.2).

\begin{proposition}[Optimality of anisotropic entropic optimal transport]
Let the reference SDE be a linear additive noise SDE  with drift $A$ and diffusion $H$ \eqref{eq: linear_additive_noise_SDE}. Given a set of marginals $p_0,\ldots, p_{{N-1}}$ over times $\{t_i\}_{i=0}^{N-1}$ (not necessarily coming from the reference SDE) for which $D_{KL}(p_{i}\mid\text{Leb})<\infty$ for all $1\leq i\leq N-1$, and $D_{KL}(p_{i}\mid(p_{A,H})_{i})<\infty$ for all $1\leq i \leq N-2$, suppose there exists $\pi_{i,i+1}$, the anisotropic entropic OT solution (\ref{eq: generalized_entropic_OT_problem})  obtained by Sinkhorn's algorithm with marginals $\mu = p_i, \nu = p_{i+1}$ and transition kernel
\begin{align}
    K_{A,H}^{i}(x,y) \propto \exp(-\frac12 (y-e^{A (t_{i+1}-t_{i})}x)^\top(\Sigma_i)^{-1}(y-e^{A(t_{i+1}-t_{i})}x))
    \label{eq: gibbs_kernel_ap}
\end{align}
where $\Sigma_{i} = \int_{t_i}^{t_{i+1}} e^{A(t_{i+1}-s)} H e^{A^\top(t_{i+1}-s)} \, \mathrm{d}s$. Let $\pi$ denote the joint distribution given by the Markov concatenation of couplings (see Definition \ref{def:Markov_concatenation_couplings}):
\begin{align*}
    \pi = \pi_{N-2,N-1} \circ \ldots \circ \pi_{0,1}.
\end{align*}
Then, $\pi$ minimizes relative entropy to the law of the reference SDE among Markov couplings:
$$
\pi = \argmin_{\pi \in \Pi(p_{0}, \ldots, p_{N-1}),\ \pi \text{ Markov}} D_{KL}(\pi \| p_{A, H}^N),
$$
where $p_{A,H}^N$ is the law of the reference SDE, discretized over time points $t_0, \ldots, t_{N-1}$.
\label{prop: md_Sinkhorn_optimal}
\end{proposition}

\begin{proof}
For each $i=0, \ldots, N-2$, by  \eqref{eq: schrodinger_bridge}, we have $\pi_{i,i+1} = \argmin_{\pi \in \Pi(p_i, p_{i+1})}D_{KL}(\pi \| K_{A,H}^{i})$. Since {$K_{A,H}^{i} \sim \mathcal{N}(e^{A(t_{i+1}-t_{i})}X_t, \Sigma_{i})$} is the transition kernel of the SDE at time $t_i$ \citep[Sec 6.1]{sarkka2019applied}, this implies that $$\pi_{i,i+1} = \argmin_{\pi \in \Pi(p_i, p_{i+1})}D_{KL}(\pi \| (p_{A,H})_{i,i+1}),$$ where $ (p_{A,H})_{i,i+1}$ denotes the joint distribution of the marginals of $p_{A,H}$ at the two time points $t_i, t_{i+1}$. The details are as follows. Let $K_{i,i+1}(x,y)$ denote the transition kernel from time
$t_{i}$ to time $t_{i+1}$ for the SDE $\mathrm{d}X_{t}=AX_{t} \:\mathrm{d}t+G\mathrm{d}W_{t}$.
In particular, if $ (p_{A,H})_{i}$ is the marginal of this SDE at time $t_{i}$,
then we have that $K_{i,i+1}(x,y)\mathrm{d}(p_{A,H})_{i}(x)\mathrm{d}y$ is equal to
the joint distribution $ (p_{A,H})_{i,i+1}$. Hence

\[
D_{KL}(\pi\mid(p_{A,H})_{i,i+1})=D_{KL}(\pi\mid K_{A,H}(x,y)d(p_{A,H})_{i}(x)dy).
\]
We claim that for any $\pi\in\Pi(p_{i},p_{i+1})$, 
\begin{align*}
D_{KL}(\pi\mid K_{A,H}(x,y)d(p_{A,H})_{i}dy) & =\int\log\left(\frac{d\pi}{d(K\cdot(p_{A,H})_{i}\otimes\text{Leb})}(x,y)\right)d\pi\\
 & =\int\log\left(\frac{d\pi}{d(K\cdot p_{i}\otimes dp_{i+1})}(x,y)\frac{d(K\cdot p_{i}\otimes dp_{i+1})}{d(K\cdot(p_{A,H})_{i}\otimes\text{Leb})}(x,y)\right)d\pi\\
 & =\int\log\left(\frac{d\pi}{d(K\cdot p_{i}\otimes dp_{i+1})}(x,y)\frac{dp_{i}}{d(p_{A,H})_{i}}(x)\frac{dp_{i+1}}{d\text{Leb}}(y)\right)d\pi\\
 & =D_{KL}(\pi\mid K_{A,H}(x,y)dp_{i}(x)dp_{i+1}(y))+D_{KL}(p_{i}\mid(p_{A,H})_{i}) \\
 &+ D_{KL}(p_{i+1}\mid\text{Leb}).
\end{align*}
Note that $(p_{A,H})_{0}=p_{0}$.
Hence, under the assumption that for all $i=1,\ldots,N-1$, $D_{KL}(p_{i}\mid\text{Leb})<\infty$, the minimizers for the following two minimization problems are identical:
\[
\min_{\pi\in\Pi(p_{i}p_{i+1})}D_{KL}(\pi\mid(p_{A,H})_{i,i+1})\text{ and }\min_{\pi\in\Pi(p_{i}p_{i+1})}D_{KL}(\pi\mid K_{A,H}(x,y)dp_{i}(x)dp_{i+1}(y)).
\]


Since $\pi$ is constructed as a Markov concatenation, the conclusion follows from Lemma 3.4 of \citet{benamou2019entropy}, which in this case tells us that: if $p_{A,H}^{N}$ is the projection
of the law of the SDE onto the set of times $\{t_{0},\ldots,t_{N-1}\}$,
then for any $N$-coupling $\pi$ which is constructed as a Markovian
concatenation $\pi_{N-2, N-1} \circ\ldots \pi_{0,1}$, and has $i$th marginal
equal to $p_{i}$, we have 
\[
D_{KL}(\pi\mid p_{A,H}^{N})=\sum_{i=0}^{N-2}D_{KL}(\pi_{i}\mid (p_{A,H})_{i,i+1}) - \sum_{i=1}^{N-2}D_{KL}(p_i \mid (p_{A,H})_i ).
\]
Hence minimizing over each $D_{KL}(\pi_{i,i+1}\mid (p_{A,H})_{i,i+1})$
(for $\pi\in\Pi(p_{i},p_{i+1})$) is equivalent to minimizing over
Markovian $\pi$'s with $i$th marginal $p_{i}$.

\end{proof}

The previous proposition assumes that $D_{KL}(p_{i}\mid\text{Leb})<\infty$ for all $i\neq 0$. The following computation shows that this assumption is satisfied if the diffusion matrix is non-degenerate. 

We note that $X_t \sim e^{At}X_0 + \int_{0}^{t} e^{A(t-s)}GdW_s$.
As $X_0$ is independent from the noise $W_s$, and $\int_{0}^{t} e^{A(t-s)}GdW_s \sim \mathcal{N}(0, \Sigma_t)$, where $\Sigma_t = \int_0^t e^{A(t-s)}He^{A^\top (t-s)}ds$, it follows that 
\begin{align*}
p_t = e^{At}p_0 * \mathcal{N}(0, \Sigma_t).
\end{align*}
Thus, we may compute \cite[Theorem 9.4.1]{cover1999elements}
\begin{align*}
D_{KL} (\mathcal{N}(0, \Sigma_t) \mid Leb) = - \frac12 \log((2\pi e)^d \det(\Sigma_t)).
\end{align*}
From the definition of $\Sigma_t$, $\det(\Sigma_t) \neq 0$ as long as $H = GG^\top$ is full rank. Then, we use the fact that the $D_{KL}(\cdot \mid Leb)$ is convex, and apply Jensen's inequality for measures, to deduce that 
\begin{align*}
D_{KL}(p_0 * \mathcal{N}(0, \Sigma_t)\mid Leb) \leq \int D_{KL} (\mathcal{N}(x, \Sigma_t) \mid Leb) dp_0 (x) =  D_{KL} (\mathcal{N}(0, \Sigma_t) \mid Leb).
\end{align*}
This shows directly that $p_t$ has finite entropy whenever $t>0$ as desired.

\begin{remark}[Applying EOT for SDE trajectory inference]
\label{rmk: EOT_for_traj_inf}
In the standard entropic OT problem \eqref{eq: entropic_OT_problem}, the cost is the squared Euclidean distance, and $K \sim \mathcal{N}(x, \epsilon^2 \Id)$ is an isotropic Gaussian kernel. As noted in \citet{lavenant2021towards}, this implies that entropy regularized OT can be leveraged for trajectory inference from observed marginals, given a reference SDE. For example, to find the discretized law on paths $\pi^* \in \Pi(\mu, \nu)$ satisfying $P(X_t = \mu, X_{t+\Delta t} = \nu)$, which minimizes relative entropy to the law of a pure diffusion process $\mathrm{d}X_t = \sigma \mathrm{d}W_t$, one should set the entropic regularization $\epsilon^2 = \sigma^2 \Delta t$. This would correspond to minimizing the KL divergence to {$K(x,y)= e^{-\frac{\|y-x\|^2}{2 \sigma^2 \Delta t}} \sim \mathcal{N}(x, \sigma^2 \Delta t)$}, which is the approximate transition kernel of the reference SDE. 
\end{remark}

\section{Maximum Likelihood Estimation}
\label{sec:mle}

\begin{proof}[Proof of Proposition \ref{prop: MLE_SDE_params}] We follow the standard procedure for deriving maximum likelihood estimators \citep{peiris2003maximum, pavliotis2014stochastic}. Our likelihood function is given by
\begin{align*}
\mathcal{L} = \Pi_{i=0}^{N-2}\left(\Pi_{j=0}^{M-1}p(X_{{i+1}}^{(j)}|X_{i}^{(j)}) \right),
\end{align*}
where we have denoted $X_i = X_{t_i}$ for shorthand. As in \citet{pavliotis2014stochastic}, we consider the discretized law $X_{{i+1}}|X_{{i}} \sim \mathcal{N}(X_{i} + AX_i \Delta t, H \Delta t)$, which implies
\begin{align*}
p(X^{(j)}_{{i+1}}|X^{(j)}_{i}) = \frac{1}{(2\pi  \Delta t)^{d/2} \det(H)^{1/2}}\exp\left(-\frac12(\Delta X_i^{(j)} - AX_{{i}}^{(j)} \Delta t)^\top (H \Delta t)^{-1}(\Delta X_i^{(j)} - AX_{{i}}^{(j)} \Delta t) \right),
\end{align*}
where $\Delta X_i^{(j)}= X_{{i+1}}^{(j)}-X_{{i}}^{(j)}$. Plugging this back into the likelihood expression yields
\begin{align*}
\mathcal{L} &= \frac{1}{(2\pi  \Delta t)^{\frac{dM(N-1)}{2}} \det(H)^{\frac{M(N-1)}{2}}} \\
&\quad \times \exp\bigg(-\sum_{i=0}^{N-2} \sum_{j=0}^{M-1}\frac12(\Delta X_i^{(j)} - AX_{{i}}^{(j)} \Delta t)^\top (H \Delta t)^{-1}(\Delta X_i^{(j)} - AX_{{i}}^{(j)} \Delta t)\bigg)\\
\log(\mathcal{L})&= -\frac{M(N-1)}{2}\big(d\log(2\pi  \Delta t) + \log(\det(H))\big)\\
&\quad - \frac{1}{2}\sum_{i=0}^{N-2}\sum_{j=0}^{M-1}\Big( (\Delta X_i^{(j)} - AX_{{i}}^{(j)} \Delta t)^\top (H \Delta t)^{-1}(\Delta X_i^{(j)} - AX_{{i}}^{(j)} \Delta t) \Big).
\end{align*}
We then derive the maximum likelihood estimators through matrix calculus \citep{petersen2008matrix}:
\begin{align*}
\frac{d \log(\mathcal{L})}{dA} &= -\frac{1}{2} \sum_{i=0}^{N-2} \sum_{j=0}^{M-1} \frac{d}{dA} \left( (\Delta X_i^{(j)} - AX_{{i}}^{(j)} \Delta t)^\top (H \Delta t)^{-1}(\Delta X_i^{(j)} - AX_{{i}}^{(j)} \Delta t) \right) \\
&= -\frac{1}{2} \sum_{i=0}^{N-2} \sum_{j=0}^{M-1}  - 2 \Delta t\frac{d}{dA}\left((\Delta X_{i}^{(j)})^\top  (H \Delta t)^{-1}AX_{i}^{(j)}\right) + \Delta t^2\frac{d}{dA}\left(X_{i}^{{(j)}^\top }A^\top (H \Delta t)^{-1}AX_{i}^{(j)}\right)\\
&= -\frac{1}{2} \sum_{i=0}^{N-2} \sum_{j=0}^{M-1}  - 2 \Delta t (H \Delta t)^{-1}(\Delta X_{i}^{(j)})  X_{i}^{{(j)}^\top } + 2 \Delta t^2(H \Delta t)^{-1}AX_{i}^{(j)}X_{i}^{{(j)}^\top } \\
&=  \sum_{i=0}^{N-2} \sum_{j=0}^{M-1}    (H)^{-1}(\Delta X_{i}^{(j)})  X_{i}^{{(j)}^\top } -  \Delta t(H)^{-1}AX_{i}^{(j)}X_{i}^{{(j)}^\top }.
\end{align*}
We can solve for the MLE linear drift $A$ by setting $\frac{d\log(\mathcal{L})}{dA}=0$:
\begin{align*}
 \Delta t(H)^{-1}A\sum_{i=0}^{N-2} \sum_{j=0}^{M-1} X_{i}^{(j)}X_{i}^{{(j)}^\top } = (H)^{-1}\sum_{i=0}^{N-2} \sum_{j=0}^{M-1}\Delta X_{i}^{(j)}  X_{i}^{{(j)}^\top } \\
A = \frac{1}{ \Delta t}\left(\sum_{i=0}^{N-2} \sum_{j=0}^{M-1}\Delta X_{i}^{(j)}  X_{i}^{{(j)}^\top }\right)\left(\sum_{i=0}^{N-2} \sum_{j=0}^{M-1} X_{i}^{(j)}X_{i}^{{(j)}^\top }\right)^{-1}.
\end{align*}

We now estimate the diffusion $H$. For simplicity, we work with $H^{-1}= (GG^\top )^{-1}$
\begin{align*}
\frac{d \log(\mathcal{L})}{dH^{-1}} &=\frac{d}{dH^{-1}}\bigg(\frac{M(N-1)}{2}\log(\det(H^{-1})) \\
&\qquad\qquad - \frac{1}{2 \Delta t}\sum_{i=0}^{N-2}\sum_{j=0}^{M-1}\Big( (\Delta X_i^{(j)} - AX_{{i}}^{(j)} \Delta t)^\top H^{-1}(\Delta X_i^{(j)} - AX_{{i}}^{(j)} \Delta t) \Big)\bigg)\\
&= \frac{M(N-1)}{2}H - \frac{1}{2 \Delta t}\sum_{i=0}^{N-2}\sum_{j=0}^{M-1} \Big( (\Delta X_i^{(j)} - AX_{{i}}^{(j)} \Delta t)(\Delta X_i^{(j)} - AX_{{i}}^{(j)} \Delta t)^\top  \Big).
\end{align*}

We can solve for the MLE of $H$ by setting $\frac{d\log(\mathcal{L})}{dH^{-1}}=0$
\begin{align*}
    H &= \frac{1}{M(N-1) \Delta t}\sum_{i=0}^{N-2}\sum_{j=0}^{M-1} \left( (\Delta X_i^{(j)} - AX_{{i}}^{(j)} \Delta t)(\Delta X_i^{(j)} - AX_{{i}}^{(j)} \Delta t)^\top  \right)\\
    &= \frac{1}{MT}\sum_{i=0}^{N-2}\sum_{j=0}^{M-1} \left( (\Delta X_i^{(j)} - AX_{{i}}^{(j)} \Delta t)(\Delta X_i^{(j)} - AX_{{i}}^{(j)} \Delta t)^\top  \right).
\end{align*}
\end{proof}

\begin{corollary}
Let $X_t$ evolve according to a general additive noise SDE \eqref{eq: additive_noise_SDE}. Then the Euler-approximated likelihood function is given by
\begin{align*}
\mathcal{L} &= \frac{1}{(2\pi  \Delta t)^{\frac{dM(N-1)}{2}} \det(H)^{\frac{M(N-1)}{2}}} \\
&\quad \times \exp\bigg(-\sum_{i=0}^{N-2} \sum_{j=0}^{M-1}\frac12(\Delta X_i^{(j)} - b(X_{{i}}^{(j)}) \Delta t)^\top (H \Delta t)^{-1}(\Delta X_i^{(j)} - b(X_{{i}}^{(j)}) \Delta t)\bigg).
\end{align*}
Taking the logarithm gives
\begin{align*}
\log(\mathcal{L})&= -\frac{M(N-1)}{2}\left(d\log(2\pi  \Delta t) + \log(\det(H))\right)\\
&- \frac{1}{2}\sum_{i=0}^{N-2}\sum_{j=0}^{M-1}\left( (\Delta X_i^{(j)} - b(X_{{i}}^{(j)}) \Delta t)^\top (H \Delta t)^{-1}(\Delta X_i^{(j)} - b(X_{{i}}^{(j)}) \Delta t) \right).
\end{align*}
Hence, given that the drift function $b$ is parameterized by values $\alpha_b^{(k)}$, the maximum likelihood solution for the drift function $b$ obeys
\begin{align*}
0 = - \frac{1}{2}\sum_{i=0}^{N-2}\sum_{j=0}^{M-1}\frac{d}{d\alpha_b^{(k)}}\left( (\Delta X_i^{(j)} - b(X_{{i}}^{(j)}) \Delta t)^\top (H \Delta t)^{-1}(\Delta X_i^{(j)} - b(X_{{i}}^{(j)}) \Delta t) \right).
\end{align*}
The MLE for the diffusion $H$ admits the closed solution 
\begin{align*}
    H &= \frac{1}{MT}\sum_{i=0}^{N-2}\sum_{j=0}^{M-1} \left( (\Delta X_i^{(j)} - b(X_{{i}}^{(j)}) \Delta t)(\Delta X_i^{(j)} -b(X_{{i}}^{(j)}) \Delta t)^\top  \right).
\end{align*}
\label{cor: general_MLE_additive_Noise}
\end{corollary}

\begin{remark} 
Previous works have predominantly focused on the case of observing a single long trajectory rather than a collection of short trajectories. Suppose that one observes a set of $M$ observed trajectories, then drift estimation may also be performed by averaging the classical MLE estimator for a single trajectory across observations. In dimension $1$, this is given by  $\hat{A}_T = \frac{1}{M} \sum_j \frac{\int_0^T X^{(j)}_t \, \mathrm{d}X^{(j)}_t}{\int_0^T (X^{(j)}_t)^2 \, \mathrm{d}t}$. Note that this is distinct from the MLE estimator that we derived for multiple trajectories in \eqref{eq: mle_drift}. Indeed, we observe that the latter estimator converges at a much faster rate than the averaged classical estimator.
\end{remark}

\section{Additional Experiments}
\label{sec:consistency}
\begin{figure}[h!]
    \centering
    \includegraphics[width=0.4\textwidth]{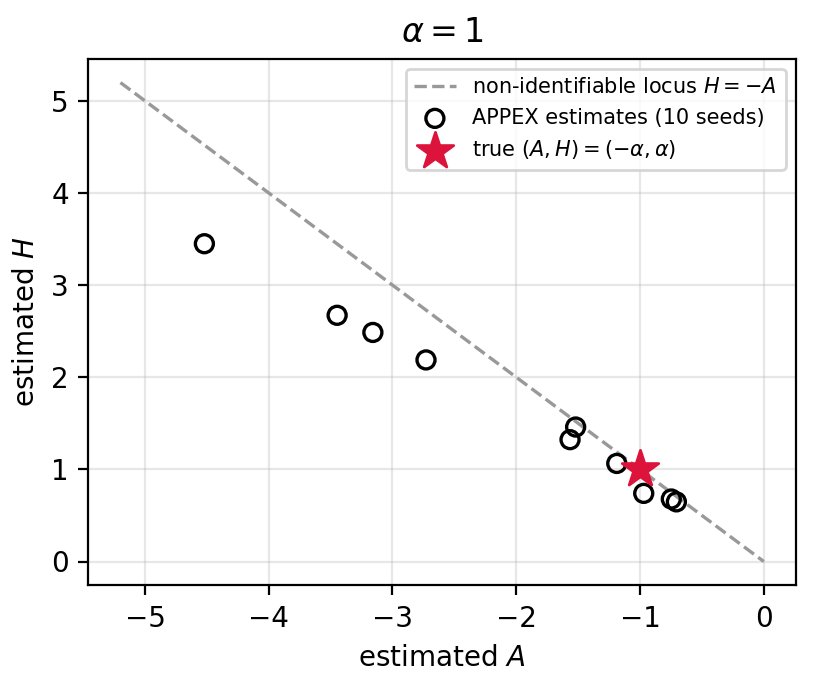}
    \caption{We plot APPEX's drift and diffusion estimates across $10$ random seeds for the Ornstein--Uhlenbeck SDE
    $\mathrm{d}X_t = -X_t\,\mathrm{d}t + \mathrm{d}W_t$ from the non-identifiable setting, where marginals are sampled from $p_{\textrm{eq}} =\mathcal{N}(0,\tfrac12)$.} 
\label{fig:OU_nonidentifiable}
\end{figure}

We also present the results of additional experiments verifying the consistency of our method, as we increase the number of samples per marginal $M$ (Figures \ref{fig:convergence} and \ref{fig:convergence12d}).
\begin{figure}[h!]
    \centering
    \begin{subfigure}{0.41\textwidth}
        \centering
        \includegraphics[width=\textwidth]{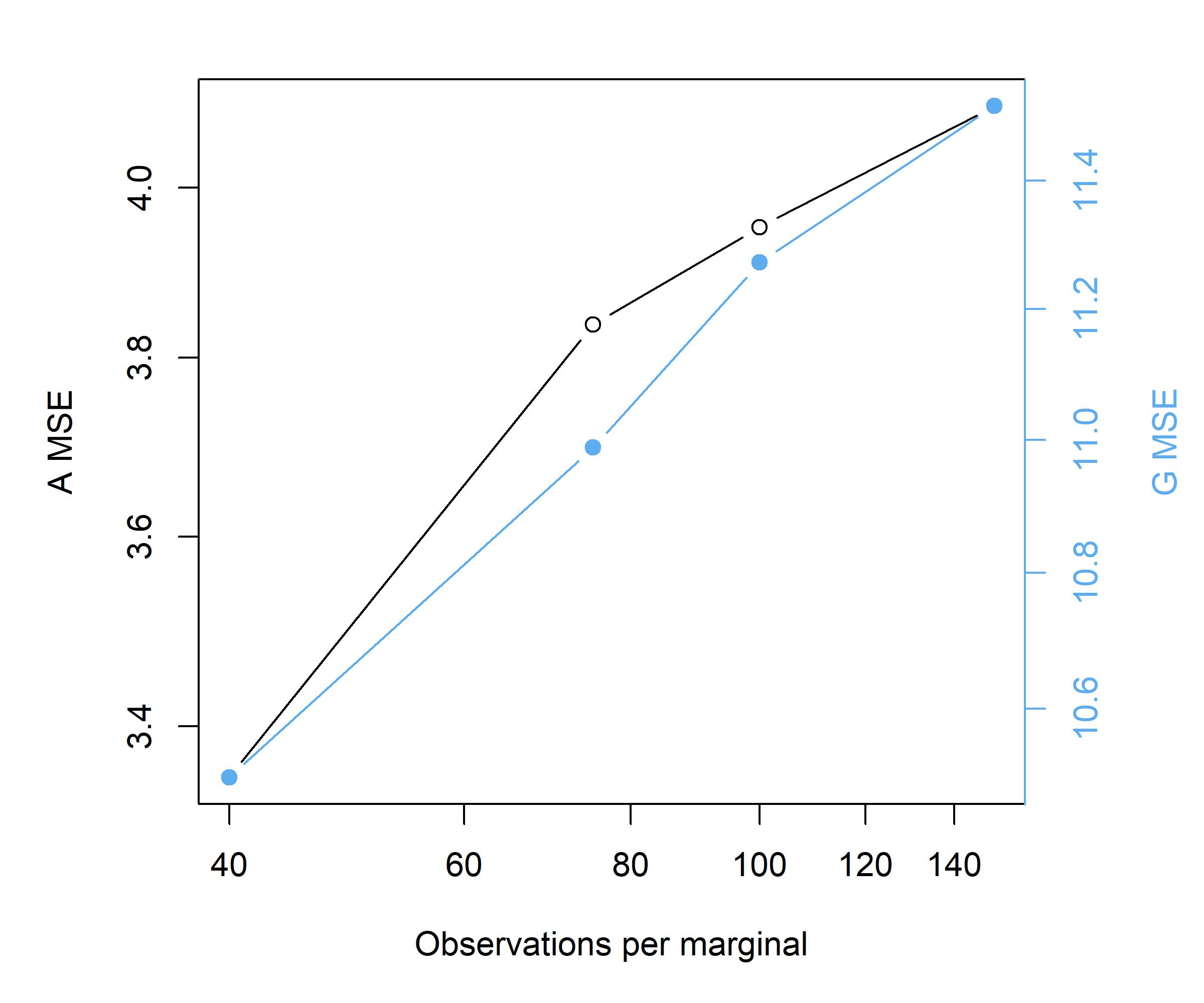}
        \caption{WOT}
    \end{subfigure}
    \begin{subfigure}{0.41\textwidth}
        \centering
        \includegraphics[width=\textwidth]{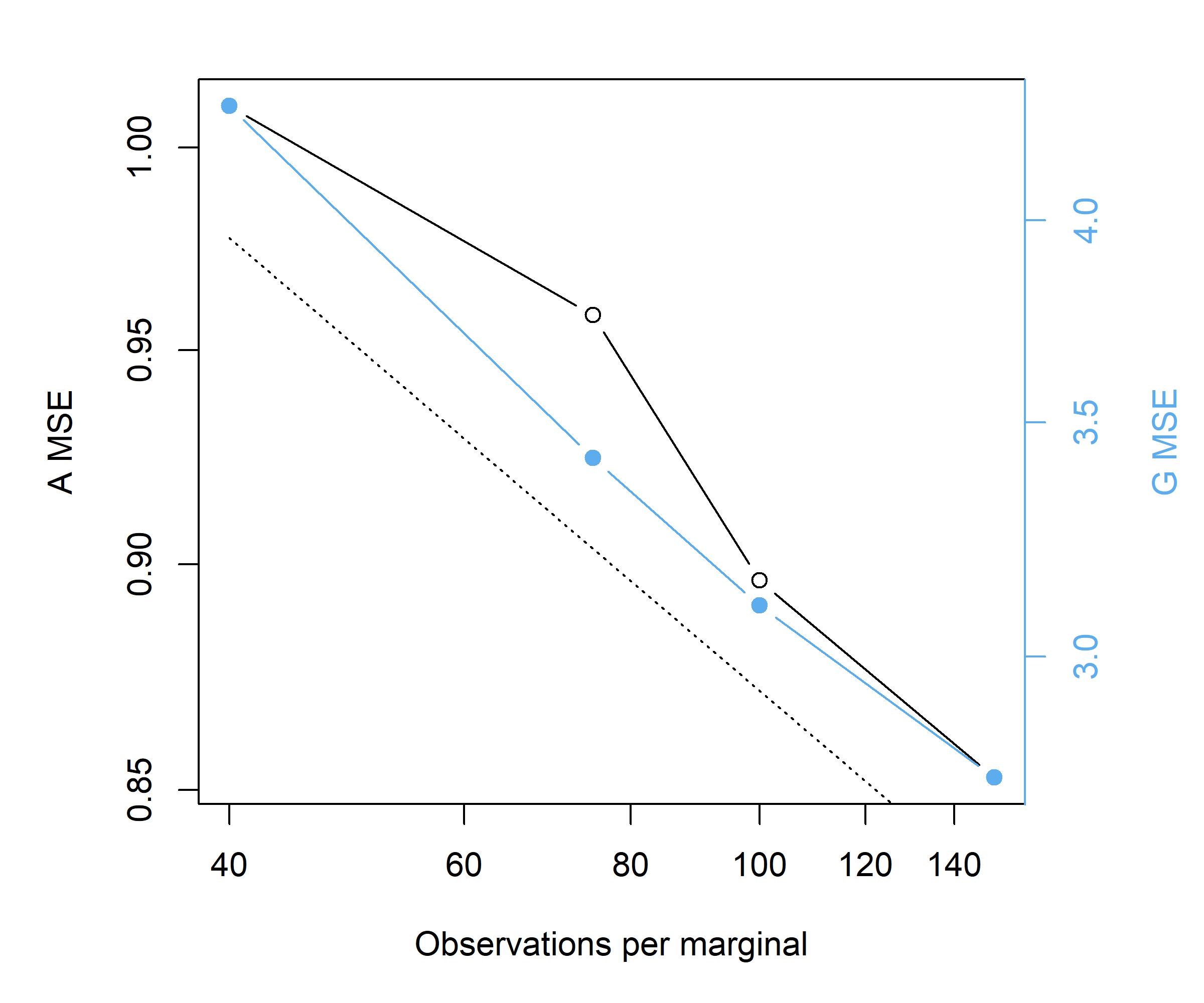}
        \caption{APPEX}
    \end{subfigure}
    \caption{Average mean squared error (MSE) of \( A \) and \( GG^\top \) as the number of observations per marginal increases across 50 random 2d systems. Note that this is a log-log plot and the dotted line in the APPEX plot shows that the convergence follows a power law relationship with exponent \(-0.125\). The number of observed marginals is fixed at $20$.}
    \label{fig:convergence}
\end{figure}

\begin{figure}[h!]
    \centering
    \begin{subfigure}{0.41\textwidth}
        \centering
        \includegraphics[width=\textwidth]{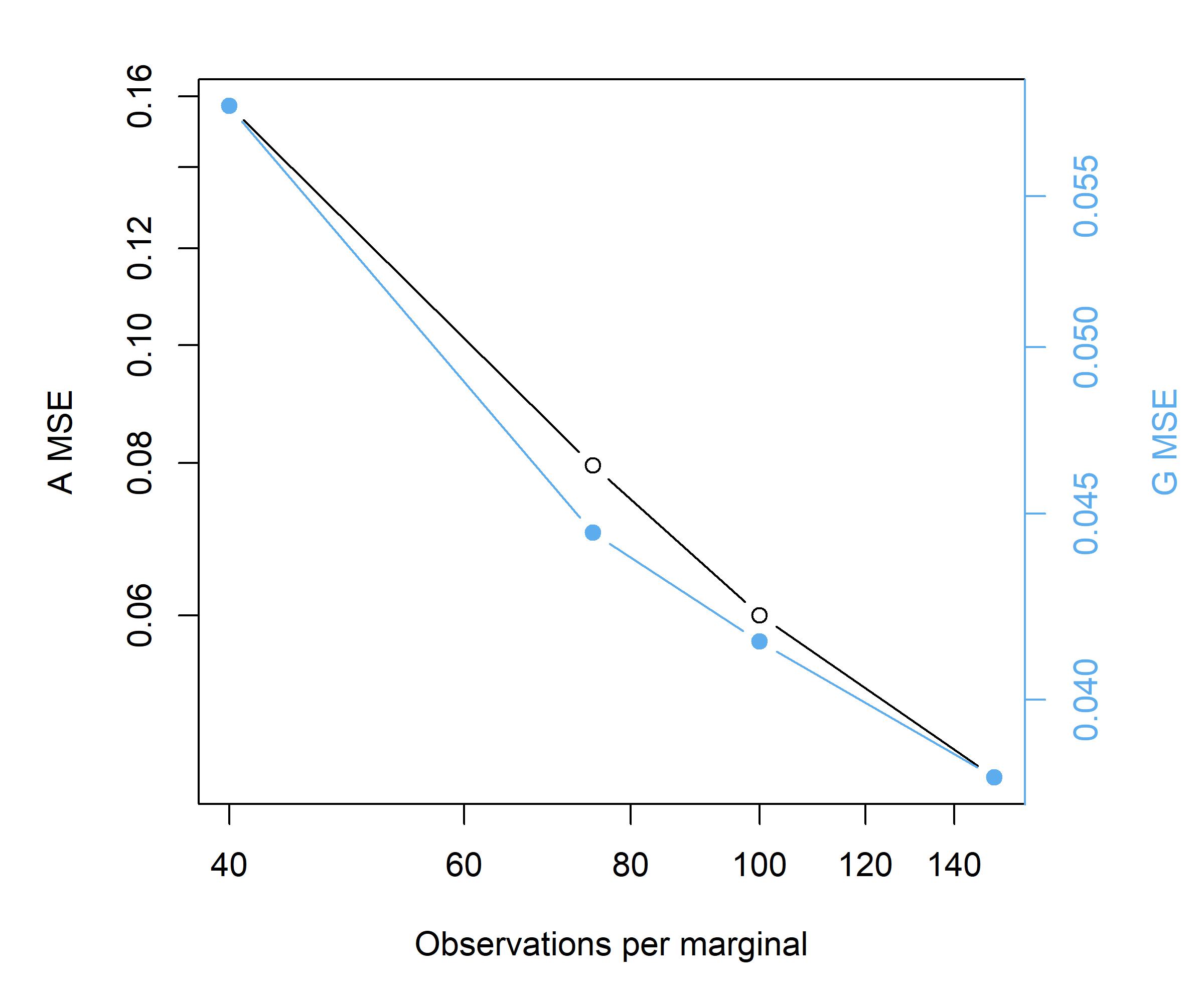}
        \caption{WOT}
    \end{subfigure}
    \begin{subfigure}{0.41\textwidth}
        \centering
        \includegraphics[width=\textwidth]{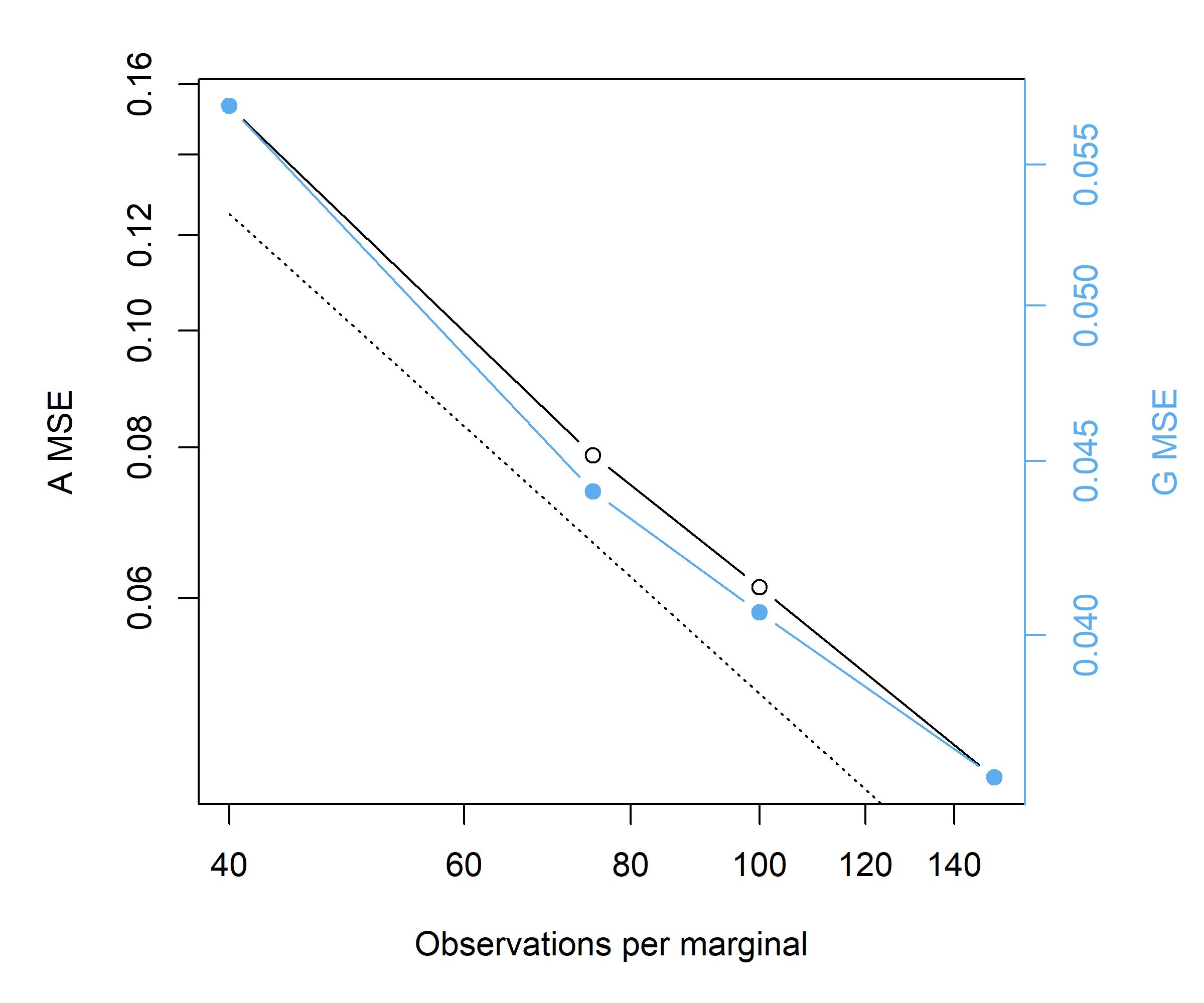}
        \caption{APPEX}
    \end{subfigure}
    \caption{Average mean squared error (MSE) of \( A \) and \( GG^\top \) as the number of observations per marginal increases across 70 random 12d systems. In this log-log plot, the dotted line in the APPEX plot indicates that the convergence follows a power law relationship with exponent \(-1\). The number of observed marginals is fixed at $20$.}
    \label{fig:convergence12d}
\end{figure}

\vskip 0.2in
\bibliography{references}

\end{document}